\documentclass{article}
\usepackage[margin=1in]{geometry}
\usepackage{graphicx}

\usepackage[dvipsnames]{xcolor}
\usepackage{hyperref}
\newcommand\myshade{100}
\colorlet{mylinkcolor}{RubineRed}
\colorlet{mycitecolor}{RoyalBlue}
\colorlet{myurlcolor}{RoyalBlue}

\hypersetup{
  linkcolor  = mylinkcolor!\myshade!black,
  citecolor  = mycitecolor!\myshade!black,
  urlcolor   = myurlcolor!\myshade!black,
  colorlinks = true,
}

\usepackage{pdfcomment}
\usepackage[inline, nomargin, mode=multiuser, author=]{fixme} 
\fxloadlayouts{pdfcnote}
\fxuseenvlayout{colorsig} 
\fxusetargetlayout{changebar} 
\FXRegisterAuthor{aa}{anaa}{AA}
\FXRegisterAuthor{zy}{anzy}{ZY}
\fxusetheme{colorsig}
\fxsetface{margin}{\scriptsize}

\usepackage{mymathstyle}

\usepackage{caption}
\usepackage{subcaption}

\usepackage{wrapfig}

\usepackage[shortlabels]{enumitem}

\usepackage{parskip}

\usepackage{algorithm2e}

\usepackage[backend=biber, style=authoryear-comp, maxbibnames=10, maxcitenames=2, url=false, uniquelist=false, isbn=false, sortcites=false, natbib]{biblatex}
\DeclareFieldFormat{title}{\mkbibquote{#1}} 
\let\hat\widehat
\let\bar\overline
\newcommand{\expectunder}[2]{\bbE^{#1}\bra{#2}}
\newcommand{\probunder}[2]{\bbP^{#1}\bra{#2}}
\newcommand{\probbar}[1]{\bar{\bbP}\bra{#1}}
\newcommand{\Probbar}[1]{\bar{\bbP}[#1]}
\newcommand{\probbarunder}[2]{\bar{\bbP}_{#1}\bra{#2}}
\newcommand{\probb}[2]{\bbP_{#1}^{#2}}
\newcommand{\thetahatk}{\hat{\theta}^k}
\newcommand{\bhatk}{\hat{b}^k}
\newcommand{\Uhathk}{\hat{U}_{h}^{k}}
\newcommand{\Uhk}{U_{h}^{k}}
\newcommand{\psibarstar}{\bar{\psi}^{*}}
\newcommand{\hatbar}[1]{\hat{\bar{#1}}}
\newcommand{\uhexp}[1]{\mathtt{u}_{#1}^{\mathtt{exp}}}
\newcommand{\Qhexp}[1]{\bbQ_{#1}^{\mathtt{exp}}}
\newcommand{\val}[2]{V_{#1}^{#2}}
\newcommand{\act}{\mathtt{act}}
\newcommand{\obs}{\mathtt{obs}}

\title{On the Role of Information Structure in Reinforcement Learning for Partially-Observable Sequential Teams and Games}
\author{
Awni Altabaa \\
Department of Statistics \& Data Science\\
Yale University\\
\texttt{awni.altabaa@yale.edu}
\and
Zhuoran Yang \\
Department of Statistics \& Data Science\\
Yale University\\
\texttt{zhuoran.yang@yale.edu}
}
\date{
  }
\begin{document}
\maketitle

\begin{abstract}
    In a sequential decision-making problem, the \textit{information structure} is the description of how events in the system occurring at different points in time affect each other.
    Classical models of reinforcement learning (e.g., MDPs, POMDPs, Dec-POMDPs, and POMGs) assume a simple and highly regular information structure, while more general models like predictive state representations do not explicitly model the information structure. By contrast, real-world sequential decision-making problems typically involve a complex and time-varying interdependence of system variables, requiring a rich and flexible representation of information structure. 

    In this paper, we argue for the perspective that explicit representation of information structures is an important component of analyzing and solving reinforcement learning problems. Taking inspiration from the control literature, we propose \textit{partially-observable sequential teams} and \textit{partially-observable sequential games} as reinforcement learning models with an explicit representation of information structure as part of the problem specification, capturing classical models of reinforcement learning as special cases.
    We then use these models to carry out an information-structural analysis of the statistical hardness of general sequential decision-making problems, obtaining a characterization via a graph-theoretic analysis of the DAG representation of the information structure.
    The central quantity in this analysis is the minimal set of variables, observable or latent, that $d$-separates the past observations from future observations.
    We prove an upper bound on the sample complexity of learning a general sequential decision-making problem in terms of its information structure by exhibiting an algorithm achieving the upper bound.
    This recovers known tractability results and gives a novel perspective on reinforcement learning in general sequential decision-making problems, providing a systematic way of identifying new tractable classes of problems.
\end{abstract}

\section{Introduction}\label{sec:intro}

The \textit{information structure} of a sequential decision-making problem is a description of how events in the system occurring at different points in time affect each other. In particular, in a causal sequential system, the information structure describes the subset of past events which have a direct effect on the present. This includes the information available to each agent at each time that they take an action as well as the information that affects the dynamics of the system. The control community has long recognized the importance of information structure, leading to the development of the celebrated Witsenhausen intrinsic model~\parencite{witsenhausenIntrinsicModelDiscrete1975a}, and extensive study since the 1970s~\parencite[e.g.,][]{witsenhausen1971information,ho1972team,hoEquivalenceInformationStructures1973,yoshikawa1978decomposition,witsenhausenEquivalentStochasticControl1988,andersland1992information,teneketzisInformationStructuresNonsequential1996,tatikondaControlCommunicationConstraints2000,
mahajanInformationStructuresOptimal2012,nayyarCommonInformationApproachDecentralized2014,
saldiGeometryInformationStructures2022}

In contrast to the control literature, reinforcement learning has so far primarily studied problems where the information structure is either fixed and highly regular, or not explicitly considered. For example, in an MDP or a Markov team/game, it is assumed that there exists a Markovian state variable that is observable by the agent(s) and which forms a sufficient statistic for the evolution of the system. Such a model lacks the expressivity needed to naturally capture real-world sequential decision-making problems where each event in the system may have an arbitrary dependence on past events.

Similarly, the treatment of partial-observability in commonly studied models is restrictive compared to the complexity of real-world problems. In general, partial observability refers to the fact that a system's evolution is dictated by a potentially large number of sequential events, but only a subset of these will be observable by the learning agent. For example, in a POMDP---the typical model of partial observability studied in the RL literature---it is assumed that there exists a Markovian state and that the observables at each point in time are noisy measurements of the current state. This assumption is often unrealistic, since general systems may not have ``states'' per se, and observations may be generated with more complex dependencies.

The highly regular information structures of these models make analysis more tractable and enable favorable learning results~\parencite[e.g.,][]{singh2000convergence,sutton2008convergent,munos2008finite,abbasi2011regret,lattimore2012pac}. Correspondingly, reinforcement learning has achieved notable empirical success in a wide range of domains, including in multi-agent systems~\parencite[e.g.,][]{mnihPlayingAtariDeep2013,koberReinforcementLearningRobotics2013,mnihHumanlevelControlDeep2015,silverMasteringGameGo2016,shalev-shwartzSafeMultiAgentReinforcement2016,vinyalsAlphastarMasteringRealtime2019}. Despite this success, a general theory of information structure in reinforcement learning is missing. As we show in this work, the tractability of modeling and learning a sequential decision-making problem can indeed be characterized in terms of its information structure.
By explicitly modeling the information structure, we can identify a broader class of tractable decision-making problems and develop more tailored approaches to reinforcement learning which exploit modeling the information structure.

In this work, we argue for the perspective that information structure is an important component of analyzing and solving reinforcement learning problems. A rich and flexible representation of information structure is needed to faithfully represent real-world sequential decision-making problems, where the system evolves according to a complex and time-varying dependence on the past, and different agents will have different information available to them at different points in time.
Our main contributions are summarized as follows: 
\textbf{1)} we present a general model of sequential decision-making with an explicit representation of information structure; 
\textbf{2)} through this model, we analyze the ``complexity'' of sequential decision-making problems as a function of their information structure, and, in doing so, identify a class of tractable problems; and 
\textbf{3)} we prove an upper bound on the sample complexity of learning a general sequential decision-making problem as a function of its information structure by exhibiting an algorithm achieving the upper bound.
We give a more detailed overview below.

\begin{figure*}
    \begin{minipage}[t]{0.32\textwidth}
        \centering
        \includegraphics[width=0.9\textwidth]{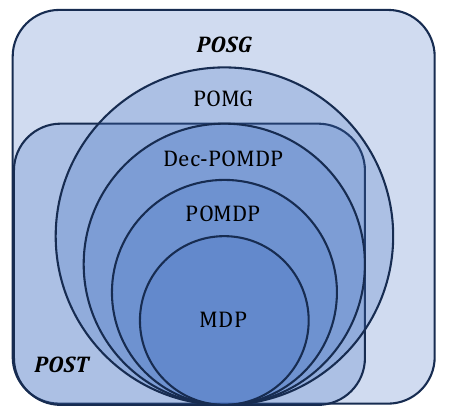}
        \caption{A depiction of the generality of our proposed models. POSTs and POSGs capture MDPs, POMDPs, Dec-POMDPs, and POMGs as special cases.}\label{fig:model_venn_diagram}
    \end{minipage}
    \hfill
    \begin{minipage}[t]{0.65\textwidth}
        \centering
        \includegraphics[width=0.95\textwidth]{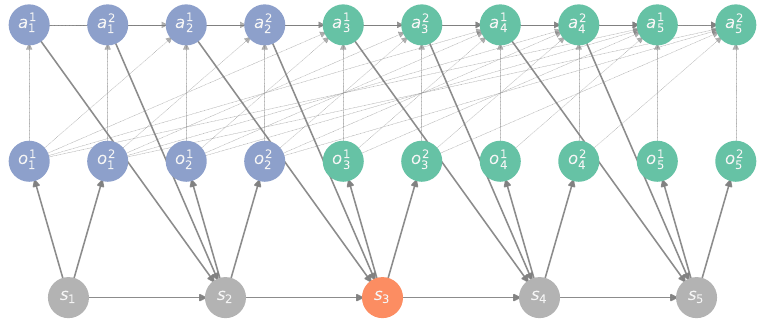}
        \caption{A depiction of the information structure of a (2-agent) Dec-POMDP/POMG within the POST/POSG framework. Blue nodes indicate past observables, green nodes indicate future observables, and orange nodes indicate the information structural state $\calI^\dagger$. This shows that in the case of Dec-POMDPs/POMGs, the information structural state recovers the latent Markovian state.}\label{fig:decpomdp_infostruct_state}
    \end{minipage}
\end{figure*}


\subsection{Overview of contributions and technical challenges}

\subsubsection*{An expressive model that explicitly represents information structure and partial observability}

Taking inspiration from the control literature, we propose \textit{partially-observable sequential teams} (POST) and \textit{partially-observable sequential games} (POSG) as highly general models with an \textit{explicit representation of information structure} as part of the problem specification. This forms a unifying framework that captures many commonly studied RL models as special cases, including MDPs, Markov teams/games, POMDPs, and Dec-POMDPs/POMGs (\Cref{fig:model_venn_diagram}). It also enables the study of additional more specialized decision-making problems (e.g., real-time communication with feedback or mean-field problems) in a common framework. The explicit representation of information structure enables a richer analysis of sequential decision-making problems, as well as more tailored learning algorithms.

In addition to an explicit representation of information structure, this framework also leads to a more general formalization of \textit{partial-observability}. Through its ability to represent information structure, our proposed model gives the most general formulation of partial observability---each system variable has an arbitrary dependence on the past, and the learning agent may observe an arbitrary subset of the system variables. This generality allows us to distinguish between ``observability'' in the context of decision-making (i.e., the information available at the time of making each decision) and observability in the context of learning (i.e., the information available to the learning algorithm).

\subsubsection*{Theoretical analysis of sequential decision-making through information structure.}

POSTs and POSGs are highly general models that capture a wide array of sequential decision-making problems, some of which possess an information structure that makes them tractable and some of which do not. A core contribution of this work is to characterize the rank of the observable dynamics through a graph-theoretic complexity metric of the information structure. In particular, we identify that the complexity is captured by the minimal set of variables (including latent variables) that $d$-separates the past observations from future observations. This gives a measure of the ``complexity'' of the objective-relevant part of the system dynamics for a sequential decision-making problem. This result gives a clear and interpretable condition in terms of the information structure for when a sequential decision-making problem can be represented tractably or not. Moreover, this recovers known results on the tractability of various structured classes of sequential decision-making problems such as MDPs, POMDPs, Dec-POMDPs, etc. The result centers around a graph-theoretic quantity of the DAG representation of the information structure, which can be interpreted as an \textit{effective} information-structural state, generalizing the typical notion of a Markovian state.

\subsubsection*{Learning theory: information-structural analysis of statistical hardness}

A key challenge in reinforcement learning is how to construct robust and efficient representations of probabilities of the form $\bbP(\mathtt{future}\ggiven \mathtt{past})$. Without good representations, modeling and learning these probabilities would be intractable. Predictive state representations (PSRs) give a powerful way to represent these probabilities, predicting future observations given the past, without explicitly modeling a latent state.
Since standard PSRs cannot represent our POST and POSG models, we formalize a generalization of predictive state representations which can. This generalized PSR formulation may be of independent interest since the analysis techniques which have recently proven successful for standard PSRs~\parencite[e.g.,][]{zhangReinforcementLearningMultiagent2021,ueharaProvablyEfficientReinforcement2022,chenPartiallyObservableRL2022,liuWhenPartiallyObservable2022,zhanPACReinforcementLearning2022,huangProvablyEfficientUCBtype2023} carry over to our generalized PSR formulation. We identify a class of POSTs and POSGs which admit a (well-conditioned) generalized PSR representation, and explicitly construct this representation through the information structure.

We prove an upper bound on the achievable sample complexity of learning general sequential decision-making problems as a function of the information structure. This is a result that roughly says ``any sequential decision-making problem with an information structure $\calI$ can be learned with a sample complexity at most $f(\calI)$''. 
The key quantity characterizing the sample complexity is the cardinality of the aforementioned information-structural state space.
We prove this result by exhibiting an algorithm achieving the upper bound, adapting recent work on learning in standard PSRs~\parencite{huangProvablyEfficientUCBtype2023} and extending it to our generalized PSR model. We tackle this question for both the team setting, where we learn an $\epsilon$-optimal policy, and the game setting, where we learn an $\epsilon$-equilibrium.

\subsection{Related Work}

\textbf{The study of information structure in the control literature.} In stochastic control, the construct of the ``information structure'' is used to model the structural properties of a system which may restrict the flow, storage, and processing of information. The role of information structure in decentralized control has been extensively studied since~\citet{witsenhausen1971separation} and~\citet{ho1972team} began investigating information structures in the context of team decision theory. For example, early work showed that the information structure can determine the tractability of optimal decentralized control problems~\parencite{witsenhausen1971separation,papadimitriouComplexityMarkovDecision1987}. More generally, information structure plays an important role in the analysis of multi-agent team decision problems and games, as well as in the design of efficient algorithms, especially in the decentralized setting.
See, e.g.,~\citet{nayyar2010optimal, nayyar2013common, nayyar2013decentralized, mahajanInformationStructuresOptimal2012, ouyang2016dynamic, dave2022decentralized, guan2023zero} and the references therein.
We also refer the reader to~\citet{yukselStochasticTeamsGames2023} for a comprehensive overview of the interaction between information and control, including recent progress in the field. The models we propose in this paper are closely related to Witsenhausen's intrinsic model~\parencite{witsenhausenIntrinsicModelDiscrete1975a}, but with some added elements to model partial-observability in the context of reinforcement learning. Whereas the above-mentioned work studies the role of information structure in planning and control, we study the role of information structure in \textit{reinforcement learning} (i.e., statistical estimation and sample complexity).

\textbf{Learning under partial observations.} In an MDP, where the system dynamics obey a Markovian property and are fully observable, reinforcement learning has been shown to be both computationally and statistically efficient~\parencite[e.g.,][]{auerNearoptimalRegretBounds2008,agrawalOptimisticPosteriorSampling2017,azarMinimaxRegretBounds2017,rashidinejadBridgingOfflineReinforcement2021}. However, under partial observability, reinforcement learning can be computationally and statistically intractable in the worst case, even when assuming a Markovian latent state. Such worst-case hardness results are well-known. For example,~\parencite{papadimitriouComplexityMarkovDecision1987,mundhenkComplexityFinitehorizonMarkov2000,lassisComputationalComplexityStochastic2012} show that planning is computational intractable and~\parencite{mosselLearningNonsingularPhylogenies2005,krishnamurthyPACReinforcementLearning2016} show that learning is statistically intractable, in the worst-case. In these worst cases, the hardness comes from instances where the observations reveal little information about the latent state, which causes errors in learned representations to be uncontrollable. Accordingly, sub-classes of POMDPs have been identified in recent work where added structural conditions make efficient learning possible. One such condition is decodability~\parencite[see e.g.,][]{krishnamurthyPACReinforcementLearning2016,efroniProvableReinforcementLearning2022, guo2023provably, zhang2023provable}, which assumes that the latent state can be decoded from the current observation (i.e., Block MDP), or an $m$-step history of observations. Another set of conditions is the ``observability'' condition~\parencite{golowich2022learning, golowichPlanningObservablePomdps2022a, golowich2023planning} and its cousin the ``weakly revealing'' condition~\parencite{jinSampleEfficientReinforcementLearning2020,liuWhenPartiallyObservable2022} which require different belief states to induce distinguishable distributions over observations. 
These conditions are further extended to POMDP models in the function approximation setting,  where the state or observation spaces are large and function approximators (e.g., linear functions) are used to represent the model. See, e.g.,~\citet{ueharaProvablyEfficientReinforcement2022, cai2022reinforcement,wangEmbedControlPartially2022, guo2023provably, zhang2023provable}. In our work, we build on the above by identifying a class of POSTs/POSGs that can be learned efficiently. This is significant since POSTs/POSGs are much more general than POMDPs and do not assume the existence of a latent state.

\textbf{Predictive state representations.} Predictive state representations were introduced by \citet{littmanPredictiveRepresentationsState} building on prior work on observable operator models by \citet{jaegerObservableOperatorModels2000} which proposed the idea of predictive representations as an alternative to belief states for modeling HMMs and POMDPs~\parencite[see also][]{singhLearningPredictiveState2003,singhPredictiveStateRepresentations2004,jamesPlanningPredictiveState2004,mccrackenOnlineDiscoveryLearning2005}. PSRs are a way to \textit{represent} the dynamics of a sequential decision-making problem by modeling the (conditional) probabilities of a small set of future trajectories, typically called ``core tests''. In a PSR, the probability of any future trajectory is a deterministic function of the conditional probabilities of the core tests. That is, the probabilities of the core tests encode all the information that the past contains about the future.~\citet{littmanPredictiveRepresentationsState} showed that any POMDP can be represented as a PSR. Various reinforcement learning methods for PSRs have been proposed under the assumption that data distribution is explorative, including spectral algorithms~\parencite{bootsClosingLearningplanningLoop2011,jiangCompletingStateRepresentations2018,zhangReinforcementLearningMultiagent2021} and supervised learning approaches~\parencite{hefnySupervisedLearningDynamical2015}. In addition, when it comes to the online setting where the algorithm needs to explore, there is a line of work that extends the theory and algorithms for online POMDP learning to PSRs. Moreover, some of these works propose generic theory and algorithms that can be applied to a large class of models including MDPs, POMDPs, and two-player zero-sum dynamic games with partial observability.
See, e.g.,~\citet{liuWhenPartiallyObservable2022,liuOptimisticMLEGeneric2022,zhanPACReinforcementLearning2022,chenPartiallyObservableRL2022,huangProvablyEfficientUCBtype2023, zhongGECUnifiedFramework2023, qiu2023posterior, liu2024maximize}. Our work is particularly related to the works that combine the idea of optimism in the face of uncertainty~\citep{sutton2018reinforcement} and maximum likelihood model estimation~\citep{liuWhenPartiallyObservable2022,liuOptimisticMLEGeneric2022,zhanPACReinforcementLearning2022,chenPartiallyObservableRL2022,huangProvablyEfficientUCBtype2023}. Specifically, our algorithm extends the UCB-type algorithm proposed by~\citet{huangProvablyEfficientUCBtype2023} for standard PSRs to learn a generalization of PSRs which captures POSTs/POSGs.

\textbf{Learning in multi-agent systems.} Most applications of interest in reinforcement learning involve the participation of multiple agents in the same environment. Empirical research has achieved striking success in several domains, including for example in the games of Go~\parencite{silverMasteringGameGo2016}, Starcraft~\parencite{vinyalsAlphastarMasteringRealtime2019}, and Poker~\parencite{brownSuperhumanAIMultiplayer2019}, as well as in robotic control~\parencite{koberReinforcementLearningRobotics2013,} and autonomous driving~\parencite{shalev-shwartzSafeMultiAgentReinforcement2016}. There also exists a growing literature of theoretical work. For example,~\citet{brafmanRmaxaGeneralPolynomial2002,baiNearoptimalReinforcementLearning2020,songWhenCanWe2021} tackle learning in Markov games (MGs)---a generalization of single-agent MDPs that assumes the existence of a Markovian state which is observable by all agents. Another model that has been explored in the literature is imperfect-information extensive-form games (IIEFG), which assume tree-structured transitions and deterministic emission, and can be viewed as a subclass of partially-observable Markov games (POMGs). Learning under this model has been studied in~\citet{zinkevichRegretMinimizationGames2007,kozunoModelfreeLearningTwoplayer2021,farinaModelfreeOnlineLearning2021}. More recently,~\citet{liuSampleEfficientReinforcementLearning2022} studied reinforcement learning in POMGs using an MLE-based algorithm building on their previous work in the single-agent setting~\parencite{liuWhenPartiallyObservable2022}. To address the computational intractability of the planning step for such model-based algorithms,~\citet{liu2024partially} proposed a quasi-efficient algorithm for multi-agent POMGs that runs in quasi-polynomial time with quasi-polynomial sample complexity. Their proposed algorithm leverages the common information approach~\parencite[see][]{nayyarCommonInformationApproachDecentralized2014} to construct an approximate Markov game where the state space of this new game corresponds to the space of approximate common information among agents. The idea of leveraging an information-sharing structure in multi-agent reinforcement learning has also appeared in~\citet{subramanian2022approximate, mao2020information, kara2022near, kao2022common, tang2023novel}.


In each of the above-mentioned models (e.g., MDP, POMDP, MG, IIEFG, POMG, etc.), a particular fixed information structure is assumed. We emphasize that the POST/POSG model proposed in our work allows the information structure to be specified arbitrarily and hence captures these models as special cases within a unifying framework. Moreover, our analysis and proposed algorithm significantly expand the class of multi-agent sequential decision-making problems that can be efficiently learned.

\subsection{Notation}

We use the convention that upper case letters denote random variables and lowercase letters denote realizations of those random variables (e.g., $X_t$ is the random variable denoting the state at time $t$ and $x_t \in \bbX_t$ is a particular realization). When clear from context, $\prob{x_t}$ means $\prob{X_t = x_t}$. We will tend to use blackboard symbols to denote the spaces that variables lie in (e.g., $\bbX_t$ for the space $X_t$ lies in) and calligraphic symbols to denote sets (e.g., $\calS$ for the indices of system variables). We use $\calP\paren{\bbX}$ to denote the space of probability measures on $\bbA$ and $\calP\paren{\bbB \given \bbA}$ to denote the set of stochastic kernels from $\bbA$ to $\bbB$. $i:j$ denotes the set $\set{i, i+1, \ldots, j}$. $\sigma_k(A)$ denotes the $k$-th largest eigenvalue of $A$.

For measures $p, q$ over a (finite) set $\bbX$, we define the total-variation distance as $\tv{p, q} \coloneqq \sum_{x \in \bbX} \aabs{p(x) - q(x)}$ and the hellinger-squared distance as $\hellingersq{p, q} = \frac{1}{2} \sum_{x \in \bbX} (\sqrt{p(x)} - \sqrt{q(x)})^2$. Note that the conventional total variation distance usually has a factor of $1 / 2$ which we omit for convenience. For a vector $x$ and a symmetric positive semi-definite matrix $A$, $\nnorm{x}_A \coloneqq \sqrt{x^\top A x}$. We define the matrix norms $\nnorm{A}_p = \max_{\nnorm{x}_p=1} \nnorm{Ax}_p$, and $\norm{A}_{\max} = \max_{ij} \lvert A_{ij} \rvert$. $A^\dagger$ is the Moore-Penrose pseudo-inverse.


We refer the reader to~\Cref{sec:summary_notation} for a table summarizing all notation used throughout the paper.

\section{Generic Sequential Decision Making Problems and Generalized PSRs}\label{sec:preliminaries}

In this section, we formulate the \textit{generic} sequential decision-making problem. This is an unstructured model with full generality which is defined in terms of the probability distribution of a sequence of variables---it contains no description of how different variables relate to each other. We then introduce a generalization of predictive state representations for this generic model. This forms the backdrop for the models of sequential decision-making presented in~\Cref{sec:info_struct}, which contain an explicit representation of information structure.

\subsection{Generic Sequential Decision-Making Problems}\label{ssec:seq_dec_making}


Consider a controlled stochastic process $(X_1, \ldots, X_H)$, where $X_h$ is a random variable corresponding to the variable at time $h$. At each time $h \in [H]$, the variable $X_h$ may be either an `observation' (i.e., observable system variable) or an `action'. The dynamics of this stochastic process are described by a tuple $(H, \set{\bbX_h}_h, \calO, \calA, \bbP)$, where $H$ is the time horizon, $\bbX_h$ is the variable space at time $h$ (i.e., $X_h \in \bbX_h$), $\calO \subset [H]$ is the index set of observations (i.e., $X_h$ is an observation if $h \in \calO$), $\calA \subset [H]$ is the index set of actions, and $\bbP = \sset{\bbP_h}_{h \in \calO}$ is a set of probability kernels which describes the the probability of any trajectory $x_1, \ldots, x_H$ given that the actions are executed,
\begin{equation}
    \prob{\sset{x_s : s \in \calO} \given \sset{x_s: s \in \calA}} = \prod_{h \in \calO} \bbP_h \bra{x_h \given x_1, \ldots, x_{h-1}}.
\end{equation}

A choice of policy $\pi = \sset{\pi_h}_{h \in \calA}$ induces a probability distribution on $\bbX_1 \times \cdots \times \bbX_H$ as follows
\begin{equation}
    \bbP^\pi \paren{x_1, \ldots, x_H} = \prod_{h \in \calO} \bbP_h \paren{x_h \given x_1, \ldots, x_{h-1}} \cdot \prod_{h \in \calA} \pi_h \paren{x_h \given x_1, \ldots, x_{h-1}}.
\end{equation}



We now define some notation. Let $\bbH_h = \prod_{s \in 1:h} \bbX_s$ denote the space of histories at time $h$ and $\bbF_h = \prod_{s \in h+1:H} \bbX_s$ denote the space futures at time $h$. Similarly, let $\bbH_h^o = \obs(\bbH_h) = \prod_{s \in \calO_{1:h}} \bbX_s$ denote the observation component of histories and let $\bbH_h^a = \act(\bbH_h) = \prod_{s \in \calA_{1:h}} \bbX_s$ denote the action component. Here, $\calO_{i:j}$ denotes $\calO \cap \sset{i, \ldots, j}$, and similarly for $\calA_{i:j}$. The observation and action components of the futures, $\bbF_h^o$ and $\bbF_h^a$ respectively, are defined similarly.

We define the \textit{system dynamics matrix} $\bm{D}_h \in \reals^{\abs{\bbH_h} \times \abs{\bbF_h}}$ as the matrix giving the probability of each possible pair of history and future at time $h$ given the execution of the actions,
\begin{equation}\label{eq:dynamics_matrix}
    \bra{\bm{D}_h}_{\tau_h, \omega_h} = \probbar{\tau_h, \omega_h} = \prob{\tau_h^o, \omega_h^o \given \mathrm{do}(\tau_h^a, \omega_h^a)}, \quad \tau_h \in \bbH_h, \omega_h \in \bbF_h,
\end{equation}
where $\omega_h^o = \obs(\omega_h)$ are is the observation component of the future $\omega_h$, $\omega_h^a = \act(\omega_h)$ is the action component, and similarly for $\tau_h^o, \tau_h^a$. Note that the actions are actively executed via the $\mathrm{do}$-operation. Hence, the system dynamics matrices are independent of any action-selection criteria. Note that $\bm{D}_H \in \reals^{\abs{\bbH_H} \times 1}$ is defined as $[\bm{D}_H]_{\tau_H} = \probbar{\tau_H}$, and $\bm{D}_0 = \bm{D}_H^\top$.

We introduce the notion of the \textit{rank} of the dynamics. The rank of such a controlled stochastic process is the maximal rank of its dynamics matrices. This is a measure of the complexity of the dynamics.
\begin{definition}[Rank of dynamics]\label{def:rank_dynamics}
    The rank of the dynamics $\set{\bm{D}_h}_{h \in [H]}$ is $r = \max_{h \in [H]} \mathrm{rank}(\bm{D}_h)$.
\end{definition}

This defines the dynamics of the system. A sequential decision-making problem is such a controlled stochastic process together with an \textit{objective}. The objective is defined by a reward function $R: \bbX_1 \times \cdots \times \bbX_H \to [0,1]$ mapping a trajectory to a reward in $[0,1]$. The agent(s) can affect the dynamics of the system through their choice of actions or policies. Each action $X_h, h \in \calA$ may be chosen by either a single agent or one of several agents (e.g., a team). The policy at time $h \in \calA$ is a mapping $\pi_h : \bbH_{h-1} \to \calP(\bbX_h)$ from previous observations to an action (or a distribution over actions, if randomized). The collection of policies at all time steps is denoted $\bm{\pi} = (\pi_h : h \in \calA)$, and induces a probability distribution over trajectories, denoted $\bbP^{\bm{\pi}}$. Then, the value of a policy $\bm{\pi}$ is the expected value of the reward under the measure $\bbP^{\bm{\pi}}$, $V^R(\bm{\pi}) \coloneqq \expectunder{\bm{\pi}}{R(X_1, \ldots, X_H)}$, where $\bbE^{\bm{\pi}}$ is the expectation associated with $\bbP^{\bm{\pi}}$.

The formalism of sequential decision-making problems introduced in this section is highly generic, but does not explicitly model the \textit{information structure}. In the next section, we introduce the models of \textit{partially observable sequential teams/games}, which explicitly represent information structures. We then show that the information structure characterizes the rank of a sequential decision-making problem as per~\Cref{def:rank_dynamics}.
\subsection{Generalized Predictive State Representations}\label{ssec:gen_psr}

Predictive state representations (PSR)~\parencite{littmanPredictiveRepresentationsState,jaegerObservableOperatorModels2000} are a model of dynamical systems and sequential decision-making problems based on predicting future observations given the past, without explicitly modeling a latent state.
In this section, we propose and formalize a generalization of standard PSRs.

In the standard formulation of sequential decision-making and predictive state representations, the sequence of variables is such that observations and actions always occur in an alternating manner (i.e., $o_{h}, a_{h}, o_{h+1}, a_{h+1}, \ldots$). The POST/POSG models we will propose are more general, and hence require a more flexible formalization of PSRs which allows for arbitrary order of observations and actions as well as arbitrary variable spaces at each time point. This generalization of PSRs will be used in our reinforcement learning algorithms.

The ``PSR rank'' of a sequential decision-making problem coincides with the rank of its dynamics, as defined in~\Cref{def:rank_dynamics}. Recall that the system dynamics matrix $\bm{D}_h \in \reals^{\abs{\bbH_h} \times \abs{\bbF_h}}$ is indexed by all possible observable histories $\tau_h$ and futures $\omega_h$. Denote the rank of the system dynamics at time $h$ by $r_{h} := \mathrm{rank}(\bm{D}_h)$.

Consider a sequential decision-making problem as defined in~\Cref{ssec:seq_dec_making} (i.e., with an arbitrary order of observations and actions, and arbitrary variable spaces). At the heart of predictive state representations is the concept of ``core test sets.'' A core test set at time $h$ is a set of futures such that the set of probabilities of those futures conditioned on the past encodes all the information that the past contains about the future. This is formalized in the definition below as a set of futures such that the submatrix of the full dynamics matrix restricted to those futures is full rank.

\begin{definition}[Core test sets]\label{def:core_test_sets}
    A core test set at time $h$ is a subset of $d_h \geq r_h$ futures, $\bbQ_h \coloneqq \sset{q_h^1, \ldots, q_h^{d_h}} \subset \bbF_h$, such that the submatrix $\bm{D}_h[\bbQ_h] \in \reals^{\aabs{\bbH_h} \times d_h}$ is full-rank, $\mathrm{rank}(\bm{D}_h[\bbQ_h]) = \mathrm{rank}(\bm{D}_h) = r_h$.
\end{definition}

A core test set implies the existence of a matrix $\bm{W}_h \in \reals^{\abs{\bbF_h} \times d_h}$ such that $\bm{D}_h = \bm{D}_h[\bbQ_h] \cdot \bm{W}_h^\top$.

Denote the $\tau_h$-th row of $\bm{D}_h[\bbQ_h]$ by
\begin{equation}\label{eq:def_psi}
    \psi_h(\tau_h) := \paren{\probbar{\tau_h, q_h^1}, \ldots, \probbar{\tau_h, q_h^{d_h}}} \in \reals^{d_h}.
\end{equation}
The vector $\psi_h(\tau_h)$ is a sufficient statistic for the history $\tau_h$ in predicting the probabilities of all futures conditioned on $\tau_h$. This is sometimes called the \textit{prediction features} of a history $\tau_h$.

For any integer $d_h \geq r_h$, there exists a core test set of size $d_h$. In particular, for any low-rank sequential decision-making problem, there exists a minimal core test set of size $r_h$ at each $h$. However, the minimal core test set depends on the system dynamics matrix $\bm{D}_h$, which is unknown in the learning setting. In the literature on reinforcement learning in PSRs, it is typically assumed that a core test set is known.
We address the problem of constructing a PSR representation for POSTs/POSGs in~\Cref{sec:post_psrep}.

For a core test set $\bbQ_h$, let $\bbQ_h^A = \set{\mathrm{act}(q) : q \in \bbQ_h}$, where $\mathrm{act}(q)$ denotes the action components of the test $q \in \bbQ_h$. Let $Q_A = \max_h \abs{\bbQ_h^A}$ and $d = \max_h d_h$.

With core test sets defined, we are now ready to present the definition of a generalized predictive state representation. The essential element in a PSR is a set of operators $M_h: \bbX_h \to \reals^{d_{h} \times d_{h-1}}$ for each time point $h \in [H]$. Given the prediction features at time $h-1$, $\psi_{h-1}(x_1, \ldots, x_{h-1}) \in \reals^{d_{h-1}}$, the linear map $M_h(x_h)$ computes the prediction features at time $h$, incorporating the additional observation $x_h$. That is, $\psi_h(x_1, \ldots, x_h) = M_h(x_h) \psi_{h-1}(x_1, \ldots, x_{h-1})$. The full definition is given below.

\begin{definition}[Generalized Predictive State Representations]\label{def:gen_psr}
    Consider a sequential decision-making problem $\paren{X_h \in \bbX_h}$ where $\calA, \calO$ partition $[H]$ into actions and observations, respectively. Then, a predictive state representation of this sequential decision-making problem is a tuple $\theta = \pparen{\sset{\bbQ_h}_{0 \leq h \leq H-1}, \phi_H, \bm{M}, \psi_0}$ given by
    \begin{enumerate}
        \item $\set{\bbQ_h}_{0 \leq h \leq H-1}$ are core test sets, including for $h = 0$, where $\bbQ_0 = \sset{q_0^1, \ldots, q_0^{d_0}} \subset \bbF_0$ are core tests before the system begins.
        \item$\psi_0 \in \reals^{d_0}$ is the vector $\psi(\emptyset) =  \pparen{\Probbar{q_0^1}, \ldots, \Probbar{q_0^{d_0}}}$.
        \item $\bm{M} = \sset{M_h}_{1 \leq h \leq H-1}$ is a set of mappings $M_h : \bbX_{h} \to \reals^{d_{h} \times d_{h-1}}$, from an observation/action to a matrix of size $d_{h} \times d_{h-1}$.
        \item $\phi_H : \bbX_{H} \to \reals^{d_{H-1}}$ is a mapping from the final observation to a $d_{H-1}$-dimensional vector.
    \end{enumerate}

    This tuple satisfies
    \begin{align}\label{eq:psr_def_M_psi}
        \probbar{x_1, \ldots, x_H}  &= \phi_{H}(x_{H})^\top M_{H-1}(x_{H-1}) \cdots M_1(x_{1}) \psi_0 \\
        \psi_h(x_1, \ldots, x_h) &= M_h(x_{h}) \cdots M_1(x_{1}) \psi_0, \, \forall h
    \end{align}
\end{definition}

To obtain a probability for a trajectory $\tau_h = (x_1, \ldots, x_h)$, with $h < H$, note that $\sum_{\omega_h \in \bbF_h} \probbar{\tau_h, \omega_h} = \abs{\bbF_h^a} \probbar{\tau_h}$. Hence,
\begin{align*}
    \probbar{\tau_h} &= \frac{1}{\abs{\bbF_h^a}} \sum_{\omega_h} \probbar{\tau_h, \omega_h} \\
    &= \frac{1}{\prod_{s \in h+1:H} \abs{\bbX_s}^{\bm{1}\{s \in \calA\}}} \sum_{x_H} \cdots \sum_{x_{h+1}} \phi_H^\top M_H(x_H) \cdots M_{h+1}(x_{h+1}) \psi_h(\tau_h).
\end{align*}
Thus, if we recursively define $\phi_h,\, h < H$ via,
\begin{equation}
    \frac{1}{\abs{\bbX_{h}}^{\bm{1}\{h \in \calA\}}} \sum_{x_{h}} \phi_{h}^\top M_{h}(x_{h}) = \phi_{h-1}^\top,
\end{equation}
with $\phi_H$ as the terminating condition, then, we can obtain $\probbar{\tau_h}$ for any $h < H$, via an inner product between $\phi_h$ and $\psi_h(\tau_h)$,
\begin{equation}
    \probbar{\tau_h} = \phi_h^\top \psi_h(\tau_h).
\end{equation}
Finally, if we define $\bar{\psi}_h(\tau_h) = \psi_h(\tau_h) / \Probbar{\tau_h}$, then we obtain the \textit{conditional} probability of the core tests given the history, $\bar{\psi}_h(\tau_h) = \pparen{\Probbar{q_h^1 \,|\, \tau_h},\, \ldots,\, \Probbar{q_h^{d_h} \,|\, \tau_h}} \in \reals^{d_h}$. $\bar{\psi}_h(\tau_h)$ is known as the (normalized) prediction feature of the history $\tau_h$~\parencite{littmanPredictiveRepresentationsState}.

\begin{remark}[Generality and difference from standard PSRs]
    In standard PSRs, observations and actions are assumed to occur in an alternating manner, and hence observable operators are defined on \textit{pairs} of observations and actions (i.e., $M_h(o_h, a_h)$). This structure leads to a somewhat simpler description compared to the above. However, our formulation is more general, as it allows each variable to be treated independently, and allows for an arbitrary sequence of variables with arbitrary spaces. This generality will be needed when modeling problems with an explicit representation of information structure.
\end{remark}

An important condition for the learnability of PSR models, which was used in prior work~\parencite[including][]{huangProvablyEfficientUCBtype2023,liuOptimisticMLEGeneric2022}, is the so-called ``well-conditioning assumption''. We state the analogous assumption for our generalized PSR model below.

\begin{assumption}[$\gamma$-well-conditioned generalized PSR]\label{ass:psr_gamma_wellcond}
    A PSR model $\theta = \paren{\set{\bbQ_h}_{0 \leq h \leq H-1}, \phi_H, \bm{M}, \psi_0}$, as defined in~\Cref{def:gen_psr}, is said to be $\gamma$-well conditioned for $\gamma > 0$ if it satisfies
    \begin{enumerate}
        \item For any $h \in [H]$,
        \begin{equation}\label{psr_gamma_wellcond_cond1}
            \max_{\substack{z \in \reals^{d_h}\\ \onenorm{z} \leq 1}} \max_\pi \sum_{\omega_h \in \bbF_h} \pi(\omega_h) \abs{m_h(\omega_h)^\top z} \leq \frac{1}{\gamma},
        \end{equation}
        where $m_h(\omega_h)^\top = \phi_H(x_H)^\top M_{H-1}(x_{H-1}) \cdots M_{h+1}(x_{h+1})$ with $\omega_h = (x_{h+1}, \ldots, x_H) \in \bbF_h$. The maximization is over policies $\pi$ such that for any fixed future observations $\omega_h^o$, $\sum_{\omega_h^a} \pi(\omega_h^o, \omega_h^a) = 1$.
        \item For any $h \in [H-1]$,
        \begin{equation*}
            \max_{\substack{z \in \reals^{d_h}\\ \onenorm{z} \leq 1}} \sum_{x_h \in \bbX_h} \onenorm{M_h(x_h) z} \pi(x_h) \leq \frac{\abs{\bbQ_{h+1}^A}}{\gamma},
        \end{equation*}
        where $\pi(x_h) = 1$ when $h \notin \calA$ and $\sum_{x_h} \pi(x_h) = 1$ when $h \in \calA$.
    \end{enumerate}
\end{assumption}

To understand this condition, recall that $m_h(\omega_h)^\top \psi_h(\tau_h) = \probbar{\tau_h, \omega_h}$. We may think of $z$ in~\Cref{ass:psr_gamma_wellcond} as representing the error in estimating $\psi_h(\tau_h)$, the probabilities of core tests at time $h$ given the history $\tau_h$. The $\gamma$-well-conditioned assumption ensures that the error in estimating the overall PSR (i.e., the probability of a particular trajectory) does not blow up when the estimation error of $\psi_h(\tau_h)$ is small.

The following result states that any sequential decision-making problem of the form described in~\Cref{ssec:seq_dec_making} admits a generalized PSR representation. The proof and explicit construction are given in~\Cref{sec:appdx_genpsr_exists}.
\begin{proposition}\label{prop:seq_decmaking_exist_psr}
    Let $(X_1, \ldots, X_H)$ be any sequential decision-making problem with observation index set $\calO$, action index set $\calA$, and variable spaces $\set{\bbX_h}_{h \in [H]}$. Let $r_h = \mathrm{rank}(\bm{D}_h)$, where $D_h, h \in [H]$ are the system dynamics matrices. Then, there exists a PSR representation $\psi_0$, $\phi_H: \bbX_H \to \reals^{r_{H-1}}$, $M_h: \bbX_h \to \reals^{r_{h+1} \times r_h}, h \in [H-1]$, satisfying~\Cref{def:gen_psr}.
\end{proposition}
\begin{proof}
    The proof is given in~\Cref{sec:appdx_genpsr_exists}.
\end{proof}

\section{Information Structure}\label{sec:info_struct}

The ``information structure'' of a dynamical system describes how events in a system occurring at different points in time affect each other, whether those events are observable by the learning agent or not. In this section, we will introduce novel reinforcement learning models that explicitly represent information structure as part of the problem specification. We will show that this enables a rich analysis of the dynamics of the system, and is crucial for characterizing the statistical hardness of general reinforcement learning problems. In~\Cref{ssec:post} we propose the partially-observable sequential teams model; in~\Cref{ssec:posg} we extend the model to the game setting; in~\Cref{ssec:info_struct_rank} we show that the information structure characterizes the complexity of the observable dynamics of a sequential decision-making problem; and in~\Cref{ssec:info_struct_rank_examples} we use this result to analyze common models of sequential decision-making problems.

\subsection{Partially-Observable Sequential Teams}\label{ssec:post}

In~\Cref{ssec:seq_dec_making}, we introduced the generic sequential decision-making problem, which modeled the dynamics of observable variables via probabilities of the form $\bbP_h(x_h | x_1, \ldots, x_{h-1})$. This model lacked a representation of the structure that determines how different variables affect each other, including potentially unobservable system variables. For example, in a POMDP, there crucially exists an unobserved Markovian state which injects useful structure in the observable dynamics. In general, this behavior is captured by the notion of ``information structures''.

The information structure of a dynamical system describes how events in a system occurring at different points in time affect each other, whether those events are observable by the learning agent or not. In this section, we propose partially-observable sequential teams as a reinforcement learning model that explicitly represents information structure. This is a general model that captures structured models such as POMDPs as a special case and enables a richer analysis of general sequential decision-making problems.

A POST is a controlled stochastic process consisting of a sequence of variables, where each variable is either a ``system variable'' or an ``action variable''. Crucially, the POST model includes an explicit representation of information structure. POSTs also model the \textit{observability} of each system variable with respect to the learning algorithm (i.e., which system variables are available to the learning algorithm). Unlike more specialized models of sequential decision-making, there is no restriction on the order of system variables and action variables (e.g., don't need to be alternating). The information structure of a POST describes the dependence between these variables. The ``information set'' of a system variable describes the subset of past variables that directly affect it. The information set of an action variable describes the information available to the agent when choosing an action, hence defining the policy class they optimize over. A formal definition follows below.

\begin{definition}[Partially-Observable Sequential Team Model]\label{def:partially_obs_seq_team}
    A partially-observable sequential team (POST) is a controlled stochastic process that specifies the joint distribution of $T$ variables $\paren{X_t}_{t\in [T]}$, together with a specification of the observability of each variable. Here each $X_t$ is either a system variable or an action variable, and is either observable by the learning agent or not. A partially-observable sequential team is specified by the following components.
    \begin{enumerate}
        \item \textbf{Variable Structures.}
        The variables $\sset{X_t}_{t \in [T]}$ are partitioned into two disjoint subsets --- system variables and action variables. $\calS \subset [T]$ indexes system variables and $\calA \subset [T]$ indexes action variables, with $\calS \cap \calA = \emptyset,\, \calS \cup \calA = [T]$.
        \item \textbf{Variable Spaces.}
        Let $\bbX_t$ be the space that the variable $X_t$ takes values in, which is assumed to be finite for all $t \in [T]$.
        \item \textbf{Information Structure.}
        For $t \in [T]$, the ``information set'' $\calI_t \subset [t-1]$ of the variable $X_t$ is the set of past variables that are coupled to $X_t$ in the dynamics. That is, the value of $I_t := \paren{X_s \,:\, s\in \calI_t}$ directly determines the distribution of $X_t$. We call $I_t$ the ``information variable'' at time $t$, and call $\bbI_t = \prod_{s \in \calI_t} \bbX_s$ the ``information space''. We denote realizations of $I_t$ by $i_t = \paren{x_s \in \bbX_s \,:\, s \in \calI_t} \in \bbI_t$.
        \item \textbf{System Kernels.}
        For any $t\in \calS$,
        $\calT_t$ is a mapping
        from $\bbI_{t}$ to $\calP(\bbX_t)$ that specifies the conditional distribution of a system variable $X_t$ given $I_t$.
        That is,  $X_t \sim \calT_t(\cdot | \set{X_s, s \in \calI_t})$ for all $t \in \calS$. If $\calI_t = \emptyset$ then $\calT_t$ is simply a (unconditional) distribution on $\bbX_t$.
        \item \textbf{Decision Kernels.} Each agent chooses a decision kernel (i.e., policy) $\pi_t: \bbI_t \to \calP(\bbX_t)$, specifying the distribution over actions at time $t \in \calA$. That is, the action variable $X_t$ at time $t \in \calA$ satisfies $X_t \sim \pi_t(\cdot | \set{X_s, s \in \calI_t})$. The joint policy is denoted by $\pi = \paren{\pi_t}_{t \in \calA}$.
        \item \textbf{Observability.} We denote the observable system variables by $\calO \subset \calS$. We require that the information sets of the action variables are observable, $\calO \supset \cup_{t \in \calA} \pparen{\calI_t \cap \calS}$. We define $\calU \coloneq \calO \cup \calA$, and let $H \coloneq \abs{\calU}$ be the time-horizon of the \textit{observable} variables (including actions).
        \item \textbf{Reward Function.} At the end of an episode, the team receives the reward $R\paren{x_s,\ s \in \calU}$, where $R: \prod_{s \in \calU} \bbX_s \to [0,1]$ is the ``reward function.''
    \end{enumerate}
\end{definition}

With the above components, any set of decision kernels (joint policy) $\bm{\pi}$ induces a unique probability measure over $\bbX_{1} \times \cdots \times \bbX_T$, which is  given by
\begin{equation}\label{eq:trajectory_prob}
    \probunder{\bm{\pi}}{X_1 = x_1, \ldots X_T = x_t} = \prod_{t \in \calS} \calT_t(x_t | \set{x_s : s \in \calI_t}) \prod_{t \in \calA} \pi_t(x_t | \set{x_s : s \in \calI_t}).
\end{equation}

We will be interested in modeling the \textit{observable} dynamics of the POST. We index the observable variables by their order among observables $h \in [H]$ rather than their order among all variables as follows,
\begin{equation}\label{eq:observable_index}
    \paren{X_{t(h)}}_{h \in [H]} = \paren{X_{t(1)}, \ldots, X_{t(H)}} = \paren{X_t}_{t \in \calU},
\end{equation}
where $t: [H] \to \calU$ maps the index over observables to the index over all variables. That is, $t(1)$ is the index of the first observable, $t(2)$ is the index of the second observable, etc. The distribution of the observables is obtained by marginalizing over the unobservable variables,
\begin{equation}
    \begin{split}
        \probunder{\pi}{X_{t(1)} = x_{t(1)}, \ldots, X_{t(H)} = x_{t(H)}} &= \sum_{\substack{x_s \in \bbX_s \\ s \in \calO^\complement}} \probunder{\bm{\pi}}{X_1 = x_1, \ldots X_T = x_t} \\
        &= \sum_{\substack{x_s \in \bbX_s \\ s \in \calO^\complement}} \prod_{t \in \calS} \calT_t(x_t | \set{x_s : s \in \calI_t})  \prod_{t \in \calA} \pi_t(x_t | \set{x_s : s \in \calI_t}).
    \end{split}
\end{equation}

The value of a policy is given by its expected reward,
\begin{equation}
    V(\bm{\pi}) \coloneqq \expectunder{\bm{\pi}}{R(X_{t(1)}, \ldots, X_{t(H)})},
\end{equation}
where $\bbE^\pi$ is the expectation associated with the probability measure $\bbP^\pi$. The objective of a POST is to learn a policy $\pi = (\pi_t)_{t \in \calA}$ which maximizes the expected reward,
\begin{equation*}
    \sup_{\substack{\pi_t \in \calP\paren{\bbX_t \given \bbI_t} \\ t \in \calA}} \expectunder{\bm{\pi}}{R(X_{t(1)}, \ldots, X_{t(H)})}.
\end{equation*}
When the variable spaces $\bbX_t$ are finite, this supremum is attained by a deterministic policy, $\pi = \pparen{\pi_t, t \in \calA}, \ \pi_t \colon \bbI_t \to \bbX_t$.


\textbf{Modeling simultaneous events.}
The POST model is highly versatile and can model events as occurring either in sequence or simultaneously, with an arbitrary dependence on the past. This is controlled by the specification of the information sets. For example, to represent $m$ events occurring simultaneously, the corresponding variables can occupy any ordering of consecutive time points, $X_{t+1}, \ldots, X_{t+m}$, as long as their information sets do not contain any of the other variables occurring at that time (i.e., $\calI_{s} \subset [t]$ for all $s \in \set{t+1, \ldots, t+m}$). For example, agents may act simultaneously and observations may be emitted simultaneously, as occurs in models like Dec-POMDPs or POMGs. However, POSTs further allow for an irregular sequence of observations and actions, as occurs in many real-world scenarios.

\textbf{Representing the identity of the agent taking each action via the information structure.}
In~\Cref{def:partially_obs_seq_team} we do not need to label each action with the agent that executes it since this can be captured by the information structure. In particular, the POST model does not need to distinguish between agents and actions. This is without loss of generality since the underlying `identity' of an agent (i.e., the same agent acting multiple times and remembering their past observations) can be captured by the information structure. For example, the information sets can be specified in such a way so that for any $t \in \calA$, $\calI_t$ contains all variables that were observed by this agent in the past. Some examples are given in~\Cref{ssec:info_struct_rank_examples}. In the game setting, the identity of the agent needs to be modeled explicitly since it also determines the reward function associated with each action. We discuss this in~\Cref{ssec:posg}. We highlight that the generality in specifying information structures makes POSTs/POSGs powerful models with the ability to capture the complexity of real-world multi-agent systems.

\textbf{Representation of the information structure as a directed acyclic graph.}
The information structure of a POST can be naturally represented as a (labeled) directed acyclic graph (DAG). Given the variable structure and information structure of a POST, $\paren{\calS, \calA, \calO, \set{\calI_t}_{t}}$, its DAG representation is given by $\calG(\calV, \calE, \calL)$. The nodes of the graph are the set of variables, $\calV = [T] = \calS \cup \calA$. The edges $\calE \subset \calV \times \calV$ of the DAG are given by
\begin{equation*}
    \calE = \set{(i, t): t \in [T], i \in \calI_t}.
\end{equation*}

That is, there exists an edge from $i$ to $t$ if $i$ is in the information set of $t$. Finally, $\calL$ contains labels for each node as being a system variable (in $\calS$) or an action variable (in $\calA$). Further, the observability of system variables is also labeled. This DAG represents a (directed) \textit{graphical model} for the POST. In particular, the probability distribution on $\bbX_1 \times \cdots \times \bbX_T$ factors according to $\calG$,
\begin{equation}\label{eq:traj_dag}
    \prob{X_1, \ldots, X_T} = \prod_{t \in \calV} \prob{X_t \given \mathrm{pa}(X_t)},
\end{equation}
where $\mathrm{pa}(X_t)$ is the set of parents of $X_t$ in $\calG$ (which are $\calI_t$), and the probability $\prob{X_t \given \mathrm{pa}(X_t)}$ is given by a system kernel if $t \in \calS$ and a decision kernel if $t \in \calA$. This representation of the information structure as a DAG will be crucial for our analysis of the observable dynamics of POSTs in~\Cref{ssec:info_struct_rank}.

\textbf{POSTs within the taxonomy of decentralized control.}
In the control literature, there exists a taxonomy of decentralized systems. The model presented here falls within the class of dynamic sequential teams, and allows for non-classical information structures. This model is closely related to Witsenhausen's intrinsic model~\parencite{witsenhausenStandardFormSequential1973,witsenhausenIntrinsicModelDiscrete1975a,witsenhausenEquivalentStochasticControl1988}.
The intrinsic model has been studied extensively in the control literature, including for example in~\citep{mahajanGraphicalModelingApproach2009} where graphical modeling techniques are used to identify reduced classes of optimal policies. The main difference between POSTs and the intrinsic model is the introduction of a description of the ``observability'' of each system variable, which enables studying partially-observable reinforcement learning. The POST model can capture multiple agents acting in arbitrary environments, as long as the order in which agents act is predetermined and independent of the system dynamics (hence the name ``sequential'').  To our knowledge, general models with an explicit representation of information structure have so far not been considered in the reinforcement learning setting. In this work, we study the role of information structure in reinforcement learning through our novel POST/POSG models.

\textbf{Notation: observable futures, histories, and dynamics.} Finally, we introduce some notation that allows us to cast POSTs in terms of the generic sequential decision-making model of~\Cref{ssec:seq_dec_making}. We define the set of histories at time $h$ as $\bbH_h \coloneqq \prod_{s \in \calU_{1:h}} \bbX_s$, and the set of futures at time $h$ as $\bbF_h \coloneqq \prod_{s \in \calU_{h+1:H}} \bbX_s$. A history $\tau_h \in \bbH_h$ takes the form $\tau_h = \pparen{x_s \in \bbX_s \,\colon\, s \in \calU_{1:h}}$, and a future $\omega_h \in \bbF_h$ takes the form $\omega_h = \pparen{x_s \in \bbX_s \,\colon\, s \in \calU_{h+1:H}}$. We separate the actions from other observations via $\tau_h^o = \obs(\tau_h)$, $\tau_h^a = \act(\tau_h)$, $\omega_h^o = \obs(\omega_h)$, $\omega_h^a = \act(\omega_h)$. Recall that $\obs(\cdot)$ extracts the observation component of a trajectory and $\act(\cdot)$ extracts the action component (e.g., $\tau_h^o = (x_s, s \in \calO_{1:h}))$.  We denote the observation and action components of the histories as $\bbH_h^o \coloneqq \obs(\bbH_h)$, $\bbH_h^a \coloneqq \act(\bbH_h)$, respectively, and define $\bbF_h^o, \bbF_h^a$ similarly. The (observable) system dynamics matrix of a POST is defined by,
\begin{equation}\label{eq:seq_team_dynamics_mat}
    \begin{split}
        \bra{\bm{D}_h}_{\tau_h, \omega_h} &:= \prob{\tau_h^o, \omega_h^o \given \mathrm{do}(\tau_h^a, \omega_h^a)} \\
        &\equiv \sum_{\substack{x_s \in \bbX_s\\ s \in \calO^\complement}} \prod_{t \in \calS} \calT_t\paren{x_t \given \set{x_i, i \in \calI_t}}.
    \end{split}
\end{equation}

By introducing a model with an explicit representation of information structure, we gained the ability to perform a richer analysis of the dynamics of Sequential decision-making problems. In particular, we will show that the ``complexity'' of the dynamics can be related to a graph-theoretic analysis of the information structure. First, we extend the model to the game setting.

\subsection{Partially-Observable Sequential Games}\label{ssec:posg}
In a POST, all agents share the same objective. In the game setting, different agents may have different objectives which compete with each other in interesting ways. Information structures play a crucial role in the study of games. The information available to one agent when making its decisions, compared to the information available to competing agents, determines how well it can achieve its objective. In particular, the information structure of a problem determines the set of equilibria it admits. There has been a plethora of work in the game theory community studying such problems. 

Analogously to partially-observable sequential teams, we define partially-observable sequential \textit{games} (POSGs). The dynamics of a POSG are identical to a POST, with the same formalization of variable structure, variable spaces, information structure, system kernels, and decision kernels. In contrast to a POST, agents in a POSG may have different objectives. In a POSG, there exists $N$ agents, with agent $i \in [N]$ deciding the actions at times $t \in \calA^{i}$, where $\calA^i \subset \calA$. Each agent has its own objective defined by a reward function $R^i$. This is defined formally below.

\begin{definition}[Partially-Observable Sequential Game Model]\label{def:partially_obs_seq_game}
    A partially-observable sequential game (POSG) is a controlled stochastic process consisting of the following components: variable structure, variable spaces, information structure, system kernels, decision kernels, and observability. These are defined in an identical manner to~\Cref{def:partially_obs_seq_team}. Additionally, POSGs define a \textbf{reward structure} as follows. Let $N$ be the number of agents. Each agent may act several times. Denote by $\calA^{i} \subset \calA$ the index of action variables associated to agent $i \in [N]$. Each agent has a reward function $R^i: \prod_{t \in \calU} \bbX_t \to [0,1]$ which they aim to maximize.
\end{definition}

Denote by $\pi^i = (\pi_t : t \in \calA^i)$ the collection of decision kernels belonging to agent $i$, one for each action they take. Denote by $\bm{\pi} = (\pi^1, \ldots, \pi^N)$ the collection of all agents' policies. Fixing $\bm{\pi}$ induces a probability distribution over $\bbX_1 \times \cdots \bbX_T$ in the same way as in the team setting,
\begin{equation}\label{eq:trajectory_prob_game}
    \probunder{\bm{\pi}}{X_1 = x_1, \ldots X_T = x_t} = \prod_{t \in \calS} \calT_t(x_t | \set{x_s : s \in \calI_t}) \prod_{t \in \calA} \pi_t(x_t | \set{x_s : s \in \calI_t}).
\end{equation}

The value of a policy $\bm{\pi}$ for agent $i \in [N]$ is defined as the expected value of their reward $R^i$ under $\bbP^{\bm{\pi}}$,
\begin{equation}
V^i(\bm{\pi}) \equiv V^i(\pi^i, \bm{\pi}^{-i}) \coloneqq \expectunder{\bm{\pi}}{R^i(X_{t(1)}, \ldots, X_{t(H)})},
\end{equation}
where $\bm{\pi}^{-i} = (\pi^j : j \neq i)$.

The nature of randomization in agents' policies is crucial to the analysis of solution concepts in the game setting. To model randomized policies, which are potentially correlated, we introduce a random seed $\omega \in \Omega$ which is sampled at the beginning of an episode. Then, the policy at time $t \in \calA$ can be modeled as a deterministic function mapping the seed $\omega$ and information variable $i_t \in \bbI_t$ to an action $\bbX_t$. That is, $\pi_t : \Omega \times \bbI_t \to \bbX_t$. To model independently randomized policies with each agent having private randomness, we consider the special case where the seed has the product structure $\omega = (\omega_1, \ldots, \omega_N) \in \Omega_1 \times \cdots \times \Omega_N$, and $\omega_i$ is the seed belonging to agent $i \in [N]$. Then, for $t \in \calA^i$, $\pi_t : \Omega_i \times \bbI_t \to \bbX_t$. For each agent $i \in [N]$, define the three policy spaces,
\begin{enumerate}
    \item Deterministic policies, $\Gamma_{\mathrm{det}}^i = \set{\pi^i : \pi^i = \paren{\pi_t: \bbI_t \to \bbX_t, t \in \calA^i}}$,
    \item Independently-randomized policies, $\Gamma_{\mathrm{ind}}^i = \set{\pi^i :  \pi^i = \paren{\pi_t: \Omega_i \times \bbI_t \to \bbX_t, t \in \calA^i}}$,
    \item Correlated randomized policies, $\Gamma_{\mathrm{cor}}^i = \set{\pi^i :  \pi^i = \paren{\pi_t: \Omega \times \bbI_t \to \bbX_t, t \in \calA^i}}$.
\end{enumerate}

Define the joint deterministic policy space, as $\bm{\Gamma}_{\mathrm{det}} = \Gamma_{\mathrm{det}}^1 \times \cdots \times \Gamma_{\mathrm{det}}^N$, and similarly for the independently-randomized policy space $\bm{\Gamma}_{\mathrm{ind}}$, and the correlated randomized policy space $\bm{\Gamma}_{\mathrm{cor}}$.

When studying games, a common question is to find an \textit{equilibrium} within a particular policy space. At a high level, an equilibrium is a joint policy where no agent can do better by deviating from their policy when the other agents keep their policies fixed. We will consider several notions of equilibrium. We begin by defining the notion of a \textit{best-response}. Suppose that agent $i$'s policy space is $\Gamma^i$ (e.g., $\Gamma_{\mathrm{det}}^i$, $\Gamma_{\mathrm{ind}}^i$, or $\Gamma_{\mathrm{cor}}^i$). Then, we say that agent $i$'s policy $\pi^i$ is a best response to $\bm{\pi}^{-i}$ if there is no policy in $\Gamma^i$ which achieves a higher value. This is formalized in the definition below.

\begin{definition}[Best response]\label{def:best_response}
    For a joint policy $\bm{\pi}$, $\pi^i$ is said to be a best-response to $\bm{\pi}^{-i}$ in the policy space $\Gamma^i$ (e.g., $\Gamma_{\mathrm{det}}^i$, $\Gamma_{\mathrm{ind}}^i$, or $\Gamma_{\mathrm{cor}}^i$), if $V^i (\pi^i, \bm{\pi}^{-i}) = \max_{\tilde{\pi}^i \in \Gamma^i} V^i (\tilde{\pi}^i, \bm{\pi}^{-i}) =: V^{i, \dagger}(\bm{\pi}^{-i})$.
\end{definition}

This leads to the definition of two notions of equilibria. A \textit{Nash Equilibrium} (NE) is a joint policy where all agents are best-responding in the space of independently-randomized policies. A \textit{Coarse Correlated Equilibrium} (CCE) is a joint policy where all agents are best-responding in the space of correlated randomized policies. The difference between NE and CCE is that the randomness in the joint policy must be independent in an NE but can be correlated in a CCE. Since $\Gamma_{\mathrm{ind}} \subset \Gamma_{\mathrm{cor}}$, coarse correlated equilibria are a generalization of Nash equilibria. We define them formally below.

\begin{definition}[Nash Equilibrium]
    A joint policy $\bm{\pi} \in \Gamma_{\mathrm{ind}}$ is said to be a Nash equilibrium if for all agents $i \in [N]$, $V^i(\bm{\pi}) = \max_{\tilde{\pi}^i \in \Gamma_{\mathrm{ind}}^i} V^i (\tilde{\pi}^i, \bm{\pi}^{-i}) =: V^{i, \dagger}(\bm{\pi}^{-i})$. A joint policy $\bm{\pi} \in \Gamma_{\mathrm{ind}}$ is said to an $\varepsilon$-approximate Nash equilibrium if $V^i(\bm{\pi}) \geq  V^{i, \dagger}(\bm{\pi}^{-i}) - \varepsilon$ for all $i \in [N]$.
\end{definition}

\begin{definition}[Coarse Correlated Equilibrium]
    A joint policy $\bm{\pi} \in \Gamma_{\mathrm{cor}}$ is said to be a coarse correlated equilibrium if for all agents $i \in [N]$, $V^i(\bm{\pi}) = \max_{\tilde{\pi}^i \in \Gamma_{\mathrm{cor}}^i} V^i (\tilde{\pi}^i, \bm{\pi}^{-i}) =: V^{i, \dagger}(\bm{\pi}^{-i})$. A joint policy $\bm{\pi} \in \Gamma_{\mathrm{cor}}$ is said to an $\varepsilon$-approximate Nash equilibrium if $V^i(\bm{\pi}) \geq  V^{i, \dagger}(\bm{\pi}^{-i}) - \varepsilon$ for all $i \in [N]$.
\end{definition}

Since we consider finite-space sequential games, an equilibrium is guaranteed to exist~\parencite{nash1951non}.

\begin{remark}[Notion of equilibrium can be represented through information structure]
    The policy classes defined above (i.e., deterministic, independently-randomized, correlated randomized) can be directly modeled by the information structure. For example, to represent correlated randomized policies, the random seed $\omega \in \Omega$ can be modeled as an observable variable at time $t = 0$ which is in all agents' information sets. Similarly, independently randomized policies can be represented through a different random seed for each agent at time $t=0$, and including the appropriate random seed in each action's information set. Hence, the information structure itself can decide which equilibrium notion we are interested in. Moreover, this allows us to consider additional notions of equilibrium where, for example, only subsets of agents can be correlated with each other (e.g., this may be useful in modeling multi-team problems). Note that adding random seeds in order to model randomized policies does not affect the information-structural state $\calI_h^\dagger$ since the seeds don't appear in $\calG^\dagger$. For concreteness, we focus on NE and CCE in our presentation.
\end{remark}

\subsection{Information Structure Determines the Rank of POSTs/POSGs}\label{ssec:info_struct_rank}
For any sequential decision-making problem, the rank of the dynamics as defined in~\Cref{def:rank_dynamics} is a measure of the ``complexity'' of the observable dynamics. In the case of POSTs and POSGs, where the information structure is explicitly represented, we can explicitly characterize the rank of the dynamics as a function of the information structure. This provides a systematic means of identifying tractable sequential decision-making problems.

In this section, we will show that the information structure of POSTs/POSGs can be used to obtain a bound on the rank of the observables system dynamics matrices $\bm{D}_h$. This coincides with the PSR rank as shown in~\Cref{prop:seq_decmaking_exist_psr}, hence characterizing the complexity of the \textit{representation}. To motivate this, we recall the following classic result on the rank of POMDPs~\parencite[Theorem 1]{littmanPredictiveRepresentationsState}.

\begin{example*}[POMDPs have rank bounded by $\abs{\bbS}$]
    Consider a POMDP with states $s_t \in \bbS$, observations $o_t \in \bbO$, and actions $a_t \in \bbA$. The system dynamics are given by $\prob{s_{t+1}, o_{t+1} \given s_{1:t}, a_{1:t}, o_{1:t}} = \prob{s_{t+1} \given s_t, a_t} \prob{o_{t+1} \given s_{t+1}}$. We will derive a bound on the PSR rank of this partially observable system. For each history $\tau_t = \paren{o_1, a_1, \ldots, o_t, a_t}$ and a future $\omega_t = \paren{o_{t+1}, a_{t+1}, \ldots o_T, a_T}$, we have,
    \begin{equation*}
            \bm{D}_{\tau_t, \omega_t} = \prob{\tau_{t}^o, \omega_t^o \given \mathrm{do}(\tau_t^a, \omega_t^a)} = \sum_{s_{t+1} \in \bbS} \prob{\omega_{t} \given s_{t+1}} \prob{s_{t+1} \given \tau_t} \prob{\tau_t^o \given \tau_t^a}.
    \end{equation*}

    Hence, defining $\bm{D}_{t, 1} := \bra{\prob{\omega_{t} \given s_{t+1}}}_{\omega_{t}, s_{t+1}}$ and $\bm{D}_{t,2} := \bra{\prob{s_{t+1} \given \tau_t} \prob{\tau_t^o \given \tau_t^a}}_{s_{t+1}, \tau_t}$, we have that $\bm{D}_t = \bm{D}_{t,1} \bm{D}_{t,2}$. Thus, $\mathrm{rank}(\bm{D}_t) \leq \abs{\bbS}$ for all $t$. Hence, the rank of the observable dynamics of a POMDP is bounded by the number of states.

    \hfill $\square$
\end{example*}

In the above, the existence of a latent state implied a simplification of the system dynamics and a bound on the rank. We will use the same high-level idea to generalize the analysis to arbitrary sequential decision-making problems, bounding the rank of the observable system dynamics of POSTs and POSGs via their information structure. Our analysis relies on the literature of directed graphical models.~\citet{wrightMethodPathCoefficients1934} was the first to use DAGs to represent causal relationships. We refer the reader to~\citep{spirtesCausationPredictionSearch2000,pearlModelsReasoningInference2000,kollerProbabilisticGraphicalModels2009} for modern texts on the subject.

The main tools in our analysis will be the DAG representation $\calG$ of the information structure $\set{\calI_t, t \in [T]}$. Recall that the probability distribution of trajectories in $\bbX_1 \times \cdots \times \bbX_T$ factors according to $\calG$ (\Cref{eq:traj_dag}), forming a directed graphical model. We begin by defining the central quantity in our analysis which we call the ``information-structural state,'' hinting at the role it will play. The information structural state is defined for each point in time as a subset of the past, whether observed or latent, which forms a sufficient statistic for the future.

\begin{definition}[Information-structural state]\label{def:I_dagger}
    Let $\calG^\dagger$ be the DAG obtained from $\calG$ by removing all edges directed towards actions. That is, it consists of the edges $\calE^\dagger \coloneqq \calE \setminus \set{(x, a): x \in \calN, a \in \calA}$. For each $h \in [H]$, let $\calI_h^\dagger \subset [t(h)]$ be the minimal set of past variables (observed or unobserved) which $d$-separates the past observations $\pparen{X_{t(1)}, \ldots, X_{t(h)}}$ from the future observations $\pparen{X_{t(h+1)}, \ldots, X_{t(H)}}$ in the DAG $\calG^\dagger$. Define $\bbI_h^\dagger:= \prod_{s \in \calI_h^\dagger} \bbX_s$ as the joint space of those variables.
\end{definition}

The notation $\calI_h^\dagger$ is chosen to emphasize that this set depends on the information structure, $\calI = \set{\calI_t, t \in \calN}$, and that it simplifies or ``inverts'' the dynamics in some sense. Recall that the notation $t(h)$ denotes the index of the $h$-th observable, as defined in~\Cref{eq:observable_index}. $D$-separation is a property between nodes in DAGs which is central to identifying conditional independence relations~\citep{danAxiomsAlgorithmsInferences1989,vermaCausalNetworksSemantics1990,geigerIdentifyingIndependenceBayesian1990}.

We emphasize that $\bbI_h^\dagger$ may contain observable variables as well as \textit{unobservable} system variables, and $\bbI_h^\dagger \not\subset \bbH_h$ in general. As we will see, unobservable system variables can introduce crucial structure that simplifies the observable system dynamics. Note that $\calG^\dagger$, and hence $\calI_h^\dagger$, are independent of the information sets of action variables. That is, they only depend on the information structure of system variables.

The following theorem states that the rank of the observable system dynamics of POSTs and POSGs is bounded by the cardinality of $\bbI_h^\dagger$. In particular, $i_h^\dagger \in \bbI_h^\dagger$ can be thought of as an \textit{information-structural state} in the sense that it describes a set of system variables, either observable or latent, which provide a sufficient statistic of the past at time $h$ for predicting future observations---$I_h^\dagger$ ``separates'' the past from the future. Hence, the quantity $\aabs{\bbI_h^\dagger}$ admits an interpretation as the size of an effective state space at time $h$.

\begin{theorem}[Rank of observable system dynamics of POSTs and POSGs]\label{theorem:post_posg_rank}
    The rank of the observable system dynamics of a POST or POSG is bounded by
    \[r \leq \max_{h \in [H]} \Big\lvert \bbI_h^\dagger \Big\rvert .\]
\end{theorem}
\begin{proof}
    The proof is given in~\Cref{sec:appdx_seq_team_rank_proof}.
\end{proof}

This result shows that the complexity of the observable system dynamics, and hence the complexity of the sequential decision-making problem, is characterized by the information structure through $\bbI_h^\dagger$. This is significant because the rank of a generic sequential decision-making problem can be exponential in the horizon, in general. This result identifies the conditions under which the rank is manageable, and hence the problem can be represented in a tractable manner. Note that the ``information-structural state'' is a generalization of the standard notion of a latent state. For example, in the case of POMDPs, the information-structural state is indeed the latent Markovian state (as depicted in~\Cref{fig:pomdp_infostruct_state}). Additional examples are discussed next.

\subsection{Examples of Information Structures and their Rank}\label{ssec:info_struct_rank_examples}
The analysis in the previous section characterizes the rank of any sequential decision-making problem as a function of its information structure. In this section, we illustrate this on several sequential decision-making problems, characterizing the information-structural complexity of their dynamics. The procedure is as follows: \textbf{1)} formulate the sequential decision-making problem as a POST/POSG; \textbf{2)} represent the information structure as a labeled directed acyclic graph $\calG$; \textbf{3)} remove incoming edges into the action variables to produce $\calG^\dagger$; \textbf{4)} apply~\Cref{theorem:post_posg_rank} to find the information structural state at each point in time through a $d$-separation analysis.

\textbf{Illustration: translating to the POST/POSG framework.} We begin by illustrating how an arbitrary sequential decision-making problem can be formulated in the POST/POSG framework. Consider a POMDP with variables $(s_1, o_1, a_1, s_2, o_2, a_2, ...)$. This can be formulated as a POST/POSG by a simple relabelling of variables as follows.
\begin{equation*}
    \begin{matrix} s_1 \\ \downarrow \\ x_1 \end{matrix} \quad
    \begin{matrix} o_1 \\ \downarrow \\ x_2 \end{matrix} \quad
    \begin{matrix} a_1 \\ \downarrow \\ x_3 \end{matrix} \quad
    \begin{matrix} s_2 \\ \downarrow \\ x_4 \end{matrix} \quad
    \begin{matrix} o_2 \\ \downarrow \\ x_5 \end{matrix} \quad
    \begin{matrix} a_2 \\ \downarrow \\ x_6 \end{matrix} \quad
    \begin{matrix} \hphantom{a_2} \\ \cdots \\ \hphantom{x_6} \end{matrix} \quad
    \begin{matrix} s_t \\ \downarrow \\ x_{3t-2} \end{matrix} \quad
    \begin{matrix} o_t \\ \downarrow \\ x_{3t-1} \end{matrix} \quad
    \begin{matrix} a_t \\ \downarrow \\ x_{3t} \end{matrix} \quad
    \begin{matrix} \hphantom{a_2} \\ \cdots \\ \hphantom{x_6} \end{matrix} \quad
\end{equation*}
Here, the system variables $\calS$ are the $s$-type and $o$-type variables, with system index set $\calS = \sset{1,2,4,5,7,8, ...}$, and the action variables are the $a$-type variables with action index set $\calA = \sset{3, 6, 9, ...}$. The observable system variables are the $o$-type variables only, with index set $\calO = \sset{2, 5, 8, ...} \subset \calS$. This can be done for any sequential decision-making problem.

To ease notation, let us not explicitly write the indices in this section, but rather use the original notation for the variables in the problem formulation. For example, we'll write $\calS = \sset{s_t, o_t, t \in [T]}$. Similarly, we use the notation $\calI(x)$ to mean the information set corresponding to the variable $x$. Similarly, $\calI^\dagger(x)$ denotes the information-structural state at the time when $x$ occurs. For example, in a POMDP $\calI(s_t) = \sset{s_{t-1}, a_{t-1}}$, $\calI(o_t) = \sset{s_t}$, and $\calI(a_t) = \sset{o_{1:t}, a_{1:t}}$.

Below, we will consider several examples of sequential decision-making problems, and apply the information-structural analysis of~\Cref{theorem:post_posg_rank} to obtain a bound on the rank of the observable dynamics.

\textbf{Decentralized POMDPs and POMGs.} At each time $t$, the system variables of a decentralized POMDP (or POMG) consist of a latent state $s_t$, observations for each agent $o_t^1, \ldots, o_t^N$, and actions of each agent $a_t^1, \ldots, a_t^N$. The latent state transitions are Markovian and depend on the agents' joint action. The observations are sampled via a kernel conditional on the latent state. Each agent can use their own history of observations to choose an action. Thus, the information structure is given by,
\begin{equation*}
    \calI(s_t) = \set{s_{t-1}, a_{t-1}^1, \ldots, a_{t-1}^N},\, \calI(o_t^i) = \set{s_t},\, \calI(a_t^i) = \set{o_{1:t-1}^i, a_{1:t-1}^i}.
\end{equation*}
Here, the observable variables are $\calU = \set{o_{1:T}^i, a_{1:T}^i,\, i \in [N]}$\footnote{Here, since we don't explicitly write the index sets $\calS, \calO, \calA$, we use the notation $\calI(x)$ to mean the information set corresponding to the variable $x$. Similarly, $\calI^\dagger(x)$ denotes the information-structural state at the time when $x$ occurs. Since events may occur simultaneously, there is not a unique ordering of variables. For example, in a Dec-POMDP/POMG $(s_t, o_t^1, o_t^2)$ and $(s_t, o_t^2, o_t^1)$ are both valid orderings. When mapping such models onto the POST/POSG framework, we may choose any ordering arbitrarily. Similarly, we slightly abuse notation when defining the set of observables $\calU$, where what we mean is the ``time indices'' of the variables in $\{\cdot\}$.}. By~\Cref{theorem:post_posg_rank}, we have $\calI^\dagger(o_t^i) = \set{s_t}, \, \forall t, i $, as shown in~\Cref{fig:decpomdp_I_dagger}. Thus, the rank of a Dec-POMDP is bounded by $\abs{\bbS}$, where $\bbS$ is the state space. Note that in the case of models with a true latent state (e.g., POMDPs, Dec-POMDPs, and POMGs), the information-structural state coincides with the true latent state.

\begin{figure}
    \centering
    \begin{subfigure}[t]{0.49\textwidth}
        \includegraphics[width=\textwidth]{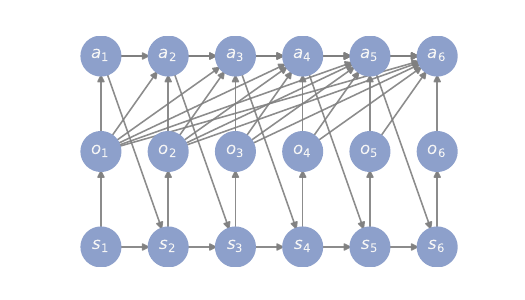}
    \end{subfigure}
    \begin{subfigure}[t]{0.49\textwidth}
        \includegraphics[width=\textwidth]{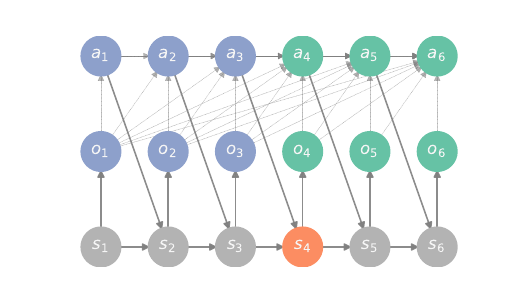}
    \end{subfigure}
    \caption{An illustrative example of the information-structural state for POMDPs. \textbf{Left.} The DAG representation of the information structure $\calG$. \textbf{Right.} The DAG $\calG^\dagger$ is depicted by drawing the edges corresponding to the information sets of the action variables with dotted lines. The information-structural state coincides with the Markovian state $s_t$, and is depicted in red. Future observables are drawn in green, and past observables are drawn in blue.}\label{fig:pomdp_infostruct_state}
\end{figure}

\begin{figure}[!ht]
    \centering
    \begin{subfigure}[t]{0.6\textwidth}
        \centering
        \includegraphics[width=\textwidth]{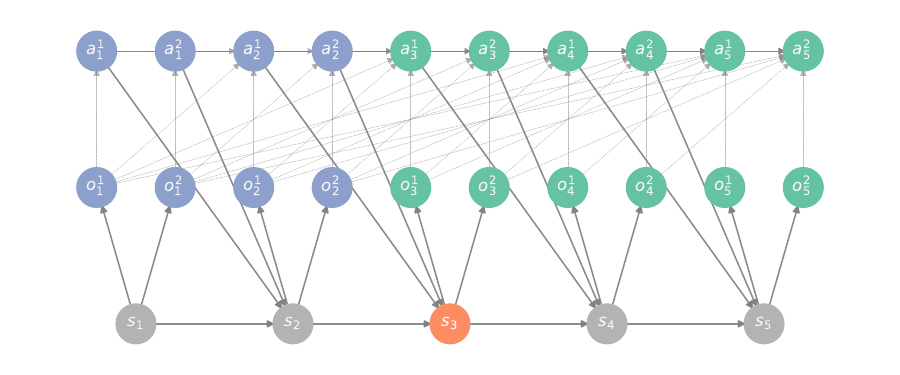}
        \caption{Decentralized POMDP/POMG information-structure.}\label{fig:decpomdp_I_dagger}
    \end{subfigure}
    \begin{subfigure}[t]{0.61\textwidth}
        \centering
        \includegraphics[width=\textwidth]{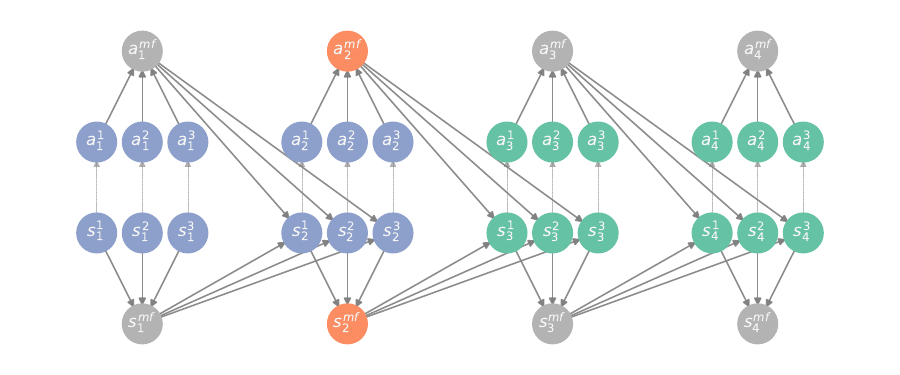}
        \caption{``Mean-field'' information structure.}\label{fig:meanfield_I_dagger}
    \end{subfigure}
    \begin{subfigure}[t]{0.36\textwidth}
        \centering
        \includegraphics[width=\textwidth]{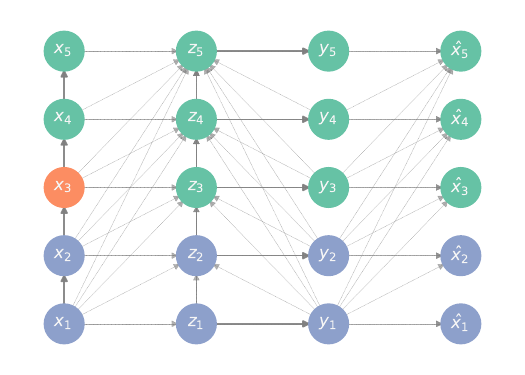}
        \caption{Point-to-point real-time communication with feedback information structure.}\label{fig:pt2pt_com_I_dagger}
    \end{subfigure}

    \begin{subfigure}[t]{0.4\textwidth}
        \includegraphics[width=\textwidth]{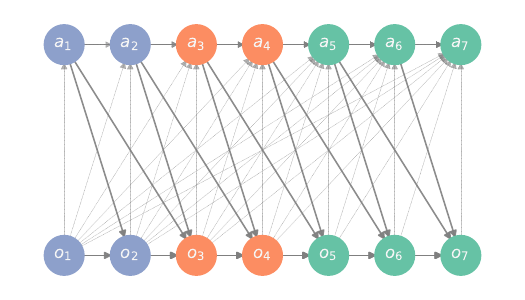}
        \caption{Limited-memory ($m=2$) information structures.}\label{fig:2step_mem_I_dagger}
    \end{subfigure}
    \begin{subfigure}[t]{0.4\textwidth}
        \includegraphics[width=\textwidth]{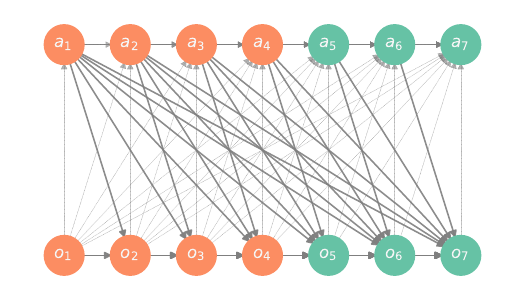}
        \caption{Fully connected information structure.}\label{fig:fully_connected_I_dagger}
    \end{subfigure}
    \caption{DAG representation of various information structures. Solid edges indicate the edges in $\calE^\dagger$ and light edges indicate the information sets of action variables. Grey nodes represent unobservable variables, blue nodes represent past observable variables, green nodes represent future observable variables, and red nodes represent the information structural state $\calI_h^\dagger$. To find $\calI_h^\dagger$, as per~\Cref{theorem:post_posg_rank}, we first remove the incoming edges into the action variables, then we find the minimal set among all past variables (both observable and unobservable) which $d$-separates the past observations from the future observations.}\label{fig:info_structs_dag}
\end{figure}

\textbf{Limited-memory information structures.} Consider a sequential decision making problem with variables $o_t, a_t, t \in [T]$ and an information structure with $m$-length memory. That is, observations can only depend directly on at most $m$ of the most recent observations and actions.
That is, the information structure is
\begin{equation*}
    \calI(o_t) = \set{o_{t-m:t-1}, a_{t-m:t-1}},\, \calI(a_t) = \set{o_{1:t}, a_{1:t-1}}.
\end{equation*}

The observables are all observations and actions, $\calU = \set{o_{1:T}, a_{1:T}}$. By~\Cref{theorem:post_posg_rank} we have that $\calI^\dagger(o_t) = \set{o_{t-m:t-1}, a_{t-m:t-1}}$, as shown in~\Cref{fig:2step_mem_I_dagger}. Hence, the rank of this sequential decision-making process is bounded by $\abs{\bbO}^{m} \abs{\bbA}^{m}$.

\textbf{Symmetric / ``Mean-field'' Information Structures.} Consider a sequential decision-making problem with $N$ agents. Each agent has their own local state, $s_t^i \in \bbS_{\mathrm{local}}$. Similarly, at each time point, each agent takes an action $a_i^t \in \bbA_{\mathrm{loc}}$. The global state $s_t = (s_t^1, \ldots, s_t^N) \in \bbS_{\mathrm{loc}}^N =: \bbS$ is composed by of all agents' local states. Similarly, the joint action space is $\bbA \coloneqq \bbA_{\mathrm{loc}}^N$. Consider a symmetric information structure where the evolution of each agent's local state depends only on a symmetric aggregation of all agents' states and actions, rather than on the local state/action of any particular agent. That is, the identity of who is in what state or takes which action does not matter---only the distribution of states and actions. This is often referred to as a ``mean-field'' setting (in the limit). Here, the transition depends only on the distribution of local states and actions, defined as $s_t^{\mathrm{mf}} = \mathrm{dist}(s_t) \coloneq \frac{1}{N} \delta_{s_t^i}$, $a_t^{\mathrm{mf}} = \mathrm{dist}(a_t) \coloneq \frac{1}{N} \delta_{a_t^i}$, for $s_t \in \bbS, a_t \in \bbA$. Different agents can have different transition kernels for their local state.
Hence, by introducing $\mathrm{dist}(s_t), \mathrm{dist}(a_t)$ as auxiliary unobserved variables at each time $t$, we obtain the following information structure,
\begin{equation*}
    \calI(s_t^i) = \set{\mathrm{dist}(s_{t-1}), \mathrm{dist}(a_{t-1})}, \ \calI(a_t^i) = \set{s_t^i}
\end{equation*}
and an application of~\Cref{theorem:post_posg_rank} bounds the rank by
\begin{equation*}
    \abs{\bbI^\dagger(s_t^i)} < \aabs{\bbS_{\mathrm{loc}}} \aabs{\bbA_{\mathrm{loc}}} \paren{\frac{N}{\aabs{\bbS_{\mathrm{loc}}} - 1} + 1}^{\aabs{\bbS_{\mathrm{loc}}} - 1} \paren{\frac{N}{\aabs{\bbA_{\mathrm{loc}}} - 1} + 1}^{\aabs{\bbA_{\mathrm{loc}}} - 1}.
\end{equation*}
This is compared to $\aabs{\bbS_{\mathrm{loc}}}^N \cdot \aabs{\bbA_{\mathrm{loc}}}^N$ (e.g., if we modeled this as an MDP with the state $s_t$), which is much larger when the number of agents is large. The information structure and $d$-separation decomposition are depicted in~\Cref{fig:meanfield_I_dagger}.

\textbf{Point-to-Point Real-Time Communication with Feedback.} Consider the following model of real-time communication with feedback. Let $x_t$ be the Markov source. At time $t$, the encoder receives the source $x_t \in \bbX$ and encodes sending a symbol $z_t \in \bbZ$. The symbol is sent through a memoryless noisy channel which outputs $y_t$ to the receiver. The decoder produces the estimate $\hat{x}_t$. The output of the noisy channel is also fed back to the encoder. The encoder and decoder have full memory of their observations and previous ``actions''. The observation variables are $\calO = \sset{x_{1:T},\, y_{1:T}}$ and the ``actions'' are $\calA = \sset{z_{1:T},\, \hat{x}_{1:T}}$. Hence, the information structure is given by the following,
\begin{equation*}
    \calI(x_t) = \set{x_{t-1}}, \ \calI(z_t) = \set{x_{1:t}, y_{1:t-1}, z_{1:t-1}}, \ \calI(y_t) = \set{z_{t}}, \ \calI(\hat{x}_t) = \set{y_{1:t}}.
\end{equation*}
By~\Cref{prop:seq_decmaking_exist_psr}, we have that,
\begin{equation*}
    \calI^\dagger(x_t) = \set{x_{t}}, \ \calI^\dagger(z_t) = \set{x_{t}}, \ \calI^\dagger(y_t) = \set{x_t, z_{t}}, \ \calI^\dagger(\hat{x}_t) = \set{x_{t}}.
\end{equation*}
Hence, the rank is bounded by $\aabs{\bbX} \aabs{\bbZ}$. This is depicted in~\Cref{fig:pt2pt_com_I_dagger}.

\textbf{Fully-Connected Information Structures.} Consider a sequential decision making problem with variables $o_t, a_t, t \in [T]$ and a fully-connected information structure. That is, each observation directly depends on the entire history of observations and actions. Thus, the information structure is
\begin{equation*}
    \calI(o_t) = \set{o_{1:t-1}, a_{1:t-1}},\ \calI(a_t) = \set{o_{1:t}, a_{1:t-1}}
\end{equation*}
The observables are all observations and actions, $\calU = \set{o_{1:T}, a_{1:T}}$. By~\Cref{theorem:post_posg_rank} we have that $\calI^\dagger(o_t) = \set{o_{1:t-1}, a_{1:t-1}}$, as shown in~\Cref{fig:fully_connected_I_dagger}. Hence, the rank of this sequential decision-making process can be exponential in the time horizon.

The examples above show that the tractability of a sequential decision-making problem in terms of the complexity of its dynamics depends directly on its information structure. This gives an interpretation of why certain models, like POMDPs, are more tractable than those with arbitrary information structures. Previous work primarily considers particular problem classes with fixed and highly regular information structures. In this work, we argue for the importance of explicitly modeling the information structure of a sequential decision-making problem.


\begin{remark}[Necessity of generalized PSRs]\label{remark:necessity_gen_psr}
    The formalization of generalized PSRs in~\Cref{ssec:gen_psr} was necessary to enable the study of information structure through POSTs/POSGs. An alternative (naive) solution to construct PSR representations for models with non-alternating observations and actions is to aggregate consecutive observations and actions to force them to obey the standard formulation of PSRs. This approach results in a loss of ``resolution'' in the information structure. That is, when you aggregate consecutive system variables, you also aggregate the DAG which represents the information structure, losing potentially important structure. In particular, in the worst case, such aggregation could result in an exponential increase in the rank of the dynamics. The examples given above elucidate this. Consider for example the ``mean-field'' information structure. If we aggregated local states and actions into a combined global state and joint action, the PSR rank would indeed be $\aabs{\bbS_{\mathrm{loc}}}^N \aabs{\bbA_{{\mathrm{loc}}}}^N$. By comparison, by considering each local state separately without aggregation, we are able to obtain a decomposition with a much smaller PSR rank.
\end{remark}

\section{Constructing a PSR parameterization for POSTs and POSGs}\label{sec:post_psrep}
A key challenge in reinforcement learning is constructing robust and efficient representations that enable modeling probabilities of system trajectories. That is, probabilities of the form $\prob{\mathtt{future} \ggiven \mathtt{history}}$. Observable operator models~\parencite{jaegerObservableOperatorModels2000} and predictive state representations~\parencite{littmanPredictiveRepresentationsState} are a class of representations for dynamical systems with several useful properties making them amenable to learning.

The standard formulation of PSRs is too restrictive to represent POSTs/POSGs due to its strict variable structure. However, generalized PSRs as introduced in~\Cref{def:gen_psr} are able to faithfully capture the POST/POSG models. In~\Cref{ssec:info_struct_rank}, we showed that the information structure of a POST/POSG can be used to characterize the rank of the observable system dynamics, and hence its PSR rank. In this section, we explicitly construct a \textit{generalized} predictive state representation for a class of POSTs and POSGs, ultimately enabling sample-efficient reinforcement learning.

\subsection{Core test sets for POSTs/POSGs}

A crucial ingredient for modeling partially-observable systems in the predictive state representation is the notion of a core test set, as defined in~\Cref{def:core_test_sets}.  Recall that a core test set is a set of futures such that the probabilities of those futures given the past encode all the information that the past contains about the future. For systems with a simple and regular information structure such as a POMDP, a core test set may be simple to obtain. For example, undercomplete POMDPs with a full rank 1-step emission matrix admit the 1-step observation space as a core test set. 

For POSTs/POSGs with arbitrary information structures, obtaining a core test set is much more challenging without knowing the system dynamics. In this section, we identify a condition in terms of the information structure under which $m$-step futures are a core test set for POSTs/POSGs.

For each $h \in [H]$, we denote the candidate core test set of $m$-step future observations by
\begin{equation}
    \bbQ_h^m := \prod_{s \in \calU_{h+1:\min(h+m, H)}} \bbX_s.
\end{equation}

Further, we define the matrix $\bm{G}_h \in \reals^{\aabs{\bbQ_h^m} \times \aabs{\bbI_h^\dagger}}$ as encoding the probability of observing each $m$-step future conditioned on the separating information set $\bbI_h^\dagger$,
\begin{equation}
    \begin{split}
        \bm{G}_h &\coloneq \bra{\probbar{q \given i_h^\dagger}}_{q \in \bbQ_h^m, \, i_h^\dagger \in \bbI_h^\dagger}  = \bra{\prob{\obs(q) \given i_h^\dagger;\, \mathrm{do}(\act(q))}}_{q \in \bbQ_h^m, \, i_h^\dagger \in \bbI_h^\dagger}, \\
    \end{split}
\end{equation}
where $q = \pparen{x_s, \, s \in \calU_{h+1:h+m}} \in \bbQ_h^m$, and $i_h^\dagger = \pparen{x_s, \, s \in \calI_h^\dagger} \in \bbI_h^\dagger$. The operational meaning of $\bm{G}_h$ is depicted in~\Cref{fig:post_genpsr_construction}.

\begin{figure}
    \centering
    \includegraphics[width=\textwidth]{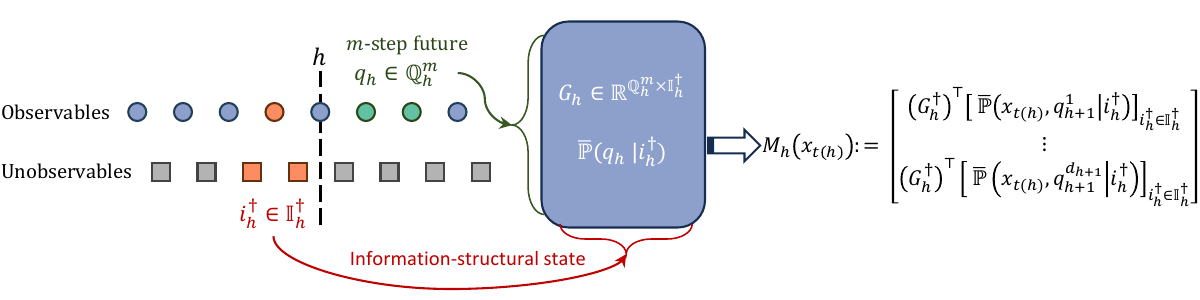}
    \caption{A depiction of the construction of a generalized predictive state representation for POST/POSG models.}\label{fig:post_genpsr_construction}
\end{figure}

We identify a condition in terms of information structure, named ``$m$-step $\calI^\dagger$-weakly revealing'', which we will show implies that the $m$-step futures are core test sets.

\begin{assumption}[$m$-step $\calI^\dagger$-weakly revealing]\label{ass:calIdagger_weakly_revealing}
    We say that a POST/POSG is $m$-step $\calI^\dagger$-weakly revealing if for all $h \in [H]$, $\mathrm{rank}(\bm{G}_h) = \lvert \bbI_h^\dagger \rvert$. Furthermore, we say that the POST/POSG is $\alpha$-robustly $m$-step $\calI^\dagger$-weakly revealing if for all $h \in [H-m+1]$, $\sigma_{\lvert \bbI_h^\dagger \rvert}(\bm{G}_h) \geq \alpha$.
\end{assumption}

The $\calI^\dagger$-weakly revealing condition is essentially an identifiability condition. If a POST/POSG is $\calI^\dagger$-weakly revealing, then, at any time point, for any two mixtures of the information-structural state with disjoint support, the distributions of the $m$-step futures are distinct. Formally, for any $\nu_1, \nu_2 \in \calP(\bbI_h^\dagger)$ with $\mathrm{supp}(\nu_1) \cap \mathrm{supp}(\nu_2) = \emptyset$, we have $\bm{G}_h \nu_1 \neq \bm{G}_h \nu_2$. That is, the future observations contain information that can distinguish between mixtures of the latent information-structural state. This description is equivalent to the condition that $\bm{G}_h$ is full-rank in~\Cref{ass:calIdagger_weakly_revealing}. The $\alpha$-robust version of the $\calI^\dagger$-weakly revealing condition requires that $\bm{G}_h$ is not only full rank, but that its $\aabs{\bbI_h^\dagger}$-th eigenvalue is bounded away from zero.

The condition holds whenever there exists a sequence of actions within the $m$-step futures such that executing these actions results in a sequence of observations which is informative about the information-structural state $i_h^\dagger \in \bbI_h^\dagger$. In general, this condition will be harder to satisfy when $\bbI_h^\dagger$ is large since it would require the $m$-step future observations to encode more information. In particular, $\bm{G}_h$ cannot be full rank when $\lvert\bbQ_h^m\rvert < \lvert\bbI_h^\dagger\rvert$. As a heuristic, when we don't have prior knowledge about the dynamics (e.g., in the learning setting), we can choose $m$ such that $\lvert\bbQ_h^m\rvert \geq \lvert \bbI_h^\dagger \rvert$. In general, it will be possible to find a smaller core test set when the $d$-separating set $\calI_h^\dagger$ is small. This happens when the system dynamics contain state-like variables which are low-dimensional.

The $\calI^\dagger$-weakly-revealing condition is a generalization of the ``weakly-revealing'' condition for POMDPs introduced in~\parencite{liuWhenPartiallyObservable2022}.~\cite{liuSampleEfficientReinforcementLearning2022} proposed an algorithm for learning weakly-revealing POMGs. Our analysis here recovers weakly-revealing POMGs as a special case and enables learning a much more general class of problems.

Recall that the vector of core test set probabilities for the history $\tau_h$ is given by the mappings $\psi_h, \bar{\psi}_h: \bbH_h \to \reals^{\abs{\bbQ_h^m}}$,
\begin{equation*}
    \psi_h(\tau_h) = \bra{\prob{q^o, \tau_h^o \given \mathrm{do}(\tau_h^a),\, \mathrm{do}(q^a)}}_{q \in \bbQ_h^m}, \ \bar{\psi}_h(\tau_h) = \bra{\prob{q^o \given \tau_h^o;\, \mathrm{do}(\tau_h^a),\, \mathrm{do}(q^a)}}_{q \in \bbQ_h^m}.
\end{equation*}
Define the mapping $m_h: \bbF_h \to \reals^{\abs{\bbQ_h^m}}$ as,
\begin{equation}
    m_h(\omega_h) \coloneqq (\bm{G}_h^\dagger)^\top \bra{\Probbar{\omega_h \,|\, i_h^\dagger}}_{i_h^\dagger \in \bbI_h^\dagger}
\end{equation}

The following lemma shows that the $m$-step futures $\bbQ_h^m$ are core test sets for any $m$-step $\calI^\dagger$-weakly revealing POST. In particular, given any future $\omega_h \in \bbF_h$ and history $\tau_h \in \bbH_h$, the conditional probability $\probbar{\omega_h \given \tau_h}$ can be written as a linear combination of the probabilities of the core tests given the history in $\bar{\psi}_h(\tau_h)$, with weights given by $m_h(\omega_h)$, depending only on $\omega_h$.
\begin{lemma}[Core test set for POSTs]\label{lemma:core_test_set}
    Suppose that the POST/POSG is $m$-step $\calI^\dagger$-weakly revealing. Then, $\bbQ_h^m$ is a core test set for all $h \in [H]$. Furthermore, we have
    \begin{equation}
        \probbar{\tau_h, \omega_h} = \iprod{m_h(\omega_h)}{\psi_h(\tau_h)}, \text{ and } \ \, \probbar{\omega_h \given \tau_h} = \iprod{m_h(\omega_h)}{\bar{\psi}_h(\tau_h)}.
    \end{equation}
\end{lemma}
\begin{proof}
    The proof is given in~\Cref{sec:appdx_post_psrep_proofs}.
\end{proof}

Finally, we remark that the information structure enables us to construct core tests in a time-dependent way. Note that the size of the information-structural state space  $\bbI_h^\dagger$ may vary with $h$. Accordingly, we can consider core test sets consisting of $m(h)$-step futures where $m(h)$ varies with $h$. Operationally, we can choose $m(h)$ based on the size of $\bbI_h^\dagger$, despite not knowing the dynamics. The information structure may also enable a more tailored construction of core tests. For simplicity of presentation, we will only consider fixed $m$ in this paper.

\subsection{Generalized PSR parameterization of POST/POSG}\label{ssec:psr_param_post}

Consider a POST/POSG which is $m$-step $\calI^\dagger$-weakly revealing.~\Cref{lemma:core_test_set} shows that the $m$-step futures $\bbQ_h^m$ are core test sets. In this section we will explicitly construct a generalized PSR parameterization (according to~\Cref{def:gen_psr}) for this class sequential decision-making problems. Moreover, we will show that this generalized PSR representation is well-conditioned when the weakly revealing condition is robust. In the next sections, we will show that well-conditioned generalized PSRs can be learned sample-efficiently.

Let $d_h \coloneqq \abs{\bbQ_h^m}$. The first observation is that the vector mappings $m_h: \bbF_h \to \reals^{d_h}$ and $\psi_h: \bbH_h \to \reals^{d_h}$ can be used to derive a recursive form of the dynamics of the POST/POSG. A direct corollary of~\Cref{lemma:core_test_set} is the following.

\begin{lemma}
    For any $h \in [H],\, \tau_{h-1} \in \bbH_{h-1},\, x_{t(h)} \in \bbX_{t(h)},\, \omega_{h} \in \bbF_{h}$, we have
    \begin{equation}
        \probbar{\tau_{h-1}, x_{t(h)}, \omega_{h}} = \iprod{m_{h-1}(x_{t(h)}, \omega_{h})}{\psi_h(\tau_{h-1})}.
    \end{equation}
\end{lemma}

Hence, given a history $\tau_{h-1} = (x_{t(1)}, \ldots, x_{t(h-1)})$, having observed another variable $x_{t(h)}$, we can update our predictions of the future and obtain the probability of trajectories of the form $(\tau_{h-1}, x_{t(h)}, \omega_{h})$ for any future $\omega_{h} \in \bbF_{h}$. Note that $x_{t(h)}$ may be either an observation or an action. Hence, we can update our prediction of the future after deciding an action, and before receiving the next observation. This is in contrast to the standard PSR formulation where predictions of the future can only be updated with a \textit{pair} of observation and action. Our formulation provides additional flexibility, which is crucial for the general information structures modeled by POSTs and POSGs.

This means that, after observing $x_{t(h)}$, we can use the $m_{h}: \bbF_{h} \to \reals^{d_{h}}$ mapping constructed in~\Cref{lemma:core_test_set} to update the probability of any candidate future $\omega_{h}$. We are particularly interested in updating the probabilities of the futures corresponding to the \textit{core test set} at the next time point, since this provides a sufficient statistic of the past. Thus, we define the matrix mapping $\bm{M}_h: \bbX_{t(h)} \to \reals^{d_{h} \times d_{h-1}}$ by,
\begin{equation}\label{eq:def_M_h}
    \bra{M_h(x_{t(h)})}_{q, \cdot} = m_{h-1}(x_{t(h)}, q)^\top, \, q \in \bbQ_{h}.
\end{equation}

That is, $M_h(x_{t(h)})$ is the matrix whose rows are indexed by the core tests at the $h$-th observable step, where the $q \in \bbQ_{h}$ row is the weights given by the $m_{h-1}$ mapping for the future consisting of $x_{t(h)}$ followed by $q$. This mapping enables us to update the probabilities of the core test sets.

\begin{lemma}
    For any $h \in [H-1],\, \tau_h \in \bbH_h,\, x_{t(h+1)} \in \bbX_{t(h+1)}$, we have
    \begin{equation}
        \psi_{h}(\tau_{h-1}, x_{t(h)}) = M_h(x_{t(h)}) \psi_{h-1}(\tau_{h-1}).
    \end{equation}
Hence, for a history $\tau_h = \paren{x_{t(1)}, \ldots, x_{t(h)}} \in \bbH_h$, we have
    \begin{equation}\label{eq:post_psr_psi}
        \psi_h(\tau_h) = M_h(x_{t(h)}) M_{h-1}(x_{t(h-1)}) \cdots M_1(x_{t(1)}) \psi_0,
    \end{equation}
where $\psi_0 = \psi_0(\emptyset)$.
\end{lemma}

Finally, observe that $\bbQ_{H-1}^m = \bbX_{t(H)}$. Hence,
\begin{equation*}
    \psi_{H-1}(\tau_{H-1}) = \paren{\probbar{\tau_{h-1}, x_{t(H)}}}_{x_{t(H)} \in \bbX_{t(H)}} \in \reals^{\abs{\bbX_{t(H)}}}.
\end{equation*}
Thus, letting $\phi_H: \bbX_{t(H)} \to \reals^{\abs{\bbX_t(H)}}$ be $\phi_H(x_{t(H)}) = \bm{e}_{x_{t(H)}}$ (the canonical basis vector), yields
\begin{equation}\label{eq:post_psr_joint_prob}
    \probbar{x_{t(h)} \,\colon\, h \in [H]}  = \phi_{H}(x_{t(H)})^\top M_{H-1}(x_{t(H-1)}) \cdots M_1(x_{t(1)}) \psi_0.
\end{equation}

Hence,~\Cref{eq:post_psr_joint_prob} together with~\Cref{eq:post_psr_psi} implies that $\paren{\bm{M}, \phi_H, \psi_0}$ is a valid generalized PSR representation for the POST/POSG (as per~\Cref{def:gen_psr}). Moreover, when the $\calI^\dagger$-weakly-revealing is robust (as per~\Cref{ass:calIdagger_weakly_revealing}), then the generalized PSR is well-conditioned. We summarize this in the following result.

\begin{theorem}[Generalized PSR representation for POST/POSG]
    \label{theorem:constructing_gen_psr}
    Consider an $m$-step $\calI^\dagger$-weakly revealing POST/POSG. Let $\sset{M_h}_{h \in [H-1]}$ be defined as in \Cref{eq:def_M_h}
    and let
    \[\psi_0 = \bra{\probbar{q}}_{q \in \bbQ_0^m},\ \phi_H(x_{t(H)}) = \bm{e}_{x_{t(H)}}.\]
    Then, $(\sset{\bbQ_h^m}_h, \phi_H, \sset{M_h}_h, \psi_0)$ forms a generalized predictive state representation. In particular,
    \begin{align*}
        \probbar{x_{t(1)}, \ldots, x_{t(H)}}  &= \phi_{H}(x_{t(H)})^\top M_{H-1}(x_{t(H-1)}) \cdots M_1(x_{t(1)}) \psi_0, \\
        \psi_h(x_{t(1)}, \ldots, x_{t(h)}) &= M_h(x_{t(h)}) \cdots M_1(x_{t(1)}) \psi_0, \ \forall h.
    \end{align*}
    Moreover, if the $\calI^\dagger$-weakly-revealing property is $\alpha$-robust, then this generalized PSR is $\gamma$-well-conditioned with $\gamma = \alpha / \max_h \aabs{\bbI_h^\dagger}^{1/2}$.
\end{theorem}
\begin{proof}
    The proof is given in~\Cref{sec:appdx_post_psrep_proofs}.
\end{proof}

We note that well-conditioning through $\alpha$-robustness of the $\calI^\dagger$-weakly revealing is necessary in the learning setting. Without the well-conditioning of the generalized PSR, small estimation errors in the parameters of the PSR can result in unbounded errors in the estimated probabilities of trajectories.

\section{Characterizing the Sample Complexity of General Reinforcement Problems via of Information Structure}\label{sec:rl_post}

In this section, we establish an upper bound on the achievable sample complexity of general reinforcement learning problems in terms of their information structure. In plain language, this is a result that says ``any sequential decision-making problem with an information structure $\calI$ can be learned with a sample complexity at most $f(\calI)$''. We prove this result by exhibiting an algorithm that achieves this upper bound for any sequential decision-making problem with a particular information structure. This identifies a class of sequential decision-making problems that are statistically tractable via conditions on the information structure, expanding the set of known-tractable problems while recovering existing tractability results as a special case (e.g., observable POMDPs). This section focuses on the team setting, whereas the next section extends the result to the game setting.

Recall that in~\Cref{def:gen_psr}, we proposed a generalization of predictive state representations which extends the standard formulation of PSRs, and in~\Cref{sec:post_psrep} we explicitly constructed a generalized PSR representation for POST/POSG models as a function of their information structure. Our approach in this section will be to introduce an algorithm for generalized PSRs and prove a corresponding sample complexity result, which will in turn imply a bound on the sample complexity of learning POST/POSG models via the information-structural characterization of the rank of observable dynamics established in~\Cref{ssec:info_struct_rank}. We note that the construct of generalized PSRs and the corresponding algorithms may be of independent interest, but can be thought of as a proof device for the purposes of this paper.

We will adapt the model-based UCB-type algorithm of~\citet{huangProvablyEfficientUCBtype2023}, extending it to generalized PSRs, including those representing POSTs. The algorithm involves the estimation of an upper confidence bound which captures the uncertainty in the estimated model and drives exploration so as to minimize this uncertainty. The UCB-based approach has the advantage of providing a last-iterate guarantee and requiring a weaker notion of a planning oracle (a standard planning oracle instead of an optimistic planning oracle as required by similar algorithms). The technical contribution of this section is to extend the algorithm and its theoretical guarantees to generalized PSRs. The tools developed in doing so can be used to directly extend any other PSR-based algorithm to generalized PSRs.


When learning generalized PSRs, we suppose that the core test sets $\set{\bbQ_h}_{0 \leq h \leq H-1}$ are known by the algorithm. For example, if the sequential decision-making problem is a POST,~\Cref{sec:post_psrep} provides conditions under which $m$-step futures form core test sets. Let $\Theta$ be the set of $\gamma$-well-conditioned generalized PSR representations with $\sset{\bbQ_h}_{0 \leq h \leq H-1}$ as core test sets. Denote by $\bar{\Theta}_\epsilon$ an optimistic $\epsilon$-cover of $\Theta$ (defined formally in~\Cref{sec:appdx_genpsr_exists}).

Recall that $d_h \coloneqq \abs{\bbQ_h}$ and $d = \max_h d_h$. Moreover, $\bbQ_h^A \coloneqq \act(\bbQ_h)$ are the action components of the core test sets and $Q_A \coloneqq \max_h \abs{\bbQ_h^A}$ is the maximal size of those action components. We define the exploration action sequences at time $h$ to be $\bbQ_{h-1}^{\mathrm{exp}} = \act(\bbX_{h} \times \bbQ_h \cup \bbQ_{h-1})$. Moreover, we define $\uhexp{h-1}$ as the policy, defined from time $h-1$ onwards, in which each selection of action sequences in $\bbQ_{h-1}^{\mathrm{exp}}$ are chosen uniformly at random. For a model $\theta$ and reward function $R$, we define the value of a policy under this model and reward as $V_\theta^R(\pi) \coloneq \sum_{\tau_H} R(\tau_H) \probb{\theta}{\pi}(\tau_H)$.

The algorithmic description is given in~\Cref{alg:post_ucb}. At each iteration $k$, the learner collects a trajectory $\tau_H^{k,h}$ for each time index $h \in [H]$ by using a particular policy that drives exploration so as to better estimate the parameters associated with the $h$-th time step. To collect the trajectory $\tau_{H}^{k,h}$, the learner executes the policy at the previous iteration, $\pi^{k-1}$, until time $h-1$ collecting the trajectory $\tau_{h-1}^{k,h}$ then executes $\uhexp{h-1}$ which samples action sequences from $\bbQ_{h-1}^{\mathrm{exp}}$ uniformly. The particular choice of the exploratory action sequences $\bbQ_{h-1}^{\mathrm{exp}}$ comes out of the proof (see proof of~\Cref{lemma:ucb_tv} in the appendix). Intuitively, $\act(\bbQ_{h-1})$ allows us to estimate the prediction features $\bar{\psi}^*(\tau_{h-1}^{k,h}) = [\bar{\bbP}(q \,|\, \tau_{h-1}^{k,h})]_{q \in \bbQ_{h-1}}$, and $\act(\bbX_h \times \bbQ_{h})$ allows us to estimate $M_h^*(x_h) \bar{\psi}^*(\tau_{h-1}^{k,h})$.


The collected trajectories are added to the dataset, together with the policies used to collect them. The next step is model estimation via (constrained) maximum likelihood estimation. The algorithm estimates a model $\hat{\theta}^k$ by selecting any model in a constrained set $\calB^k$ defined as
\begin{equation}
    \begin{split}
        &\Theta_{\min}^k = \set{\theta \in \Theta: \forall h,\, (\tau_h, \pi) \in \calD_h^k,\ \probb{\theta}{\pi}(\tau_h) \geq p_{\min}},\\
        &\calB^k = \set{\theta \in \Theta_{\min}^k : \sum_{(\tau_H, \pi) \in \calD^k} \log \probb{\theta}{\pi}(\tau_H) \geq \max_{\theta' \in \Theta_{\min}^k}  \sum_{(\tau_H, \pi) \in \calD^k} \log \probb{\theta'}{\pi}(\tau_H) - \beta}.
    \end{split}
\end{equation}

The introduction of $\Theta_{\min}^k$ ensures that $\probb{\theta^*}{\pi^{k-1}}(\tau_{h-1}^{k,h})$ is not too small so that the estimates of the prediction features $\bar{\psi}^*(\tau_{h-1}^{k,h}) = [\bar{\bbP}(q \,|\, \tau_{h-1}^{k,h})]_{q \in \bbQ_{h-1}}$ are accurate. This design differs from other MLE-based estimators~\parencite[e.g.,][]{liuWhenPartiallyObservable2022,liuOptimisticMLEGeneric2022,chenPartiallyObservableRL2022} due to the estimation of parameters capturing \textit{conditional} probabilities.

Next, the algorithm chooses a policy that drives the algorithm to trajectories $\tau_h$ whose prediction features have so far been unexplored. To do this,~\Cref{alg:post_ucb} constructs an upper confidence bound on the total variation distance between the estimated model and the true model. This is done via a bonus function $\hat{b}^k(\tau_H)$,
\begin{equation}
    \begin{split}
        \hat{b}^k(\tau_H) &= \min\set{\alpha \sqrt{\sum_{h=0}^{H-1} \norm{\hatbar{\psi}(\tau_h)}_{(\hat{U}_h^k)^{-1}}^2}, 1}, \quad \text{where},\\
        \hat{U}_h^k &= \lambda I + \sum_{\tau_h \in \calD_{h^k}} \hatbar{\psi}^k(\tau_h) \hatbar{\psi}^k(\tau_h)^\top,
    \end{split}
\end{equation}
where $\lambda$ and $\alpha$ are pre-specified parameters to the algorithm. Thus, the bonus function captures the degree of uncertainty in the estimated prediction features $\hatbar{\psi}(\tau_h)$. In particular, the bonus $\hat{b}(\tau_H)$ 
will be large for trajectories whose prediction feature $\hatbar{\psi}(\tau_h)$ lie far away from the empirical distribution of prediction features sampled in the dataset $\calD_h^k$. This is captured by computing the norm with respect to the covariance $\hat{U}_h^k$.

The algorithm then chooses an exploration policy for the next iteration which maximizes this upper confidence bound, hence collecting trajectories that have high uncertainty in their prediction features. When the estimated model is sufficiently accurate on all trajectories, the algorithm terminates and returns the optimal policy with respect to the reward function $R$ under the estimated model.

\begin{algorithm}[ht]
    \caption{Learning Generalized PSRs (e.g., POSTs) via MLE and Exploration with UCB}\label{alg:post_ucb}

    \For{$k \gets 1, \ldots, K$}{
        \For{$h \gets 1, \ldots, H$}{
            Collect $\tau_H^{k,h} = (\omega_{h-1}^{k,h}, \tau_{h-1}^{k,h})$ using $\nu(\pi^{k-1}, \uhexp{h-1})$.

             $\calD_{h-1}^k \gets \calD_{h-1}^{k-1} \cup \set{\paren{\tau_H^{k,h}, \nu\paren{\pi^{k-1}, \uhexp{h-1}}}}$.
        }

       $\calD^k = \set{\calD_h^k}_{h=0}^{H-1}$

       Compute MLE $\hat{\theta} \in \calB^k$, where
       \begin{align*}
            &\Theta_{\min}^k = \set{\theta: \forall h, (\tau_h, \pi) \in \calD_h^k, \probb{\theta}{\pi}(\tau_h) \geq p_{\min}},\\
            &\calB^k = \set{\theta \in \Theta_{\min}^k : \sum_{(\tau_H, \pi) \in \calD^k} \log \probb{\theta}{\pi}(\tau_H) \geq \max_{\theta' \in \Theta_{\min}^k}  \sum_{(\tau_H, \pi) \in \calD^k} \log \probb{\theta'}{\pi}(\tau_H) - \beta}.
       \end{align*}

       Define the bonus function, $\hat{b}^k(\tau_H) = \min\set{\alpha \sqrt{\sum_{h=0}^{H-1} \norm{\hatbar{\psi}(\tau_h)}_{(\hat{U}_h^k)^{-1}}^2}, 1}$, where $\hat{U}_h^k = \lambda I + \sum_{\tau_h \in \calD_{h^k}} \hatbar{\psi}^k(\tau_h) \hatbar{\psi}^k(\tau_h)^\top$.

       Solve the planning problem to maximize the bonus function $\pi^k = \argmax_{\pi} V_{\hat{\theta}^k}^{\hat{b}^k}(\pi)$.

       \If{$V_{\hat{\theta}^k}^{\hat{b}^k}(\pi^k) \leq \epsilon /2$}{
            $\theta^\epsilon = \hat{\theta}^k$. \textbf{break.}
       }
    }
    \Return{$\pi = \argmax_\pi V_{\theta^\epsilon}^{R}(\pi)$}
\end{algorithm}

We extend~\citeauthor{huangProvablyEfficientUCBtype2023}'s theoretical guarantees to show that~\Cref{alg:post_ucb} enjoys polynomial sample complexity for \textit{generalized} PSRs (\Cref{def:gen_psr}).

\begin{theorem}\label{theorem:post_ucb}
    Suppose~\Cref{ass:psr_gamma_wellcond} holds. Suppose the parameters $p_{\min}, \lambda, \alpha, \beta$ are chosen appropriately. In particular, let 
    \begin{align*}
        p_{\min} \leq \frac{\delta}{K H \prod_{h=1}^H \abs{\bbX_{h}}},\ \lambda = \frac{{\gamma \max_{s \in \calA} \aabs{\bbX_{s}}^2 Q_A \beta \max\sset{\sqrt{r}, Q_A \sqrt{H} / \gamma}}}{\sqrt{dH}},\\
        \alpha = O\paren{\frac{Q_A \sqrt{d H \lambda}}{\gamma^2} + \frac{\max_{s \in \calA} \abs{\bbX_{s}} Q_A \sqrt{\beta}}{\gamma}}, \ \beta = O\paren{\log \abs{\bar{\Theta}_\varepsilon}}, \ \varepsilon \leq \frac{p_{\min}}{K H}.
    \end{align*}
    Then, with probability at least $1 - \delta$,~\Cref{alg:post_ucb} returns a model $\theta^\epsilon$ and a policy $\pi$ that satisfy
    \begin{equation*}
        V_{\theta^\epsilon}^{R}(\pi^*) - V_{\theta^\epsilon}^{R}(\pi) \leq \varepsilon, \ \text{and} \ \forall \tilde{\pi}, \ \tv{\probb{\theta^\epsilon}{\tilde{\pi}}(\tau_H), \probb{\theta^*}{\pi}(\tau_H)} \leq \varepsilon.
    \end{equation*}
    In addition, the algorithm terminates with a sample complexity of,
    \begin{equation*}
        \tilde{O}\paren{\paren{r + \frac{Q_A^2 H}{\gamma^2}} \cdot \frac{r d H^3 \cdot \max_{s \in \calA} \abs{\bbX_{s}}^2 \cdot Q_A^4 \beta}{\gamma^4 \epsilon^2}}.
    \end{equation*}
\end{theorem}
\begin{proof}
    The proof is given in~\Cref{sec:post_ucb_proof}.
\end{proof}

This result shows that the sample complexity of learning a generalized PSR depends on the problem size through a few key quantities. In particular, the sample complexity scales polynomially in the underlying rank $r$, the dimension of the PSR parameterization $d$, the size of the action component of the core tests $Q_A$, the time horizon $H$, the conditioning number $\gamma^{-1}$, the size of the action spaces $\max_{s \in \calA} \aabs{\bbX_s}$, the log covering number $\log \aabs{\bar{\Theta}_\epsilon}$, and the desired suboptimality error $\epsilon$. Note that $\tilde{O}$ omits logarithmic dependence.

To apply this algorithm to a POST, we can use the generalized PSR parameterization constructed in~\Cref{sec:post_psrep}. By~\Cref{theorem:post_posg_rank} the PSR rank is bounded by $r \leq \max_{h} \lvert \bbI_h^\dagger \rvert$. If this POST is $\alpha$-robustly $\calI^\dagger$-weakly revealing, then by~\Cref{theorem:constructing_gen_psr} it admits a $\gamma$-well-conditioned generalized PSR parameterization with $\gamma = \alpha / \max_h \lvert\bbI_h^\dagger\rvert^{1/2}$ and the $m$-step futures as core test sets. Moreover, we have $d = \max_h d_h = \max_h \abs{\bbQ_h^m}$. The following corollary states that~\Cref{alg:post_ucb} can learn a partially-observable sequential team with a sample complexity that is polynomial in the size of the information-structural state space $\max_h \aabs{\bbI_h^\dagger}$.

\begin{corollary}\label{cor:post_ucb}
    Suppose a partially-observable sequential team is $m$-step $\alpha$-robustly $\calI^\dagger$-weakly revealing (\Cref{ass:calIdagger_weakly_revealing}). Applying~\Cref{alg:post_ucb} to this PSR representation, with parameters $p_{\min}, \lambda, \alpha, \beta$ chosen as in~\Cref{theorem:post_ucb}, returns a $\varepsilon$-optimal policy with a sample complexity of,
    \begin{equation*}
        \tilde{O}\paren{\paren{1 + \frac{Q_A^2 H}{\alpha^2}} \frac{\max_{h} \aabs{\bbI_h^\dagger}^7  \cdot \max_h{\abs{\bbQ_h^m}} \cdot H^5\cdot  \max_{s \in \calA}  \abs{\bbX_{s}}^2 \cdot \max_{s \in \calU} \abs{\bbX_s} \cdot Q_A^4}{\alpha^4 \epsilon^2}}.
    \end{equation*}
\end{corollary}

We can interpret this result as saying that the information structure of a sequential decision-making problem, through the quantity $\max_h \lvert \bbI_h^\dagger \rvert$, is fundamentally a measure of the complexity of the dynamics that need to be modeled. As a result, learning is tractable when $\max_h \lvert \bbI_h^\dagger \rvert$ is of modest size, and intractable otherwise. Recall that $\max_h \lvert \bbI_h^\dagger \rvert$ is small when there exists ``state-like'' variables, whether they are observable or unobservable. In this sense, $\max_h \lvert \bbI_h^\dagger \rvert$ is a fundamental quantity which generalizes the notion of a ``state''. For example, in the case of an $m$-step $\alpha$-weakly revealing POMDP, our algorithm has a sample complexity of $\mathrm{poly}(S, (OA)^m, H, \alpha^{-1}) \cdot \epsilon^{-2}$, where $S$ is the size of the state space, $O$ is the size of the observation space, and $A$ is the size of the action space. This is similar to the sample complexity of~\citep{liuWhenPartiallyObservable2022,liuSampleEfficientReinforcementLearning2022}, which designed an algorithm tailored specifically for weakly-revealing POMDPs. Our algorithms, together with the POST/POSG models, enable sample-efficient reinforcement learning for a much broader class of models all within a unified framework.

In this section, we extended the algorithm in ~\citet{huangProvablyEfficientUCBtype2023}   to generalized PSRs, enabling sample-efficient learning of POSTs. We emphasize that other PSR-based algorithms can be extended in a similar manner. In the next section, we tackle the problem of learning in the game setting where different agents have different objectives.

\section{Extension to the Game Setting}\label{sec:rl_posg}

In this section, we extend the results in the previous section to the game setting where each agent has their own objective. This will show that the sample complexity of learning sequential games can also be upper bounded in terms of information structure.

The algorithm we propose is a \textit{self-play} algorithm for learning an \textit{equilibrium} of the dynamic game problem. That is, the algorithm specifies the policies of all agents during the learning phase, collecting the trajectory of observables at each episode to improve its estimate of the system dynamics. This can be thought of as a centralized agent playing against itself. We will propose an algorithm that can find a Nash equilibrium or coarse correlated equilibrium in a sample-efficient manner. We begin with some preliminaries.

\textbf{Game setting.} Recall that a sequential decision-making problem falls within the game setting if each agent has their own objective. Following~\Cref{ssec:seq_dec_making}, we consider a sequential decision-making problem $(X_1, \ldots, X_H)$ where $\calO$ denotes the index set of (observable) system variables and $\calA$ denotes the set of action variables. We suppose the game involves $N$ agents, and denote the action index set of each agent by $\calA^i \subset \calA$, where $\sset{\calA^i}_{i \in [N]}$ partitions $\calA$. Each agent has their own reward function $R^i(X_1, \ldots, X_H)$. Note that POSGs as defined in~\Cref{def:partially_obs_seq_game} are structured models which fall within this framework.

\textbf{Equilibria and policy classes.} Recall that in the game setting, the type of randomization in each agent's policy affects the set of equilibria in the game. In~\Cref{ssec:posg}, we formalized this randomization by introducing a random seed $\omega \in \Omega$ and allowing each agent's policy to be a function of their information set and this seed. If the seed has a product structure $\omega = (\omega_1, \ldots, \omega_N)$ with each agent observing their own seed, this results in independently-randomized policies, denoted by $\Gamma_{\mathrm{ind}}^i$. If all agents use the same seed, this results in correlated randomized policies, which we denote by $\Gamma_{\mathrm{cor}}^i$. An equilibrium among independently randomized policies is called a Nash equilibrium and an equilibrium among correlated randomized policies is called a coarse correlated equilibrium.

\textbf{Estimating probabilities in the planner.} The probability of any trajectory under a joint policy $\bm{\pi}$ is given by $\Prob^{\bm{\pi}}(\tau_H) = \sum_{\omega} \probbar{\tau_H} \bm{\pi}(\tau_H | \omega) \prob{\omega}$, where $\probbar{\tau_H} = \prob{\tau_H^o \given \tau_H^a}$ as before, and $\bm{\pi}(\tau_H \,|\, \omega) = \prod_{h \in \calA} \bm{1}\{x_h = \pi_h(\tau_{h-1}, \omega)\}$. Recall that the probabilities $\probbar{\tau_H}$ are estimated by the generalized PSR model $\hat{\theta}$. We assume that the planner has knowledge of the randomization, $\prob{\omega}$. Hence, the planner in the self-play algorithm is able to compute the probability of any trajectory for each choice of policy.

\textbf{Algorithm.} The algorithmic description is presented in~\Cref{alg:posg_ucb}. In the first stage of the algorithm, the centralized learning agent has a unified goal: to explore the environment. This is done by executing policies that maximize the bonus function $\hat{b}^k(\tau_H)$ by visiting trajectories with imprecise estimates of their probability, as measured by the upper confidence bound on the total variation distance. This part is identical to~\Cref{alg:post_ucb}. Once the algorithm is sufficiently confident about the estimated probabilities of all trajectories, it computes the equilibrium using the estimated model directly. That is, $\mathtt{ComputeEquilibrium}$ computes either NE or CCE. The only difference in the exploration stage of the algorithm compared to~\Cref{alg:post_ucb} is that the termination condition involves $\varepsilon / 4$ rather than $\varepsilon / 2$ in order to guarantee an $\varepsilon$-approximate equilibrium under the added complications of the game setting.

\begin{algorithm}[t]
    \caption{Self-play UCB Algorithm for Sequential Games}\label{alg:posg_ucb}

    \For{$k \gets 1, \ldots, K$}{
        \For{$h \gets 1, \ldots, H$}{
            Collect $\tau_H^{k,h} = (\omega_{h-1}^{k,h}, \tau_{h-1}^{k,h})$ using $\nu(\pi^{k-1}, \uhexp{h-1})$.

             $\calD_{h-1}^k \gets \calD_{h-1}^{k-1} \cup \set{\paren{\tau_H^{k,h}, \nu\paren{\pi^{k-1}, \uhexp{h-1}}}}$.
        }

       $\calD^k = \set{\calD_h^k}_{h=0}^{H-1}$

       Compute MLE $\hat{\theta} \in \calB^k$, where
       \begin{align*}
            &\Theta_{\min}^k = \set{\theta: \forall h, (\tau_h, \pi) \in \calD_h^k, \probb{\theta}{\pi}(\tau_h) \geq p_{\min}},\\
            &\calB^k = \set{\theta \in \Theta_{\min}^k : \sum_{(\tau_H, \pi) \in \calD^k} \log \probb{\theta}{\pi}(\tau_H) \geq \max_{\theta' \in \Theta_{\min}^k}  \sum_{(\tau_H, \pi) \in \calD^k} \log \probb{\theta'}{\pi}(\tau_H) - \beta}.
       \end{align*}

       Define the bonus function, $\hat{b}^k(\tau_H) = \min\set{\alpha \sqrt{\sum_{h=0}^{H-1} \norm{\hatbar{\psi}(\tau_h)}_{(\hat{U}_h^k)^{-1}}^2}, 1}$, where $\hat{U}_h^k = \lambda I + \sum_{\tau_h \in \calD_{h^k}} \hatbar{\psi}^k(\tau_h) \hatbar{\psi}^k(\tau_h)^\top$.

       Solve the planning problem $\pi^k = \argmax_{\pi} V_{\hat{\theta}^k}^{\hat{b}^k}(\pi)$.

       \If{$V_{\hat{\theta}^k}^{\hat{b}^k}(\pi^k) \leq \epsilon / 4$}{
            $\theta^\epsilon = \hat{\theta}^k$. \textbf{break.}
       }
    }
    \Return{$\pi = \mathtt{ComputeEquilibrium}(\theta^\epsilon, \set{R^1, \ldots, R^N})$}
\end{algorithm}

\begin{theorem}\label{theorem:posg_ucb}
    Suppose~\Cref{ass:psr_gamma_wellcond} holds. Suppose the parameters $p_{\min}, \lambda, \alpha, \beta$ are chosen as in~\Cref{theorem:post_ucb}. 
    Then, with probability at least $1 - \delta$,~\Cref{alg:posg_ucb} returns a model $\theta^\epsilon$ and a policy $\pi$ which is an $\varepsilon$-approximate equilibrium (either NE or CCE). That is,
    \begin{equation*}
        V_{\theta^*}^i(\pi) \geq V_{\theta^*}^{i, \dagger}(\pi^{-i}) - \varepsilon, \, \forall i \in [N].
    \end{equation*}
    In addition, the algorithm terminates with a sample complexity of,
    \begin{equation*}
        \tilde{O}\paren{\paren{r + \frac{Q_A^2 H}{\gamma^2}}  \cdot \frac{r d H^3 \cdot \max_{s \in \calA} \abs{\bbX_{s}}^2\cdot  Q_A^4 \beta}{\gamma^4 \epsilon^2}}.
    \end{equation*}
\end{theorem}

\begin{proof}
    The proof is given in~\Cref{sec:game_thm_proof}.
\end{proof}

To apply this algorithm to a partially-observable sequential game, we can use the generalized PSR parameterization constructed in \Cref{sec:post_psrep}.

\begin{corollary}\label{cor:posg_ucb}
    Suppose a partially-observable sequential game is $m$-step $\alpha$-robustly $\calI^\dagger$-weakly revealing (\Cref{ass:calIdagger_weakly_revealing}). Applying~\Cref{alg:posg_ucb} to this PSR representation, with parameters $p_{\min}, \lambda, \alpha, \beta$ chosen as in~\Cref{theorem:posg_ucb}, returns a $\varepsilon$-approximate equilibrium $\pi$ with a sample complexity of,
    \begin{equation*}
        \tilde{O}\paren{\paren{1 + \frac{Q_A^2 H}{\alpha^2}} \cdot  \frac{\max_{h} | \bbI_h^\dagger| ^7 \cdot \max_h{\abs{\bbQ_h^m}} \cdot H^5 \cdot \max_{s \in \calA} \abs{\bbX_{s}}^2 \cdot \max_{s \in \calU} \abs{\bbX_s} \cdot Q_A^4}{\alpha^4 \epsilon^2}}.
    \end{equation*}
\end{corollary}

\section{Conclusion}\label{sec:discussion}

In this paper, we studied the role of information structure in learning sequential decision-making problems, arguing that explicitly modeling information structures leads to a deeper analysis of reinforcement learning problems. We proposed new models with an explicit representation of information structure which we refer to as partially-observable sequential teams and games. 
Through a DAG representation of the information structure, we characterized the rank of any sequential decision-making problem through a graph-theoretic quantity which can be interpreted as an information-structural ``state''. This analysis gives a condition in terms of the information structure for when learning is tractable. To facilitate sample-efficient learning, we also proposed a generalization of predictive state representation which can be used to construct compact representations for POSTs and POSGs. Finally, we proposed provably sample-efficient algorithms for learning optimal policies in the team setting and equilibria in the game setting.


\section*{Acknowledgment}

We thank Yusong Zhu for his input on some of the technical lemmas in~\Cref{ssec:proof_sublinear_est,ssec:proof_ucb_tv}.

\aanote[margin, noinline]{go through bibliography and update arxiv references if published?}
\printbibliography

@misc{chenPartiallyObservableRL2022,
  title = {Partially {{Observable RL}} with {{B-Stability}}: {{Unified Structural Condition}} and {{Sharp Sample-Efficient Algorithms}}},
  shorttitle = {Partially {{Observable RL}} with {{B-Stability}}},
  author = {Chen, Fan and Bai, Yu and Mei, Song},
  year = {2022},
  month = sep,
  number = {arXiv:2209.14990},
  eprint = {2209.14990},
  primaryclass = {cs, math, stat},
  publisher = {{arXiv}},
  doi = {10.48550/arXiv.2209.14990},
  urldate = {2022-10-18},
  archiveprefix = {arxiv}
}

@incollection{geerRatesConvergenceMaximum2006,
  title = {Rates of Convergence for Maximum Likelihood Estimators},
  booktitle = {Applications of Empirical Process Theory},
  author = {van de Geer, Sara},
  year = {2006},
  series = {Cambridge Series on Statistical and Probabilistic Mathematics},
  edition = {Reprint},
  number = {6},
  publisher = {{Cambridge Univ. Pr}},
  address = {{Cambridge}},
  isbn = {978-0-521-12325-9},
  langid = {english}
}

@misc{huangProvablyEfficientUCBtype2023,
  title = {Provably {{Efficient UCB-type Algorithms For Learning Predictive State Representations}}},
  author = {Huang, Ruiquan and Liang, Yingbin and Yang, Jing},
  year = {2023},
  month = jul,
  number = {arXiv:2307.00405},
  eprint = {2307.00405},
  primaryclass = {cs, stat},
  publisher = {{arXiv}},
  urldate = {2023-07-22},
  archiveprefix = {arxiv},
  langid = {english}
}

@article{jaegerObservableOperatorModels2000,
  title = {Observable {{Operator Models}} for {{Discrete Stochastic Time Series}}},
  author = {Jaeger, Herbert},
  year = {2000},
  month = jun,
  journal = {Neural Computation},
  volume = {12},
  number = {6},
  pages = {1371--1398},
  issn = {0899-7667, 1530-888X},
  doi = {10.1162/089976600300015411},
  urldate = {2023-07-16},
  langid = {english}
}

@misc{jinSampleEfficientReinforcementLearning2020,
  title = {Sample-Efficient Reinforcement Learning of Undercomplete POMDPs},
  author = {Jin, Chi and Kakade, Sham M. and Krishnamurthy, Akshay and Liu, Qinghua},
  year = {2020},
  month = oct,
  number = {arXiv:2006.12484},
  eprint = {2006.12484},
  primaryclass = {cs, math, stat},
  publisher = {{arXiv}},
  urldate = {2023-07-07},
  archiveprefix = {arxiv}
}

@article{littmanPredictiveRepresentationsState,
  title={Predictive representations of state},
  author={Littman, Michael and Sutton, Richard S},
  journal={Advances in neural information processing systems},
  volume={14},
  year={2001}
}

@misc{liuOptimisticMLEGeneric2022,
  title = {Optimistic {{MLE}} -- {{A Generic Model-based Algorithm}} for {{Partially Observable Sequential Decision Making}}},
  author = {Liu, Qinghua and Netrapalli, Praneeth and Szepesv{\'a}ri, Csaba and Jin, Chi},
  year = {2022},
  month = nov,
  number = {arXiv:2209.14997},
  eprint = {2209.14997},
  primaryclass = {cs, stat},
  publisher = {{arXiv}},
  doi = {10.48550/arXiv.2209.14997},
  urldate = {2023-07-24},
  archiveprefix = {arxiv}
}

@misc{liuSampleEfficientReinforcementLearning2022,
  title = {Sample-{{Efficient Reinforcement Learning}} of {{Partially Observable Markov Games}}},
  author = {Liu, Qinghua and Szepesv{\'a}ri, Csaba and Jin, Chi},
  year = {2022},
  month = oct,
  number = {arXiv:2206.01315},
  eprint = {2206.01315},
  primaryclass = {cs, stat},
  publisher = {{arXiv}},
  urldate = {2023-07-14},
  archiveprefix = {arxiv},
  langid = {english}
}

@misc{liuWhenPartiallyObservable2022,
  title = {When {{Is Partially Observable Reinforcement Learning Not Scary}}?},
  author = {Liu, Qinghua and Chung, Alan and Szepesv{\'a}ri, Csaba and Jin, Chi},
  year = {2022},
  month = may,
  number = {arXiv:2204.08967},
  eprint = {2204.08967},
  primaryclass = {cs, eess, stat},
  publisher = {{arXiv}},
  doi = {10.48550/arXiv.2204.08967},
  urldate = {2023-06-26},
  archiveprefix = {arxiv}
}

@inproceedings{mahajanGraphicalModelingApproach2009,
  title = {A {{Graphical}} Modeling Approach to Simplifying Sequential Teams},
  booktitle = {2009 7th {{International Symposium}} on {{Modeling}} and {{Optimization}} in {{Mobile}}, {{Ad Hoc}}, and {{Wireless Networks}}},
  author = {Mahajan, Aditya and Tatikonda, Sekhar},
  year = {2009},
  month = jun,
  pages = {1--8},
  doi = {10.1109/WIOPT.2009.5291560},
}

@inproceedings{mahajanInformationStructuresOptimal2012,
  title = {Information Structures in Optimal Decentralized Control},
  booktitle = {2012 {{IEEE}} 51st {{IEEE Conference}} on {{Decision}} and {{Control}} ({{CDC}})},
  author = {Mahajan, Aditya and Martins, Nuno C. and Rotkowitz, Michael C. and Yuksel, Serdar},
  year = {2012},
  month = dec,
  pages = {1291--1306},
  publisher = {{IEEE}},
  address = {{Maui, HI, USA}},
  doi = {10.1109/CDC.2012.6425819},
  urldate = {2023-07-16},
  isbn = {978-1-4673-2066-5},
  langid = {english}
}

@incollection{nayyarCommonInformationApproachDecentralized2014,
  title = {The {{Common-Information Approach}} to {{Decentralized Stochastic Control}}},
  booktitle = {Information and {{Control}} in {{Networks}}},
  author = {Nayyar, Ashutosh and Mahajan, Aditya and Teneketzis, Demosthenis},
  editor = {Como, Giacomo and Bernhardsson, Bo and Rantzer, Anders},
  year = {2014},
  volume = {450},
  pages = {123--156},
  publisher = {{Springer International Publishing}},
  address = {{Cham}},
  doi = {10.1007/978-3-319-02150-8_4},
  urldate = {2023-06-08},
  isbn = {978-3-319-02149-2},
  langid = {english}
}

@article{singhPredictiveStateRepresentations2004,
  title = {Predictive {{State Representations}}: {{A New Theory}} for {{Modeling Dynamical Systems}}},
  author = {Singh, Satinder and James, Michael R and Rudary, Matthew R},
  year = {2004},
  journal = {Proceedings of the Twentieth Conference on Uncertainty in Artificial Intelligence (UAI2004)},
  langid = {english}
}

@phdthesis{tatikondaControlCommunicationConstraints2000,
  type = {Thesis},
  title = {Control under Communication Constraints},
  author = {Tatikonda, Sekhar Chandra},
  year = {2000},
  urldate = {2023-05-22},
  copyright = {M.I.T. theses are protected by copyright. They may be viewed from this source for any purpose, but reproduction or distribution in any format is prohibited without written permission. See provided URL for inquiries about permission.},
  langid = {english},
  school = {Massachusetts Institute of Technology}
}

@misc{ueharaProvablyEfficientReinforcement2022,
  title = {Provably Efficient Reinforcement Learning in Partially Observable Dynamical Systems},
  author = {Uehara, Masatoshi and Sekhari, Ayush and Lee, Jason D. and Kallus, Nathan and Sun, Wen},
  year = {2022},
  month = jun,
  number = {arXiv:2206.12020},
  eprint = {2206.12020},
  primaryclass = {cs, math, stat},
  publisher = {{arXiv}},
  urldate = {2023-06-16},
  archiveprefix = {arxiv}
}

@misc{wangEmbedControlPartially2022,
  title = {Embed to Control Partially Observed Systems: Representation Learning with Provable Sample Efficiency},
  shorttitle = {Embed to {{Control Partially Observed Systems}}},
  author = {Wang, Lingxiao and Cai, Qi and Yang, Zhuoran and Wang, Zhaoran},
  year = {2022},
  month = may,
  number = {arXiv:2205.13476},
  eprint = {2205.13476},
  primaryclass = {cs, eess, stat},
  publisher = {{arXiv}},
  urldate = {2023-08-05},
  archiveprefix = {arxiv}
}

@article{witsenhausenEquivalentStochasticControl1988,
  title = {Equivalent Stochastic Control Problems},
  author = {Witsenhausen, Hans S},
  year = {1988},
  month = feb,
  journal = {Mathematics of Control, Signals, and Systems},
  volume = {1},
  number = {1},
  pages = {3--11},
  issn = {0932-4194, 1435-568X},
  doi = {10.1007/BF02551232},
  urldate = {2023-08-16},
  langid = {english}
}

@incollection{witsenhausenIntrinsicModelDiscrete1975a,
  title = {The {{Intrinsic Model}} for {{Discrete Stochastic Control}}: {{Some Open Problems}}},
  shorttitle = {The {{Intrinsic Model}} for {{Discrete Stochastic Control}}},
  booktitle = {Control {{Theory}}, {{Numerical Methods}} and {{Computer Systems Modelling}}},
  author = {Witsenhausen, Hans S},
  editor = {Beckmann, M. and K{\"u}nzi, H. P. and Bensoussan, A. and Lions, J. L.},
  year = {1975},
  volume = {107},
  pages = {322--335},
  publisher = {{Springer Berlin Heidelberg}},
  address = {{Berlin, Heidelberg}},
  doi = {10.1007/978-3-642-46317-4_24},
  urldate = {2023-08-16},
  isbn = {978-3-540-07020-7 978-3-642-46317-4}
}

@misc{zhanPACReinforcementLearning2022,
  title = {{{PAC Reinforcement Learning}} for {{Predictive State Representations}}},
  author = {Zhan, Wenhao and Uehara, Masatoshi and Sun, Wen and Lee, Jason D.},
  year = {2022},
  month = aug,
  number = {arXiv:2207.05738},
  eprint = {2207.05738},
  primaryclass = {cs},
  publisher = {{arXiv}},
  urldate = {2023-05-22},
  archiveprefix = {arxiv}
}

@misc{zhongGECUnifiedFramework2023,
  title = {{{GEC}}: {{A Unified Framework}} for {{Interactive Decision Making}} in {{MDP}}, {{POMDP}}, and {{Beyond}}},
  shorttitle = {{{GEC}}},
  author = {Zhong, Han and Xiong, Wei and Zheng, Sirui and Wang, Liwei and Wang, Zhaoran and Yang, Zhuoran and Zhang, Tong},
  year = {2023},
  month = jun,
  number = {arXiv:2211.01962},
  eprint = {2211.01962},
  primaryclass = {cs, math, stat},
  publisher = {{arXiv}},
  urldate = {2023-10-22},
  archiveprefix = {arxiv}
}

@article{dani2008stochastic,
  title={Stochastic linear optimization under bandit feedback},
  author={Dani, Varsha and Hayes, Thomas P and Kakade, Sham M},
  year={2008}
}

@article{abbasi2011improved,
  title={Improved algorithms for linear stochastic bandits},
  author={Abbasi-Yadkori, Yasin and P{\'a}l, D{\'a}vid and Szepesv{\'a}ri, Csaba},
  journal={Advances in neural information processing systems},
  volume={24},
  year={2011}
}

@article{carpentier2020elliptical,
  title={The elliptical potential lemma revisited},
  author={Carpentier, Alexandra and Vernade, Claire and Abbasi-Yadkori, Yasin},
  eprint={2010.10182},
  archivePrefix={arXiv},
  primaryClass={stat.ML},
  year={2020}
}

@article{witsenhausen1971separation,
  title={Separation of estimation and control for discrete time systems},
  author={Witsenhausen, Hans S},
  journal={Proceedings of the IEEE},
  volume={59},
  number={11},
  pages={1557--1566},
  year={1971},
  publisher={IEEE}
}

@article{ho1972team,
  author={Ho, Y. and Chu, K.},
  title={Team decision theory and information structures in optimal control problems--Part I},
  journal={IEEE Transactions on Automatic Control},
  year={1972},
  volume={17},
  number={1},
  pages={15-22},
  doi={10.1109/TAC.1972.1099850}
  }

@article{witsenhausen1971information,
  title={On information structures, feedback and causality},
  author={Witsenhausen, Hans S},
  journal={SIAM Journal on Control},
  volume={9},
  number={2},
  pages={149--160},
  year={1971},
  publisher={SIAM}
}

@article{yoshikawa1978decomposition,
  title={Decomposition of dynamic team decision problems},
  author={Yoshikawa, Tsuneo},
  journal={IEEE Transactions on Automatic Control},
  volume={23},
  number={4},
  pages={627--632},
  year={1978},
  publisher={IEEE}
}

@article{andersland1992information,
  title={Information structures, causality, and nonsequential stochastic control I: Design-independent properties},
  author={Andersland, Mark S and Teneketzis, Demosthenis},
  journal={SIAM journal on control and optimization},
  volume={30},
  number={6},
  pages={1447--1475},
  year={1992},
  publisher={SIAM}
}

@article{teneketzisInformationStructuresNonsequential1996,
  title = {On Information Structures and Nonsequential Stochastic Control},
  author = {Teneketzis, Demosthenis},
  year = {1996},
  journal = {CWI Quarterly},
  volume = {9},
  number = {3},
  pages = {241--260},
  publisher = {{Stichting Mathematisch Centrum}},
  isbn = {0922-5366}
}

@inproceedings{lattimore2012pac,
  title={PAC bounds for discounted MDPs},
  author={Lattimore, Tor and Hutter, Marcus},
  booktitle={Algorithmic Learning Theory: 23rd International Conference, ALT 2012, Lyon, France, October 29-31, 2012. Proceedings 23},
  pages={320--334},
  year={2012},
  organization={Springer}
}

@article{sutton2008convergent,
  title={A convergent $ o (n) $ temporal-difference algorithm for off-policy learning with linear function approximation},
  author={Sutton, Richard S and Maei, Hamid and Szepesv{\'a}ri, Csaba},
  journal={Advances in neural information processing systems},
  volume={21},
  year={2008}
}

@article{munos2008finite,
  title={Finite-Time Bounds for Fitted Value Iteration.},
  author={Munos, R{\'e}mi and Szepesv{\'a}ri, Csaba},
  journal={Journal of Machine Learning Research},
  volume={9},
  number={5},
  year={2008}
}

@article{singh2000convergence,
  title={Convergence results for single-step on-policy reinforcement-learning algorithms},
  author={Singh, Satinder and Jaakkola, Tommi and Littman, Michael L and Szepesv{\'a}ri, Csaba},
  journal={Machine learning},
  volume={38},
  pages={287--308},
  year={2000},
  publisher={Springer}
}

@inproceedings{abbasi2011regret,
  title={Regret bounds for the adaptive control of linear quadratic systems},
  author={Abbasi-Yadkori, Yasin and Szepesv{\'a}ri, Csaba},
  booktitle={Proceedings of the 24th Annual Conference on Learning Theory},
  pages={1--26},
  year={2011},
  organization={JMLR Workshop and Conference Proceedings}
}

@article{nash1951non,
  title={Non-cooperative games},
  author={Nash, John},
  journal={Annals of mathematics},
  pages={286--295},
  year={1951},
  publisher={JSTOR}
}

@article{koberReinforcementLearningRobotics2013,
  title = {Reinforcement Learning in Robotics: {{A}} Survey},
  author = {Kober, Jens and Bagnell, J. Andrew and Peters, Jan},
  year = {2013},
  journal = {The International Journal of Robotics Research},
  volume = {32},
  number = {11},
  pages = {1238--1274},
  publisher = {{SAGE Publications Sage UK: London, England}},
  isbn = {0278-3649}
}

@article{mnihHumanlevelControlDeep2015,
  title = {Human-Level Control through Deep Reinforcement Learning},
  author = {Mnih, Volodymyr and Kavukcuoglu, Koray and Silver, David and Rusu, Andrei A. and Veness, Joel and Bellemare, Marc G. and Graves, Alex and Riedmiller, Martin and Fidjeland, Andreas K. and Ostrovski, Georg and Petersen, Stig and Beattie, Charles and Sadik, Amir and Antonoglou, Ioannis and King, Helen and Kumaran, Dharshan and Wierstra, Daan and Legg, Shane and Hassabis, Demis},
  year = {2015},
  month = feb,
  journal = {Nature},
  volume = {518},
  number = {7540},
  pages = {529--533},
  publisher = {{Nature Publishing Group}},
  issn = {1476-4687},
  doi = {10.1038/nature14236},
  urldate = {2024-01-17},
  copyright = {2015 Springer Nature Limited},
  langid = {english}
}

@misc{mnihPlayingAtariDeep2013,
  title = {Playing {{Atari}} with {{Deep Reinforcement Learning}}},
  author = {Mnih, Volodymyr and Kavukcuoglu, Koray and Silver, David and Graves, Alex and Antonoglou, Ioannis and Wierstra, Daan and Riedmiller, Martin},
  year = {2013},
  month = dec,
  number = {arXiv:1312.5602},
  eprint = {1312.5602},
  primaryclass = {cs},
  publisher = {{arXiv}},
  doi = {10.48550/arXiv.1312.5602},
  urldate = {2024-01-17},
  archiveprefix = {arxiv}
}

@misc{shalev-shwartzSafeMultiAgentReinforcement2016,
  title = {Safe, {{Multi-Agent}}, {{Reinforcement Learning}} for {{Autonomous Driving}}},
  author = {{Shalev-Shwartz}, Shai and Shammah, Shaked and Shashua, Amnon},
  year = {2016},
  month = oct,
  number = {arXiv:1610.03295},
  eprint = {1610.03295},
  primaryclass = {cs, stat},
  publisher = {{arXiv}},
  doi = {10.48550/arXiv.1610.03295},
  urldate = {2022-07-31},
  archiveprefix = {arxiv}
}

@article{silverMasteringGameGo2016,
  title = {Mastering the Game of {{Go}} with Deep Neural Networks and Tree Search},
  author = {Silver, David and Huang, Aja and Maddison, Chris J. and Guez, Arthur and Sifre, Laurent and {van den Driessche}, George and Schrittwieser, Julian and Antonoglou, Ioannis and Panneershelvam, Veda and Lanctot, Marc and Dieleman, Sander and Grewe, Dominik and Nham, John and Kalchbrenner, Nal and Sutskever, Ilya and Lillicrap, Timothy and Leach, Madeleine and Kavukcuoglu, Koray and Graepel, Thore and Hassabis, Demis},
  year = {2016},
  month = jan,
  journal = {Nature},
  volume = {529},
  number = {7587},
  pages = {484--489},
  publisher = {{Nature Publishing Group}},
  issn = {1476-4687},
  doi = {10.1038/nature16961},
  urldate = {2022-07-31},
  copyright = {2016 Nature Publishing Group, a division of Macmillan Publishers Limited. All Rights Reserved.},
  langid = {english}
}

@article{vinyalsAlphastarMasteringRealtime2019,
  title = {Alphastar: {{Mastering}} the Real-Time Strategy Game Starcraft Ii},
  author = {Vinyals, Oriol and Babuschkin, Igor and Chung, Junyoung and Mathieu, Michael and Jaderberg, Max and Czarnecki, Wojciech M. and Dudzik, Andrew and Huang, Aja and Georgiev, Petko and Powell, Richard},
  year = {2019},
  journal = {DeepMind blog},
  volume = {2}
}

@inproceedings{zhangReinforcementLearningMultiagent2021,
  title = {Reinforcement {{Learning}} under a {{Multi-agent Predictive State Representation Model}}: {{Method}} and {{Theory}}},
  shorttitle = {Reinforcement {{Learning}} under a {{Multi-agent Predictive State Representation Model}}},
  booktitle = {International {{Conference}} on {{Learning Representations}}},
  author = {Zhang, Zhi and Yang, Zhuoran and Liu, Han and Tokekar, Pratap and Huang, Furong},
  year = {2021},
  month = oct,
  urldate = {2024-01-16},
  langid = {english}
}

@article{hoEquivalenceInformationStructures1973,
  title = {On the Equivalence of Information Structures in Static and Dynamic Teams},
  author = {Ho, Y. and Chu, K.},
  year = {1973},
  month = apr,
  journal = {IEEE Transactions on Automatic Control},
  volume = {18},
  number = {2},
  pages = {187--188},
  issn = {0018-9286},
  doi = {10.1109/TAC.1973.1100259},
  urldate = {2024-01-17},
  langid = {english}
}

@article{saldiGeometryInformationStructures2022,
  title = {Geometry of Information Structures, Strategic Measures and Associated Stochastic Control Topologies},
  author = {Saldi, Naci and Y{\"u}ksel, Serdar},
  year = {2022},
  month = jan,
  journal = {Probability Surveys},
  volume = {19},
  number = {none},
  issn = {1549-5787},
  doi = {10.1214/20-PS356},
  urldate = {2022-10-03}
}

@book{yukselStochasticTeamsGames2023,
  title = {Stochastic Teams, Games and Control under Information Constraints},
  author = {Yuksel, Serdar and Basar, Tamer},
  year = {2023},
  publisher = {{Springer}},
  langid = {english}
}

@inproceedings{efroniProvableReinforcementLearning2022,
  title = {Provable Reinforcement Learning with a Short-Term Memory},
  booktitle = {International {{Conference}} on {{Machine Learning}}},
  author = {Efroni, Yonathan and Jin, Chi and Krishnamurthy, Akshay and Miryoosefi, Sobhan},
  year = {2022},
  pages = {5832--5850},
  publisher = {{PMLR}},
  isbn = {2640-3498}
}

@misc{golowichPlanningObservablePomdps2022a,
  title={Planning in Observable POMDPs in Quasipolynomial Time}, 
  author={Noah Golowich and Ankur Moitra and Dhruv Rohatgi},
  year={2022},
  eprint={2201.04735},
  archivePrefix={arXiv},
  primaryClass={cs.LG}
}

@article{krishnamurthyPACReinforcementLearning2016,
  title = {{{PAC}} Reinforcement Learning with Rich Observations},
  author = {Krishnamurthy, Akshay and Agarwal, Alekh and Langford, John},
  year = {2016},
  journal = {Advances in Neural Information Processing Systems},
  volume = {29}
}

@article{mundhenkComplexityFinitehorizonMarkov2000,
  title = {Complexity of Finite-Horizon {{Markov}} Decision Process Problems},
  author = {Mundhenk, Martin and Goldsmith, Judy and Lusena, Christopher and Allender, Eric},
  year = {2000},
  journal = {Journal of the ACM (JACM)},
  volume = {47},
  number = {4},
  pages = {681--720},
  publisher = {{ACM New York, NY, USA}},
  isbn = {0004-5411}
}

@article{papadimitriouComplexityMarkovDecision1987,
  title = {The Complexity of {{Markov}} Decision Processes},
  author = {Papadimitriou, Christos H. and Tsitsiklis, John N.},
  year = {1987},
  journal = {Mathematics of operations research},
  volume = {12},
  number = {3},
  pages = {441--450},
  publisher = {{INFORMS}},
  isbn = {0364-765X}
}

@article{lassisComputationalComplexityStochastic2012,
  title = {On the Computational Complexity of Stochastic Controller Optimization in {{POMDPs}}},
  author = {Vlassis, Nikos and Littman, Michael L. and Barber, David},
  year = {2012},
  journal = {ACM Transactions on Computation Theory (TOCT)},
  volume = {4},
  number = {4},
  pages = {1--8},
  publisher = {{ACM New York, NY, USA}},
  isbn = {1942-3454}
}

@article{bootsClosingLearningplanningLoop2011,
  title = {Closing the Learning-Planning Loop with Predictive State Representations},
  author = {Boots, Byron and Siddiqi, Sajid M. and Gordon, Geoffrey J.},
  year = {2011},
  journal = {The International Journal of Robotics Research},
  volume = {30},
  number = {7},
  pages = {954--966},
  publisher = {{SAGE Publications Sage UK: London, England}},
  isbn = {0278-3649}
}

@article{hefnySupervisedLearningDynamical2015,
  title = {Supervised Learning for Dynamical System Learning},
  author = {Hefny, Ahmed and Downey, Carlton and Gordon, Geoffrey J.},
  year = {2015},
  journal = {Advances in neural information processing systems},
  volume = {28}
}

@inproceedings{jamesPlanningPredictiveState2004,
  title = {Planning with Predictive State Representations},
  booktitle = {2004 {{International Conference}} on {{Machine Learning}} and {{Applications}}, 2004. {{Proceedings}}.},
  author = {James, Michael R. and Singh, Satinder and Littman, Michael L.},
  year = {2004},
  pages = {304--311},
  publisher = {{IEEE}},
  isbn = {0-7803-8823-2}
}

@article{jiangCompletingStateRepresentations2018,
  title = {Completing State Representations Using Spectral Learning},
  author = {Jiang, Nan and Kulesza, Alex and Singh, Satinder},
  year = {2018},
  journal = {Advances in Neural Information Processing Systems},
  volume = {31}
}

@article{mccrackenOnlineDiscoveryLearning2005,
  title = {Online Discovery and Learning of Predictive State Representations},
  author = {McCracken, Peter and Bowling, Michael},
  year = {2005},
  journal = {Advances in neural information processing systems},
  volume = {18}
}

@inproceedings{singhLearningPredictiveState2003,
  title = {Learning Predictive State Representations},
  booktitle = {Proceedings of the 20th {{International Conference}} on {{Machine Learning}} ({{ICML-03}})},
  author = {Singh, Satinder P. and Littman, Michael L. and Jong, Nicholas K. and Pardoe, David and Stone, Peter},
  year = {2003},
  pages = {712--719}
}

@article{agrawalOptimisticPosteriorSampling2017,
  title = {Optimistic Posterior Sampling for Reinforcement Learning: Worst-Case Regret Bounds},
  author = {Agrawal, Shipra and Jia, Randy},
  year = {2017},
  journal = {Advances in Neural Information Processing Systems},
  volume = {30}
}

@book{sutton2018reinforcement,
  title={Reinforcement learning: An introduction},
  author={Sutton, Richard S and Barto, Andrew G},
  year={2018},
  publisher={MIT press}
}

@article{golowich2022learning,
  title={Learning in observable pomdps, without computationally intractable oracles},
  author={Golowich, Noah and Moitra, Ankur and Rohatgi, Dhruv},
  journal={Advances in Neural Information Processing Systems},
  volume={35},
  pages={1458--1473},
  year={2022}
}

@inproceedings{golowich2023planning,
  title={Planning and Learning in Partially Observable Systems via Filter Stability},
  author={Golowich, Noah and Moitra, Ankur and Rohatgi, Dhruv},
  booktitle={Proceedings of the 55th Annual ACM Symposium on Theory of Computing},
  pages={349--362},
  year={2023}
}

@article{auerNearoptimalRegretBounds2008,
  title = {Near-Optimal Regret Bounds for Reinforcement Learning},
  author = {Auer, Peter and Jaksch, Thomas and Ortner, Ronald},
  year = {2008},
  journal = {Advances in neural information processing systems},
  volume = {21}
}

@inproceedings{azarMinimaxRegretBounds2017,
  title = {Minimax Regret Bounds for Reinforcement Learning},
  booktitle = {International {{Conference}} on {{Machine Learning}}},
  author = {Azar, Mohammad Gheshlaghi and Osband, Ian and Munos, R{\'e}mi},
  year = {2017},
  pages = {263--272},
  publisher = {{PMLR}},
  isbn = {2640-3498}
}

@inproceedings{mosselLearningNonsingularPhylogenies2005,
  title = {Learning Nonsingular Phylogenies and Hidden {{Markov}} Models},
  booktitle = {Proceedings of the Thirty-Seventh Annual {{ACM}} Symposium on {{Theory}} of Computing},
  author = {Mossel, Elchanan and Roch, S{\'e}bastien},
  year = {2005},
  pages = {366--375}
}

@article{rashidinejadBridgingOfflineReinforcement2021,
  title = {Bridging Offline Reinforcement Learning and Imitation Learning: {{A}} Tale of Pessimism},
  author = {Rashidinejad, Paria and Zhu, Banghua and Ma, Cong and Jiao, Jiantao and Russell, Stuart},
  year = {2021},
  journal = {Advances in Neural Information Processing Systems},
  volume = {34},
  pages = {11702--11716}
}

@article{brownSuperhumanAIMultiplayer2019,
  title = {Superhuman {{AI}} for Multiplayer Poker},
  author = {Brown, Noam and Sandholm, Tuomas},
  year = {2019},
  journal = {Science},
  volume = {365},
  number = {6456},
  pages = {885--890},
  publisher = {{American Association for the Advancement of Science}},
  isbn = {0036-8075}
}

@inproceedings{farinaModelfreeOnlineLearning2021,
  title = {Model-Free Online Learning in Unknown Sequential Decision Making Problems and Games},
  booktitle = {Proceedings of the {{AAAI Conference}} on {{Artificial Intelligence}}},
  author = {Farina, Gabriele and Sandholm, Tuomas},
  year = {2021},
  volume = {35},
  pages = {5381--5390},
  isbn = {2374-3468}
}

@article{kozunoModelfreeLearningTwoplayer2021,
  title = {Model-Free Learning for Two-Player Zero-Sum Partially Observable Markov Games with Perfect Recall},
  author = {Kozuno, Tadashi and M{\'e}nard, Pierre and Munos, R{\'e}mi and Valko, Michal},
  year = {2021},
  eprint = {2106.06279},
  archiveprefix = {arxiv}
}

@article{zinkevichRegretMinimizationGames2007,
  title = {Regret Minimization in Games with Incomplete Information},
  author = {Zinkevich, Martin and Johanson, Michael and Bowling, Michael and Piccione, Carmelo},
  year = {2007},
  journal = {Advances in neural information processing systems},
  volume = {20}
}

@article{brafmanRmaxaGeneralPolynomial2002,
  title = {R-Max-a General Polynomial Time Algorithm for near-Optimal Reinforcement Learning},
  author = {Brafman, Ronen I. and Tennenholtz, Moshe},
  year = {2002},
  journal = {Journal of Machine Learning Research},
  volume = {3},
  number = {Oct},
  pages = {213--231}
}

@inproceedings{baiNearoptimalReinforcementLearning2020,
  title = {Near-Optimal Reinforcement Learning with Self-Play},
  booktitle = {Advances in Neural Information Processing Systems},
  author = {Bai, Yu and Jin, Chi and Yu, Tiancheng},
  editor = {Larochelle, H. and Ranzato, M. and Hadsell, R. and Balcan, M.F. and Lin, H.},
  year = {2020},
  volume = {33},
  pages = {2159--2170},
  publisher = {{Curran Associates, Inc.}}
}

@article{songWhenCanWe2021,
  title = {When Can We Learn General-Sum {{Markov}} Games with a Large Number of Players Sample-Efficiently?},
  author = {Song, Ziang and Mei, Song and Bai, Yu},
  year = {2021},
  eprint = {2110.04184},
  archiveprefix = {arxiv}
}

@incollection{vermaCausalNetworksSemantics1990,
  title={Causal networks: Semantics and expressiveness},
  author={Verma, Thomas and Pearl, Judea},
  booktitle={Machine intelligence and pattern recognition},
  volume={9},
  pages={69--76},
  year={1990},
  publisher={Elsevier}
}

@book{kollerProbabilisticGraphicalModels2009,
  title = {Probabilistic Graphical Models: Principles and Techniques},
  author = {Koller, Daphne and Friedman, Nir},
  year = {2009},
  publisher = {{MIT press}},
  isbn = {0-262-01319-3}
}

@article{pearlModelsReasoningInference2000,
  title = {Models, Reasoning and Inference},
  author = {Pearl, Judea},
  year = {2000},
  journal = {Cambridge, UK: CambridgeUniversityPress},
  volume = {19},
  number = {2},
  pages = {3}
}

@book{spirtesCausationPredictionSearch2000,
  title = {Causation, Prediction, and Search},
  author = {Spirtes, Peter and Glymour, Clark N. and Scheines, Richard},
  year = {2000},
  publisher = {{MIT press}},
  isbn = {0-262-19440-6}
}

@article{wrightMethodPathCoefficients1934,
  title = {The Method of Path Coefficients},
  author = {Wright, Sewall},
  year = {1934},
  journal = {The annals of mathematical statistics},
  volume = {5},
  number = {3},
  pages = {161--215},
  publisher = {{JSTOR}},
  isbn = {0003-4851}
}

@article{witsenhausenStandardFormSequential1973,
  title={A standard form for sequential stochastic control},
  author={Witsenhausen, Hans S},
  journal={Mathematical systems theory},
  volume={7},
  number={1},
  pages={5--11},
  year={1973},
  publisher={Springer}
}

@article{geigerIdentifyingIndependenceBayesian1990,
  title={Identifying independence in Bayesian networks},
  author={Geiger, Dan and Verma, Thomas and Pearl, Judea},
  journal={Networks},
  volume={20},
  number={5},
  pages={507--534},
  year={1990},
  publisher={Wiley Online Library}
}

@techreport{danAxiomsAlgorithmsInferences1989,
  title={Axioms and algorithms for inferences involving conditional independence},
  author={Dan, Geiger and Pearl, Judea},
  year={1989},
  institution={Tech. Rep. CSD 890031, R-119-I, Cognitive Systems Laboratory, University of~…}
}

@article{liu2024maximize,
  title={Maximize to explore: One objective function fusing estimation, planning, and exploration},
  author={Liu, Zhihan and Lu, Miao and Xiong, Wei and Zhong, Han and Hu, Hao and Zhang, Shenao and Zheng, Sirui and Yang, Zhuoran and Wang, Zhaoran},
  journal={Advances in Neural Information Processing Systems},
  volume={36},
  year={2024}
}

@misc{qiu2023posterior,
  title={Posterior Sampling for Competitive RL: Function Approximation and Partial Observation},
  author={Shuang Qiu and Ziyu Dai and Han Zhong and Zhaoran Wang and Zhuoran Yang and Tong Zhang},
  year={2023},
  eprint={2310.19861},
  archivePrefix={arXiv},
  primaryClass={cs.LG}
}

@misc{liu2024partially,
  title={Partially Observable Multi-agent RL with (Quasi-)Efficiency: The Blessing of Information Sharing},
  author={Xiangyu Liu and Kaiqing Zhang},
  year={2024},
  eprint={2308.08705},
  archivePrefix={arXiv},
  primaryClass={cs.LG}
}

@misc{guo2023provably,
  title={Provably Efficient Representation Learning with Tractable Planning in Low-Rank POMDP}, 
  author={Jiacheng Guo and Zihao Li and Huazheng Wang and Mengdi Wang and Zhuoran Yang and Xuezhou Zhang},
  year={2023},
  eprint={2306.12356},
  archivePrefix={arXiv},
  primaryClass={cs.LG}
}

@inproceedings{cai2022reinforcement,
  title={Reinforcement learning from partial observation: Linear function approximation with provable sample efficiency},
  author={Cai, Qi and Yang, Zhuoran and Wang, Zhaoran},
  booktitle={International Conference on Machine Learning},
  pages={2485--2522},
  year={2022},
  organization={PMLR}
}

@misc{zhang2023provable,
  title={Provable Representation with Efficient Planning for Partially Observable Reinforcement Learning},
  author={Zhang, Hongming and Ren, Tongzheng and Xiao, Chenjun and Schuurmans, Dale and Dai, Bo},
  year={2023},
  eprint={2311.12244},
  archivePrefix={arXiv},
  primaryClass={cs.LG}
}

@article{nayyar2013decentralized,
  title={Decentralized stochastic control with partial history sharing: A common information approach},
  author={Nayyar, Ashutosh and Mahajan, Aditya and Teneketzis, Demosthenis},
  journal={IEEE Transactions on Automatic Control},
  volume={58},
  number={7},
  pages={1644--1658},
  year={2013},
  publisher={IEEE}
}

@article{nayyar2010optimal,
  title={Optimal control strategies in delayed sharing information structures},
  author={Nayyar, Ashutosh and Mahajan, Aditya and Teneketzis, Demosthenis},
  journal={IEEE Transactions on Automatic Control},
  volume={56},
  number={7},
  pages={1606--1620},
  year={2010},
  publisher={IEEE}
}

@article{nayyar2013common,
  title={Common information based Markov perfect equilibria for stochastic games with asymmetric information: Finite games},
  author={Nayyar, Ashutosh and Gupta, Abhishek and Langbort, Cedric and Ba{\c{s}}ar, Tamer},
  journal={IEEE Transactions on Automatic Control},
  volume={59},
  number={3},
  pages={555--570},
  year={2013},
  publisher={IEEE}
}

@article{subramanian2022approximate,
  title={Approximate information state for approximate planning and reinforcement learning in partially observed systems},
  author={Subramanian, Jayakumar and Sinha, Amit and Seraj, Raihan and Mahajan, Aditya},
  journal={The Journal of Machine Learning Research},
  volume={23},
  number={1},
  pages={483--565},
  year={2022},
  publisher={JMLRORG}
}

@article{ouyang2016dynamic,
  title={Dynamic games with asymmetric information: Common information based perfect bayesian equilibria and sequential decomposition},
  author={Ouyang, Yi and Tavafoghi, Hamidreza and Teneketzis, Demosthenis},
  journal={IEEE Transactions on Automatic Control},
  volume={62},
  number={1},
  pages={222--237},
  year={2016},
  publisher={IEEE}
}

@inproceedings{mao2020information,
  title={Information state embedding in partially observable cooperative multi-agent reinforcement learning},
  author={Mao, Weichao and Zhang, Kaiqing and Miehling, Erik and Ba{\c{s}}ar, Tamer},
  booktitle={2020 59th IEEE Conference on Decision and Control (CDC)},
  pages={6124--6131},
  year={2020},
  organization={IEEE}
}

@inproceedings{dave2022decentralized,
  title={Decentralized Control of Two Agents with Nested Accessible Information},
  author={Dave, Aditya and Venkatesh, Nishanth and Malikopoulos, Andreas A},
  booktitle={2022 American Control Conference (ACC)},
  pages={3423--3430},
  year={2022},
  organization={IEEE}
}

@misc{guan2023zero,
  title={Zero-Sum Games between Mean-Field Teams: A Common Information and Reachability based Analysis},
  author={Guan, Yue and Afshari, Mohammad and Tsiotras, Panagiotis},
  year={2023},
  eprint={2303.12243},
  archivePrefix={arXiv},
  primaryClass={eess.SY}
}

@article{kara2022near,
  title={Near optimality of finite memory feedback policies in partially observed markov decision processes},
  author={Kara, Ali and Yuksel, Serdar},
  journal={Journal of Machine Learning Research},
  volume={23},
  number={11},
  pages={1--46},
  year={2022}
}

@inproceedings{kao2022common,
  title={Common information based approximate state representations in multi-agent reinforcement learning},
  author={Kao, Hsu and Subramanian, Vijay},
  booktitle={International Conference on Artificial Intelligence and Statistics},
  pages={6947--6967},
  year={2022},
  organization={PMLR}
}

@inproceedings{tang2023novel,
  title={A novel point-based algorithm for multi-agent control using the common information approach},
  author={Tang, Dengwang and Nayyar, Ashutosh and Jain, Rahul},
  booktitle={2023 62nd IEEE Conference on Decision and Control (CDC)},
  pages={1432--1437},
  year={2023},
  organization={IEEE}
}

\listoffixmes

\newpage
\appendix

\section{Summary of Notation}\label{sec:summary_notation}

\begingroup
\setlength{\tabcolsep}{6pt} 
\renewcommand{\arraystretch}{1.2}
\begin{table}[!ht]
\centering
\begin{tabular}{p{0.1\linewidth} p{0.9\linewidth}}
\multicolumn{2}{l}{\textbf{Generic Sequential Decision-Making  Problems}} \\
    \hline
    $\bbX_t$ & Space that the variable $X_t$ lies in within the stochastic process $(X_1, \ldots, X_H)$. \\
    $\calO$ & $\calO \subset [H]$ denotes the set of observations among the variables $(X_1, \ldots, X_H)$. \\
    $\calA$ & $\calA \subset [H]$ denotes the set of actions among the variables $(X_1, \ldots, X_H)$. \\
    $\bbH_h$ & The space of histories at time $h$. $\bbH_h \coloneq \prod_{s = 1}^h \bbX_s$.\\
    $\bbF_h$ & The space of futures at time $h$. $\bbF_h \coloneq \prod_{s = h+1}^H \bbX_s$.\\
    $\mathrm{obs}(\cdot)$ & The observation component of a trajectory. $\mathrm{obs}(x_i, \ldots, x_j) \coloneq (x_s: s \in \calO_{i:j})$.\\
    $\mathrm{act}(\cdot)$ & The action component of a trajectory. $\mathrm{act}(x_i, \ldots, x_j) \coloneq (x_s: s \in \calA_{i:j})$.\\
    $\bbH_h^{\{o,a\}}$ & The space of observation (resp., action) histories. E.g., $\bbH_h^o \coloneq \prod_{s \in \calO_{1:h}} \bbX_s = \mathrm{obs}(\bbH_h)$.\\
    $\bbF_h^{\{o,a\}}$ & The space of observation (resp., action) histories. E.g., $\bbF_h^o \coloneq \prod_{s \in \calO_{h+1:H}} \bbX_s = \mathrm{obs}(\bbF_h)$.\\
    $\probbar{\tau_h}$ & The probability of a trajectory given actions are executed. $\probbar{\tau_h} \coloneq \prob{\obs(\tau_h) \ggiven \mathrm{do}(\act(\tau_h))}$.\\
    $\bm{D}_h$ & Dynamics matrix at time $h$. $\bm{D}_h \in \reals^{\aabs{\bbH_h} \times \aabs{\bbF_h}}, \bbra{\bm{D}_h}_{\tau_h, \omega_h} \coloneq \probbar{\tau_h, \omega_h}$. $r_h \coloneq \mathrm{rank}(\bm{D}_h)$. \\
    $\pi(\tau_h)$ & For $\tau_h = (x_1, \ldots, x_h)$ and policy $\pi$, $\pi(\tau_h) \coloneq \prod_{s \in \calA_{1:h}} \pi(x_s \ggiven x_1, \ldots, x_{s-1})$.\\
    $\pi(\omega_h \ggiven \tau_h)$ & For $\tau_h = (x_1, \ldots, x_h)$, $\omega_h = (x_{h+1}, \ldots, x_{h'})$, $\pi(\omega_h \ggiven \tau_h) = \prod_{s \in \calA_{h+1:h'}} \pi(x_s \ggiven x_1, \ldots, x_{s-1})$.\\
    \hline
\multicolumn{2}{l}{\textbf{POSTs and POSGs}}\\
    \hline
    $\bbX_t$ & Space that the variable $X_t$ lies in within the stochastic process $(X_1, \ldots, X_T)$. \\
    $\calS$ & $\calS \subset [T]$ denotes the set of system variables among the variables $(X_1, \ldots, X_T)$. \\
    $\calA$ & $\calA \subset [T]$ denotes the set of action variables among the variables $(X_1, \ldots, X_T)$. \\
    $\calO$ & $\calO \subset \calS$ denotes the subset of system variables which are observable. \\
    $\calU$ & The union of observable system variables and action variables. $\calU \coloneq \calO \cup \calA$. Let $H \coloneq \aabs{\calU}$.\\
    $t(h)$ & For $h \in [H]$ indexing the order among observables, $t(h) \in \calU$ denotes the order among all variables.\\
    $\calI_t$ & The information set of the $t$-th variable. \\
    $\bbI_t$ & $\bbI_t \coloneq \prod_{s \in \calI_t} \bbX_s$ denotes the information space at time $t$. \\
    $\calI_h^\dagger$ & The minimal $d$-separating set at the $h$-th observable. See~\Cref{def:I_dagger}.\\
    $\bbI_h^\dagger$ & $\bbI_h^\dagger \coloneq \prod_{s \in \calI_h^\dagger} \bbX_s$ denotes the ``information-structural state''.\\
    \hline
\multicolumn{2}{l}{\textbf{Generalized PSRs}} \\
\hline
    $\bbQ_h$ & Core test set at time $h$. Let $d_h \coloneq \aabs{\bbQ_h}$ and $d = \max_h d_h$.\\
    $\bbQ_h^A$ & Action component of core test set at time $h$. $\bbQ_h^A = \act(\bbQ_h)$.\\
    $Q_A$ & Maximum size of the action component of core test sets. $Q_A \coloneq \max_h \aabs{\bbQ_h^A}$. \\
    $M_h$ & Observable operators of PSR representation mapping $M_h: \bbX_h \to \reals^{d_{h} \times d_{h-1}}$. \\
    $\psi_h$ & Prediction features. $\psi_h(\tau_h) \coloneq \paren{\probbar{\tau_h, q}}_{q \in  \bbQ_h}$. In PSR, $\psi_h(x_1, \ldots, x_h) = M_h(x_h) \cdots M_1(x_1) \psi_0$. \\
    $m_h$ & Prediction coefficients. $m_h(\omega_h)^\top \coloneq \phi_H(x_H)^\top M_{H-1}(x_{H-1}) \cdots M_{h+1}(x_{h+1})$.\\
    $\psi_h$ & Normalized prediction features. $\bar{\psi}_h(\tau_h) \coloneq \psi_h(\tau_h) / \probbar{\tau_h} = \paren{\probbar{q \ggiven \tau_h}}_{q \in \bbQ_h}$. \\
    \hline
\multicolumn{2}{l}{\textbf{General Mathematical Notation}} \\
    \hline
    $\tv{p, q}$ & Total variation distance. $\tv{p,q} \coloneq \sum_{x \in \bbX} \aabs{p(x) - q(x)}$.\\
    $\hellingersq{p, q}$ & Hellinger squared distance. $\hellingersq{p, q} \coloneq \frac{1}{2} \sum_{x \in \bbX} \pparen{\sqrt{p(x)} - \sqrt{q(x)}}^2$.\\
    $\sigma_k(A)$ & $k$-th largest eigenvalue of the matrix $A$. \\
    $\norm{A}_p$ & The matrix $p$-norm. $\norm{A}_p \coloneq \max_{\norm{x}_p = 1} \norm{A x}_p$.\\
    $\norm{x}_A$ & The vector norm induced by the positive semi-definite matrix $A$. $\norm{x}_A \coloneq \sqrt{x^\top A x}$.\\
    $A^\dagger$ & Moore-Penrose pseudoinverse.\\
    $\calP(\bbA)$ & Space of probability distributions over $\bbA$. $\calP(\bbB | \bbA)$ is the space of kernels from $\bbA$ to $\bbB$.\\
    $\calN_{i:j}$ & For an index set $\calN \subset [H]$, $\calN_{i:j} = \calN \cap \sset{i, \ldots, j}$.\\
    \hline
\end{tabular}

\end{table}
\endgroup

\section{Existence of Generalized PSR representations and their covering number}\label{sec:appdx_genpsr_exists}


In this section we show that any rank-$r$ sequential decision-making problem (as per~\Cref{ssec:seq_dec_making}) can be represented via a rank-$r$ generalized PSR (\Cref{def:gen_psr}). Next, we bound the covering number of the class of rank $r$ PSRs, which will be important for our MLE analysis. Similar results have been established in previous work for sequential decision-making problems with alternating observations and actions~\parencite[e.g.,][]{liuOptimisticMLEGeneric2022}. Recall that our formulation of the generic sequential decision-making problem and generalized PSRs is more general than the standard formulation since it allows for an arbitrary sequence of variables. Here, we follow a similar procedure to prove a slightly generalized result.

\begin{proposition}[Existence of Generalized PSR representation]\label{prop:svd_oom_rep}
    Consider a sequential decision-making problem with $\mathrm{rank}(\bm{D}_h) = r_h,\, h \in 0:H-1$. There exists a generalized PSR representation (i.e, observable operator model) $b_0, \set{B_h(x_h)}_{h \in [H], x_h \in \bbX_h}, \set{v_h}_{h \in 0:H}$ such that,
    \begin{enumerate}
        \item $B_h(x_h) \in \reals^{r_h \times r_{h-1}}$ and $\ttwonorm{B_h(x_h)} \leq 1$ for any $x_h$.
        \item $\abs{b_0} \leq \sqrt{\abs{\bbH_H^a}}$.
        \item $\ttwonorm{v_h} \leq \sqrt{\abs{\bbF_{h}^o} / \abs{\bbF_{h}^a}}$.
        \item For any $h$, $\frac{1}{\abs{\bbX_h}^{\bm{1}\{h \in \calA\}}} v_h^\top \sum_{x_h \in \bbX_h} B_h(x_h) = v_{h-1}^\top$.
        \item For any $\tau_h \in \bbH_h$, $\probbar{\tau_h} = v_h^\top B_h(x_h) \cdots B_1(x_1) b_0$.
    \end{enumerate}
\end{proposition}

\begin{proof}
    We construct the representation via the singular value decomposition of the matrix $\bm{D}_h^\top$. Let $U_h \in \reals^{\abs{\bbF_h} \times r_h}, \Sigma_h \in \reals^{r_h \times r_h}, V_h^\top \in \reals^{r_h \times \abs{\bbH_h}}$ be the SVD such that $\bm{D}_h^\top = U_h \Sigma_h V_h^\top$. Define $b_0, B_h, v_h^\top$ as follows,
    \begin{equation*}
        b_0 = \twonorm{\bm{D}_0}, \ \ B_h(x_h) = U_h^\top \bra{U_{h-1}}_{(x_h, \bbF_h), :}, \ \ v_h^\top = \frac{1}{\abs{\bbF_h^a}} \bm{1}^\top U_h.
    \end{equation*}
    Here, $\bra{U_{h-1}}_{(x_h, \Omega_h), :}$ denotes an $\abs{\bbF_h}$ by $r_{h-1}$ submatrix of $U_{h-1}$ consisting of the rows $(x_h, \omega_h), \, \omega_h \in \bbF_h$ (i.e., the set of futures where the variable at time $h$ is $x_h$). Note that $\abs{\bbF_H^a} = 1$ by convention, a product over an empty set. We verify each property in turn.

    First, $\ttwonorm{B_h(x_h)} = \ttwonorm{U_h^\top \bbra{U_{h-1}}_{(x_h, \bbF_h), :}} \leq 1$ since $U_{h}, U_{h-1}$ are unitary matrices. Second, 
    \begin{align*}
        \abs{b_0} &= \twonorm{\bm{D}_0} = \sqrt{\sum_{\tau_H} \probbar{\tau_H}^2} \\
        &\leq \sqrt{\sum_{\tau_H} \probbar{\tau_H}} = \sqrt{\sum_{\tau_H^a} \sum_{\tau_H^o} \prob{\tau_H^o \given \tau_H^a}} = \sqrt{\sum_{\tau_H^a} 1} = \sqrt{\prod_{s \in \calA} \abs{\bbX_s}},
    \end{align*}
    where the inequality is since $\probbar{\tau_H} \in [0,1]$. For property 3, we have
    \begin{align*}
        \twonorm{v_h} &= \frac{1}{\abs{\bbF_h^a}} \twonorm{\bm{1}^\top U_h}\\
        &\leq \frac{1}{\abs{\bbF_h^a}} \twonorm{\bm{1}} = \frac{\sqrt{\abs{\bbF_h}}}{\abs{\bbF_h^a}} = \sqrt{\abs{\bbF_{h}^o} / \abs{\bbF_{h}^a}},
    \end{align*}
    where the inequality is since $U_h$ is unitary, and the final equality is since $\abs{\bbF_h} = \abs{\bbF_h^o} \abs{\bbF_h^a}$.

    Next, to prove properties 4 and 5, we first show the following claim.
    \begin{claim*}
        For any history $\tau_h = (x_1, \ldots, x_h) \in \bbH_h$, $h \in 0:H$, we have $B_h(x_h) \cdots B_1(x_1) b_0 = U_h^\top \bra{\bm{D}_h^\top}_{:, \tau_h}$.
    \end{claim*}
    \begin{proof}[Proof of claim.]
        We prove the claim by induction. In the base case, $h = 0$, $\bm{D}_0^\top$ is a vector in $\reals^{\bbF_0}$ (note that $\bbF_0 = \bbH_H$). Hence, $U_0$ is simply the normalized vector $U_0 = \bm{D}_0^\top / \ttwonorm{\bm{D}_0^\top}$, and hence $U_0^\top \bm{D}_0^\top = \bm{D}_0 \bm{D}_0^\top / \ttwonorm{\bm{D}_0} = \ttwonorm{\bm{D}_0} = b_0$. Proceeding by induction, suppose the claim holds for $h-1$. Then, we have,
        \begin{align*}
            B_h(x_h) \cdots B_1(x_1) b_0 &= B_h(x_h) U_{h-1}^\top \bra{\bm{D}_{h-1}^\top}_{:, \tau_{h-1}} \\
            &= U_h^\top \bra{U_{h-1}}_{(x_h, \bbF_h), :} U_{h-1}^\top \bra{\bm{D}_{h-1}^\top}_{:, \tau_{h-1}} \\
            &= U_h^\top \bra{U_{h-1} U_{h-1}^\top \bm{D}_{h-1}^\top}_{(x_h, \bbF_h), \tau_{h-1}} \\
            &= U_h^\top \bra{\bm{D}_{h-1}^\top}_{(x_h, \bbF_h), \tau_{h-1}} \\
            &= U_h^\top \bra{\bm{D}_{h}^\top}_{:, \tau_{h}},
        \end{align*}
        where the final equality is because $\bra{\bm{D}_{h-1}^\top}_{(x_h, \omega_h), \tau_{h-1}} = \probbar{\tau_{h-1}, x_h, \omega_h} = \probbar{\tau_h, \omega_h} = \bra{\bm{D}_{h}^\top}_{\omega_h, \tau_{h}}$.
    \end{proof}
    Using this fact, we can now show property 5 as follows,
    \begin{align*}
        v_h^\top B_h(x_h) \cdots B_1(x_1) b_0 &= \frac{1}{\abs{\bbF_h^a}} \bm{1}^\top U_h U_h^\top \bra{\bm{D}_h^\top}_{:, \tau_h} = \frac{1}{\abs{\bbF_h^a}} \bm{1}^\top \bra{\bm{D}_h^\top}_{:, \tau_h}\\
        &= \frac{1}{\abs{\bbF_h^a}} \sum_{\omega_h \in \bbF_h} \probbar{\tau_h, \omega_h} = \frac{1}{\abs{\bbF_h^a}} \sum_{\omega_h^a \in \bbF_h^a} \sum_{\omega_h^o \in \bbF_h^o} \prob{\tau_h^o, \omega_h^o \given \tau_h^a, \omega_h^a} \\
        &= \frac{1}{\abs{\bbF_h^a}} \prob{\tau_h^o \given \tau_h^a} \sum_{\omega_h^a \in \bbF_h^a} \sum_{\omega_h^o \in \bbF_h^o} \prob{\omega_h^o \given \omega_h^a, \tau_h^a, \tau_h^o} \\
        &= \frac{1}{\abs{\bbF_h^a}} \prob{\tau_h^o \given \tau_h^a} \sum_{\omega_h^a \in \bbF_h^a} 1 \\
        &= \prob{\tau_h^o \given \tau_h^a}.
    \end{align*}

    Finally, it remains to show property 4. Consider the linear equation $x^\top U_h^\top \bm{D}_h^\top = {\aabs{\bbF_h^a}}^{-1} \bm{1}^\top \bm{D}_h^\top$. Note that $ U_h^\top \bm{D}_h^\top \in \reals^{r_h \times \aabs{\bbH_h}}$ is rank $r_h$. Thus, this equation has a unique solution. Our strategy is to show that $v_h^\top$ and $v_{h+1}^\top \sum_{x_{h+1}} B_{h+1}(x_h)$ are both solutions to this linear equation, and hence $v_h^\top = v_{h+1}^\top \sum_{x_{h+1}} B_{h+1}(x_h)$. That $v_h^\top$ is a solution is clear by definition of $v_h$, $v_h^\top U_h^\top \bm{D}_h^\top = {\aabs{\bbF_h^a}}^{-1} \bm{1}^\top U_h U_h^\top \bm{D}_h^\top = {\aabs{\bbF_h^a}}^{-1} \bm{1}^\top \bm{D}_h^\top$. First, recall by the calculation above that ${\aabs{\bbF_h^a}}^{-1} \bm{1}^\top \bm{D}_h^\top$ is a vector in $\reals^{\bbH_h}$ where the $\tau_h$-th entry is $\prob{\tau_h^o \given \tau_h^a}$. We will calculate the $\tau_h$-th entry of the vector $x^\top U_h^\top \bm{D}_h$ when $x^\top = v_{h+1}^\top \sum_{x_{h+1}} B_{h+1}(x_{h+1})$,
    \begin{align*}
        \paren{v_{h+1}^\top \sum_{x_{h+1}} B_{h+1}(x_{h+1})} \bra{U_h^\top D_h^\top}_{:, \tau_h} &= \frac{1}{\abs{\bbF_{h+1}^a}} \sum_{x_{h+1}} \bm{1}^\top U_{h+1} U_{h+1}^\top \bra{U_{h}}_{(x_{h+1}, \bbF_{h+1}), :} \bra{U_h^\top D_h^\top}_{:, \tau_h} \\
        &= \frac{1}{\abs{\bbF_{h+1}^a}} \sum_{x_{h+1}}  \bm{1}^\top U_{h+1} U_{h+1}^\top \bra{U_{h} U_h^\top D_h^\top}_{(x_{h+1}, \bbF_{h+1}), \tau_h} \\
        &= \frac{1}{\abs{\bbF_{h+1}^a}} \sum_{x_{h+1}} \bra{\bm{1}^\top D_h^\top}_{(x_{h+1}, \bbF_{h+1}), \tau_h} \\
        &= \frac{1}{\abs{\bbF_{h+1}^a}} \sum_{x_{h+1}} \sum_{\omega_{h+1}} \probbar{\tau_h, x_{h+1}, \omega_{h+1}} = \frac{1}{\abs{\bbF_{h+1}^a}} \sum_{\omega_{h} \in \bbF_h} \probbar{\tau_h, \omega_h} \\
        &= \frac{1}{\abs{\bbF_{h+1}^a}} \prob{\tau_h^o \given \tau_h^a} \sum_{\omega_{h}^a \in \bbF_h^a} \sum_{\tau_h^o \in \bbF_h^o} \prob{\omega_h^o \given \tau_h, \omega_h^a} \\
        &= \frac{1}{\abs{\bbF_{h+1}^a}} \prob{\tau_h^o \given \tau_h^a} \sum_{\omega_{h}^a \in \bbF_h^a} 1 = \frac{1}{\abs{\bbF_{h+1}^a}} \abs{\bbF_{h}^a} \prob{\tau_h^o \given \tau_h^a} \\
        &= \frac{1}{\abs{\bbX_{h+1}}^{\bm{1}\{h+1 \in \calA\}}} \prob{\tau_h^o \given \tau_h^a},
    \end{align*}
    where the final inequality is since $\abs{\bbF_h^a} = \prod_{s \in h+1:H} (\abs{\bbX_s}^{\bm{1}\{s \in \calA\}})$.
\end{proof}

\begin{corollary}\label{cor:rankr_oom_rep}
    Consider a sequential decision-making problem with $\mathrm{rank}(D_h) \leq r$. Then, there exists a generalized PSR $b_0 \in \reals^r$, $\set{B_h(x_h)}_{h \in [H], x_h \in \bbX_h} \subset \reals^{r \times r}$, $v_H \in \reals^r$ such that,
    \begin{enumerate}
        \item $\twonorm{B_h(x_h)} \leq 1,\, \forall h, x_h \in \bbX_h$, $\twonorm{b_0} \leq \sqrt{\abs{\bbH_H^a}}$, and $\twonorm{v_H} \leq 1$.
        \item For any $\tau_H \in \bbH_H$, $\probbar{\tau_H} = v_H^\top B_H(x_H) \cdots B_1(x_1) b_0$.
    \end{enumerate}
\end{corollary}
\begin{proof}
    In~\Cref{prop:svd_oom_rep} we constructed such a representation with dimensions in terms of $r_h$ instead of $r$. Since $r_h \leq r$, we can pad this representation with dummy columns and/or rows filled with zeros to obtain a representation with dimensions in terms of $r$.
\end{proof}

An important part of maximum likelihood analysis is the notion of a ``bracketing number'' which controls the complexity of the model class $\Theta$~\parencite[e.g.,][]{geerRatesConvergenceMaximum2006}. In our analysis, the model class is the set of generalized PSRs of a given rank. As shown in the results above, rank-$r$ generalized PSRs can represent any rank-$r$ sequential decision-making problem, with operators whose norm is bounded. In the next result, we will consider a closely related notion to the bracketing number which crucially incorporates optimism. $\bar{\Theta}_\varepsilon$ is said to be an ``optimistic $\varepsilon$-cover'' for $\Theta$ if for each $\theta \in \Theta$, there exists $\hat{\theta} \in \bar{\Theta}_\varepsilon$ with an associated probability measure $\bar{\bbP}_{\hat{\theta}}^{\varepsilon}$ such that,
\begin{align*}
    \forall h, \tau_h,\ \bar{\bbP}_{\hat{\theta}}^{\varepsilon}(\tau_h) \geq \probbarunder{\theta}{\tau_h}, \\
    \forall h, \tau_h,\ \sum_{\tau_h} \abs{\bar{\bbP}_{\hat{\theta}}^{\varepsilon}(\tau_h) - \probbarunder{\theta}{\tau_h}} \leq \varepsilon.
\end{align*}
The first condition ensures optimism and the second condition ensures that $\bar{\Theta}_\varepsilon$ $\varepsilon$-covers $\Theta$, in the sense that the probability of any trajectory is approximated within an error $\varepsilon$. Recall that the parameter $\beta$ in~\Cref{alg:post_ucb,alg:posg_ucb}, which appears in the sample complexity results in~\Cref{theorem:post_ucb,theorem:posg_ucb}, is defined in terms of $\aabs{\bar{\Theta}_{\varepsilon}}$. The next proposition bounds the size of $\abs{\bar{\Theta}_\varepsilon}$.

\begin{proposition}[Optimistic cover of sequential decision making problems]\label{prop:bracketing_num_opt_net}
    Let $\mathfrak{M}$ be the set of all rank-$r$ sequential decision-making problems with a horizon of length $H$, observation index set $\calO \subset [H]$, action index set $\calA \subset [H]$, and variable spaces $\bbX_1, \ldots, \bbX_H$.
    Then, there exists an optimistic $\varepsilon$-cover $\bar{\Theta}_{\varepsilon}$ of $\Theta$ with cardinality bounded by,
    \begin{equation*}
        \log \abs{\bar{\Theta}_\varepsilon} \leq O\paren{r^2 \max_{h} \abs{\bbX_h} H^2 \log\paren{\frac{\max_h \abs{\bbX_h}}{\epsilon}}}.
    \end{equation*}
\end{proposition}
\begin{proof}
    Define the set of generalized PSR representations constructed in~\Cref{cor:rankr_oom_rep},
    \begin{equation*}
        \begin{split}
            \Theta \coloneqq \Biggl\{ b_0 \in \reals^r, \set{B_h(x_h)}_{h, x_h}, v_H \in \reals^r \,&\colon\, \twonorm{B_h(x_h)} \leq 1,\, \forall h, x_h,\, \twonorm{b_0} \leq \sqrt{\abs{\bbH_H^a}},\, \twonorm{v_H} \leq 1,\\
            & \text{and} \ \forall\, \tau_H \in \bbH_H,\, \probbarunder{m}{\tau_H} = v_H^\top B_H(x_H) \cdots B_1(x_1) b_0, \\
            &\text{where $m$ is a sequential decision making problem in} \ \mathfrak{M} \Biggr\}.
        \end{split}
    \end{equation*}

    Let $\calC_\delta$ be a $\delta$-cover of the above set with respect to the $\ell_\infty$-norm. For $\hat{\theta} = (b_0, \set{B_h(x_h)}, v_H) \in \calC_\delta$, define the $\varepsilon$-optimistic probabilities as,
    \begin{equation*}
        \bar{\bbP}_{\hat{\theta}}^{\varepsilon}(\tau_H) \coloneqq v_H^\top B_H(x_h) \cdots B_1(x_1) b_0 + \varepsilon/2
    \end{equation*}

    We will show that for an appropriate choice of $\delta$, $\calC_\delta$ is an optimistic $\varepsilon$-cover. In particular, for each $\theta \in \Theta$, there exists $\hat{\theta} \in \calC_{\delta}$ such that,
    \begin{align*}
        \forall h, \tau_h,\ \bar{\bbP}_{\hat{\theta}}^{\varepsilon}(\tau_h) \geq \probbarunder{\theta}{\tau_h}, \\
        \forall h, \tau_h,\ \sum_{\tau_h} \abs{\bar{\bbP}_{\hat{\theta}}^{\varepsilon}(\tau_h) - \probbarunder{\theta}{\tau_h}} \leq \varepsilon.
    \end{align*}

    To choose the value of $\delta$ for which the above holds, observe that
    \begin{align*}
        &\sum_{\tau_H} \abs{\hat{v}_H^\top \hat{B}_H(x_H) \cdots B_{1}(x_{1}) \hat{b}_0 - v_H^\top B_H(x_H) \cdots B_{1}(x_{1})b_0} \\
        &\leq \sum_{h=1}^H \sum_{\tau_H} \abs{\hat{v}_H^\top \hat{B}_H(x_H) \cdots \hat{B}_{h+1}(x_{h+1}) (\hat{B}_h(x_h) - B_h(x_h)) B_{h-1}(x_{h-1}) \cdots B_{1}(x_1) b_0}\\
        &\quad + \sum_{\tau_H} \abs{\hat{v}_H^\top B_H(x_H) \cdots B_{1}(x_{1})  (\hat{b}_0 - b_0)} \\
        &\leq \sum_h \sum_{\tau_H} r \norm{\hat{B}_h(x_h) - B_h(x_h)}_{\max} \sqrt{\abs{\bbH_H^a}} + \sum_{\tau_H} \sqrt{r} \norm{\hat{b}_0 - b_0}_{\infty}\\
        &\leq H \max_{h} \abs{\bbX_h}^{H + \abs{\calA}/2} r \delta + \max_{h}\abs{\bbX_h}^H \sqrt{r} \delta,
    \end{align*}
    where the second inequality uses $\ttwonorm{\hat{v}_H} = \ttwonorm{v_H} = \ttwonorm{B_h(x_h)} = 1$, $\ttwonorm{\hat{B}_h(x_h) - B_h(x_h)} \leq r \nnorm{\hat{B}_h(x_h) - B_h(x_h)}_{\max} \leq r \delta$, $\ttwonorm{b_0} \leq \sqrt{\aabs{\bbH_H^a}}$, and $\ttwonorm{\hat{b}_0 - b_0} \leq \sqrt{r} \iinfnorm{\hat{b}_0 - b_0} \leq \sqrt{r} \delta$. Hence, choosing $\delta \coloneq \varepsilon \cdot \max_h \abs{\bbX_h}^{-cH}$ for $c$ an absolute constant large enough achieves a $\varepsilon$-optimistic covering of $\Theta$. Hence, we let $\bar{\Theta}_\varepsilon = \calC_{\delta}$, with $\delta = \varepsilon \cdot \max_h \cdot \abs{\bbX_h}^{-cH}$. It remains to bound the size of $\aabs{\bar{\Theta}_\varepsilon}$.

    Recall that $\infnorm{\cdot} \leq \ttwonorm{\cdot}$ and that an interval $[-x, x]$ in $\reals$ admits a $\delta$-cover of size bounded by $2 x / \delta$. Now, observe that $\max_{ij} \abs{[B_h(x_h)]_{ij}} \leq \ttwonorm{B_h(x_h)} \leq 1$. Hence, for a fixed $h$, $\set{B_h(x_h)}_{x_h}$ admits a cover of size bounded by $(2/\delta)^{r^2 \abs{\bbX_h}}$. Considering all $h$, the cover is bounded by $(2 / \delta)^{r^2 \sum_h \abs{\bbX_h}} \leq (2 / \delta)^{r^2 \max_h \abs{\bbX_h} H}$. For, $b_0$, we have $\infnorm{b_0} \leq \ttwonorm{b_0} \leq \sqrt{\abs{\bbH_H^a}}$, hence the covering number is bounded by $(2 \sqrt{\abs{\bbH_H^a}}/\delta)^r$. Finally for $v_H$, we have $\infnorm{v_H} \leq \ttwonorm{v_H} \leq 1$, hence the covering number is bounded by $(2 / \delta)^r$. Thus, we have,
    \begin{equation*}
        \log \abs{\bar{\Theta}_\varepsilon} \leq O\paren{r^2 \max_{h} \abs{\bbX_h} H \log\paren{\frac{1}{\delta}}}.
    \end{equation*}
    Recalling that $\delta = \varepsilon \max_{s} \abs{\bbX_s}^{-cH}$, we obtain that,
    \begin{equation*}
        \log \abs{\bar{\Theta}_\varepsilon} \leq O\paren{r^2 \max_{h} \abs{\bbX_h} H^2 \log\paren{\frac{\max_h \abs{\bbX_h}}{\epsilon}}}.
    \end{equation*}
\end{proof}

\section{Proofs of Section~\ref{ssec:info_struct_rank}}\label{sec:appdx_seq_team_rank_proof}

\begin{theorem*}[Restatement of~\Cref{theorem:post_posg_rank}]
    The rank of the observable system dynamics of a POST or POSG is bounded by
    \[r \leq \max_{h \in [H]} \Big\lvert \bbI_h^\dagger \Big\rvert .\]
\end{theorem*}
\begin{proof}
    We have
    \begin{equation*}
        \begin{split}
            \bra{\bm{D}_h}_{\tau_h, \omega_h} &=  \prob{\tau_h^o, \omega_h^o \given \mathrm{do}(\tau_h^a, \tau_h^a)} \\
            &= \prob{\tau_h^o \given \mathrm{do}\paren{\tau_h^a}} \prob{\omega_h^o \given \tau_h^o;\,\mathrm{do}\paren{\tau_h^a, \omega_h^a}} \\
            &\stepa{=} \prob{\tau_h^o \given \mathrm{do}\paren{\tau_h^a}} \sum_{\substack{x_k \in \bbX_k \\ k \in \calI_h^\dagger}} \prob{\set{x_k,\, k \in \calI_h^\dagger} \given \tau_{h}^o;\,\mathrm{do}\paren{\tau_h^a, \omega_h^a}} \prob{\omega_h^o \given \set{x_k,\, k \in \calI_h^\dagger}, \tau_h^o;\,\mathrm{do}\paren{\tau_h^a, \omega_h^a}} \\
            &\stepb{=} \sum_{\substack{x_k \in \bbX_k \\ k \in \calI_h^\dagger}} \prob{\tau_h^o \given \mathrm{do}\paren{\tau_h^a}} \prob{\set{x_k,\, k \in \calI_h^\dagger} \given \tau_{h}^o;\,\mathrm{do}\paren{\tau_h^a}} \prob{\omega_h^o \given \set{x_k,\, k \in \calI_h^\dagger}, \tau_h^o;\,\mathrm{do}\paren{\tau_h^a, \omega_h^a}} \\
            &\stepc{=} \sum_{\substack{x_k \in \bbX_k \\ k \in \calI_h^\dagger}} \prob{\tau_h^o \given \mathrm{do}\paren{\tau_h^a}} \prob{\set{x_k,\, k \in \calI_h^\dagger} \given \tau_{h}^o;\,\mathrm{do}\paren{\tau_h^a}} \prob{\omega_h^o \given \set{x_k,\, k \in \calI_h^\dagger};\,\mathrm{do}\paren{\omega_h^a}},
        \end{split}
    \end{equation*}
    where step (a) is simply the law of total probability, step (b) is that $\{x_k, k \in \calI_h^\dagger\}$ is conditionally independent of $\mathrm{do}(\omega_h^a)$ (future actions) given $(\tau_{h}^o;\,\mathrm{do}(\tau_h^a))$ (the past), and step (c) is that $\omega_h^o$ is conditionally independent of $(\tau_h^o;\,\mathrm{do}(\tau_h^a))$ given $\{x_k,\, k \in \calI_h^\dagger\}$. 
    This is due to a result by~\citet{vermaCausalNetworksSemantics1990} which states: for three sets of variables $A, B, C$ in a directed graphical model, if $A$ and $B$ are $d$-separated by $C$, then $A \perp B \,|\, C$. 
    Recall that $\calI_h^\dagger$ is defined as the minimal set which $d$-separates $(X_{t(1)}, \ldots, X_{t(h)})$ from $(X_{t(h+1)}, \dots, X_{t(H)})$. 

    As a technical remark, note that $i_h^\dagger = \pparen{x_k, k \in \calI_h^\dagger}$ may include actions and hence,
    \begin{equation*}
        \prob{\set{x_k,\, k \in \calI_h^\dagger} \given \tau_{h}^o;\,\mathrm{do}\paren{\tau_h^a}} = \prob{\set{x_k,\, k \in \calI_h^\dagger \cap \calS} \given \tau_{h}^o;\,\mathrm{do}\paren{\tau_h^a}} \bm{1}\set{\paren{x_k,\, k \in \calI_h^\dagger \cap \calA} \ \text{matches} \ \tau_h^a},
    \end{equation*}
    since the action  components of $i_h^\dagger$ are contained in the history $\tau_h$.

    Now define two matrices
    \begin{equation*}
        \begin{split}
            \bm{D}_{h, 1} &\coloneqq \bra{\prob{\tau_h^o \given \mathrm{do}\paren{\tau_h^a}} \prob{\set{x_k,\, k \in \calI_h^\dagger} \given \tau_{h}^o;\,\mathrm{do}\paren{\tau_h^a}}}_{\tau_h, i_h^\dagger}, \quad \tau_h \in \bbH_h,\, i_h^\dagger \equiv \paren{x_k, \, k \in \calI_h^\dagger} \in \bbI_h^\dagger, \\
            \bm{D}_{h, 2} &\coloneqq \bra{\prob{\omega_h^o \given \set{x_k,\, k \in \calI_h^\dagger};\,\mathrm{do}\paren{\omega_h^a}}}_{i_h^\dagger, \omega_h}, \quad i_h^\dagger \equiv \paren{x_k, \, k \in \calI_h^\dagger} \in \bbI_h^\dagger,\, \omega_h \in \bbF_h.
        \end{split}
    \end{equation*}

    We have that $\bm{D}_{h} = \bm{D}_{h,1} \bm{D}_{h,2}$, where both $\bm{D}_{h,1}$ and $\bm{D}_{h,2}$ have rank upper bounded by $\aabs{\bbI_h^\dagger} = \prod_{s \in \calI_h^\dagger} \aabs{\bbX_s}$. Hence, $\mathrm{rank}(\bm{D}_{h}) \leq \aabs{\bbI_h^\dagger}$, and the result follows.
\end{proof}

\section{Proofs of Section~\ref{sec:post_psrep}}\label{sec:appdx_post_psrep_proofs}

\begin{lemma*}[Restatement of~\Cref{lemma:core_test_set}]
    Suppose that the POST/POSG is $m$-step $\calI^\dagger$-weakly revealing. Then, $\bbQ_h^m$ is a core test set for all $h \in [H]$. Furthermore, we have
    \begin{equation}
        \probbar{\tau_h, \omega_h} = \iprod{m_h(\omega_h)}{\psi_h(\tau_h)}, \text{ and } \ \, \probbar{\omega_h \given \tau_h} = \iprod{m_h(\omega_h)}{\bar{\psi}_h(\tau_h)}.
    \end{equation}
\end{lemma*}

\begin{proof}
    Let $\tau_h \in \bbH_h, \, \omega_h \in \bbF_h$ be any history and future, respectively. By~\Cref{theorem:post_posg_rank}, recall that we have
    \begin{equation}\label{eq:lemma_core_test_set_recall}
        \Probbar{\omega_h \,|\, \tau_h} = \sum_{i_h^\dagger \in \bbI_h^\dagger} \probbar{\omega_h \given i_h^\dagger} \prob{i_h^\dagger \given \tau_h}.
    \end{equation}
    Recall that $i_h^\dagger$ may overlap with $\tau_h$. In particular, the action component of $i_h^\dagger$ is contained in $\tau_h$. Thus, $\Prob[i_h^\dagger \,|\, \tau_h] = \Prob[\{x_k,\, k \in \calI_h^\dagger \setminus \calU_{1:h}\} \,|\, \tau_h] \cdot \bm{1}\{(x_k,\, k \in \calI_h^\dagger \cap \calU_{1:h}) \ \text{matches} \ \tau_h\}$. Note that $\calI_h^\dagger \setminus \calU_{1:h} \subset \calS$ does not contain any actions. Hence, the summation over $\bbI_h^\dagger$ is equivalent to summing over its unobservable components with the restriction that its observable components match $\tau_h$.

    Define the mappings $\tilde{m}_h \colon \bbF_h \to \reals^{\abs{\bbI_h^\dagger}}$ and $p_h \colon \bbH_h \to \reals^{\abs{\bbI_h^\dagger}}$ by
    \begin{equation*}
        \tilde{m}_h(\omega_h) =  \bra{\probbar{\omega_h \given i_h^\dagger}}_{i_h^\dagger \in \bbI_h^\dagger}, \quad  p_h(\tau_h) = \bra{\prob{i_h^\dagger \given \tau_h}}_{i_h^\dagger \in \bbI_h^\dagger}.
    \end{equation*}
    Then, we have that the conditional probability of the future $\omega_h$ given the past $\tau_h$ is given by the inner product of the above mappings, $\probbar{\omega_h \given \tau_h} = \iprod{\tilde{m}_h(\omega_h)}{p_h(\tau_h)}$.
    Recall that the vector of (conditional) core test set probabilities for the history $\tau_h$ is given by $\bar{\psi}_h(\tau_h) = \bra{\prob{q^o \given \tau_h^o;\, \mathrm{do}(\tau_h^a),\, \mathrm{do}(q^a)}}_{q \in \bbQ_h^m} \in \reals^{\abs{\bbQ_h^m}}$.
    By the definition of $\bm{G}_h$ and \Cref{eq:lemma_core_test_set_recall}, we have $\bm{G}_h\, p_h(\tau_h) = \psi_h(\tau_h)$, since, for $q \in \bbQ_h^m$,
    \begin{equation*}
    \begin{split}
        \paren{\bm{G}_h\, p_h(\tau_h)}_q &= \sum_{i_h^\dagger} \paren{\bm{G}_h}_{q, i_h^\dagger} \paren{p_h(\tau_h)}_{i_h^\dagger} \\
        &= \sum_{i_h^\dagger} \prob{q^o \given i_h^\dagger; \, \mathrm{do}(q^a)} \prob{i_h^\dagger \given \tau_h}\\
        &= \prob{q^o \given \tau_h^o; \, \mathrm{do}(\tau_h^a), \mathrm{do}(q^a)} \\
        &=: \bra{\bar{\psi}_h(\tau_h)}_q
    \end{split}
    \end{equation*}

    Since by assumption $\mathrm{rank}(\bm{G}_h) = \abs{\bbI_h^\dagger}$, its pseudo-inverse $\bm{G}_h^\dagger$ is a left inverse of $\bm{G}_h$ (i.e., $\bm{G}_h^\dagger \bm{G}_h = I$). Hence, multiplying on the left by $\bm{G}_h^\dagger$, we obtain $p_h(\tau_h) = \bm{G}_h^\dagger \bar{\psi}_h(\tau_h)$.
    Hence,
    \begin{equation*}
        \begin{split}
            \probbar{\omega_h \given \tau_h} &= \iprod{\tilde{m}_h(\omega_h)}{\bm{G}_h^\dagger \bar{\psi}_h(\tau_h)} \\
            &= \iprod{\underbrace{\paren{\bm{G}_h^\dagger}^\top \tilde{m}_h(\omega_h)}_{m_h(\omega_h)}}{\bar{\psi}_h(\tau_h)}. 
        \end{split}
    \end{equation*}

    That $\probbar{\tau_h, \omega_h} = \iprod{m_h(\omega_h)}{\psi_h(\tau_h)}$ follows directly by noting the definition of $\bar{\psi}_h(\tau_h) := \psi_h(\tau_h) / \probbar{\tau_h}$.

    Hence, we have shown that for the test set $\bbQ_h^m$, the probability of each future $\omega_h$ given a history $\tau_h$ is a linear combination of the probabilities of each test in the core test set with weights $m_h(\omega_h) \coloneqq (\bm{G}_h^\dagger)^\top \tilde{m}_h(\omega_h) \in \reals^{\aabs{\bbQ_h^m}}$ depending only on the future and not the history.
\end{proof}

\begin{theorem*}[Restatement of~\Cref{theorem:constructing_gen_psr}]
    Suppose a POST/POSG is $\alpha$-robustly $m$-step $\calI^\dagger$-weakly revealing. Then, the corresponding generalized PSR as constructed in~\Cref{sec:post_psrep} is $\gamma$-well-conditioned with $\gamma = \alpha / \max_h \aabs{\bbI_h^\dagger}^{1/2}$.
\end{theorem*}
\begin{proof}
    We first show condition (1) in~\Cref{ass:psr_gamma_wellcond}. Suppose $h > H - m$ and hence the core tests are the full futures, which have length smaller than $m$. Then for any $x \in \reals^{d_{h}}$, $d_h = \prod_{s=h}^H \abs{\bbX_s}$, we have
    \begin{equation*}
        \max_\pi \sum_{\omega_h} \abs{m_h(\omega_h)^\top x} \cdot \pi(\omega_h) = \max_\pi \sum_{\omega_h} \abs{x[\omega_h]} \pi(\omega_h) \leq \onenorm{x},
    \end{equation*}
    where $x[\omega_h]$ indexes the component of the vector $x$ corresponding to the future $\omega_h$.

    Now suppose $h \leq H - m$ (and hence the core tests consist of $m$-step futures). Then, we have,
    \begin{align*}
        \max_\pi \sum_{\omega_h} \abs{m_h(\omega_h)^\top x} \pi(\omega_h) &= \max_\pi \sum_{\omega_h} \abs{m(\omega_h)^\top \bm{G}_h \bm{G}_h^\dagger x} \cdot \pi(\omega_h) \\
        &\leq \max_\pi \sum_{\omega_h} \sum_{i^\dagger \in \bbI_h^\dagger} \abs{m(\omega_h)^\top \bm{G}_h \bm{e}_{i^\dagger}} \abs{\bm{e}_{i^\dagger}^\top \bm{G}_h^\dagger x} \cdot \pi(\omega_h).
    \end{align*}
    Now observe that for any policy $\pi$ and any $i^\dagger \in \bbI_h^\dagger$, we have
    \begin{align*}
        \sum_{\omega_h} \abs{m(\omega_h)^\top \bm{G}_h \bm{e}_{i^\dagger}} \cdot \pi(\omega_h) &= \sum_{\omega_h} \abs{\tilde{m}(\omega_h)^\top \bm{G}_h^\dagger \bm{G}_h \bm{e}_{i^\dagger}} \cdot \pi(\omega_h) \\
        &= \sum_{\omega_h} \probbar{\omega_h \given i^\dagger} \pi(\omega_h) \\
        &= \sum_{\omega_h} \probunder{\pi}{\omega_h \given i^\dagger} = 1,
    \end{align*}
    where we used the definition of $m_h(\omega_h) \coloneqq \tilde{m}_h(\omega_h)^\top \bm{G}_h^\dagger$, and $\bra{\tilde{m}_h(\omega_h)}_{i^\dagger} \coloneqq \Probbar{\omega_h \,|\, i^\dagger}$. Recall that $\pi(\omega_h)$ is such that for any fixed sequence of observations $\omega_h^o$, $\sum_{\omega_h^a} \pi(\omega_h^o, \omega_h^a) = 1$.

    Putting this observation together with the preceding inequality yields
    \begin{align*}
        \max_\pi &\sum_{\omega_h} \abs{m_h(\omega_h)^\top x} \pi(\omega_h)\\
        &\leq \sum_{i^\dagger \in \bbI_h^\dagger} \abs{\bm{e}_{i^\dagger}^\top \bm{G}_h^\dagger x} \\
        &= \onenorm{\bm{G}_h^\dagger x} \leq \onenorm{\bm{G}_h^\dagger} \cdot \onenorm{x} \\
        &\leq \frac{\sqrt{\abs{\bbI_h^\dagger}}}{\alpha} \onenorm{x},
    \end{align*}
    where the final inequality is from the relation between the one-norm and two-norm $\onenorm{\bm{G}_h^\dagger} \leq \sqrt{\abs{\bbI_h^\dagger}} \twonorm{\bm{G}_h^\dagger}$, and $\twonorm{\bm{G}_h^\dagger} \leq \frac{1}{\alpha}$, by the assumption on its eigenvalues.

    Now we show condition (2) in~\Cref{ass:psr_gamma_wellcond}. For ease of notation, we denote $x_{t(h)}$ by $x_h$. When $h > H -m$, note that $\bra{M_h(x_h)}_{q_{h+1}, q_h} = \Ind{q_h = (x_{h}, q_{h+1})}$, for all $q_h \in \bbQ_h, q_{h+1} \in \bbQ_{h+1}$. Hence,we have
    \begin{equation*}
        \max_\pi \sum_{x_h} \onenorm{M_h(x_h) z} \pi(x_h) = \onenorm{z}.
    \end{equation*}
    Now, when $h \leq H - m$, by a similar line of reasoning to the proof for condition (1), we have,
    \begin{align*}
        \max_\pi \sum_{x_h} \onenorm{M_h(x_h) z} \pi(x_h | \tau_{h-1}) &\leq \max_\pi \sum_{(x_h, q_{h+1}) \in \bbX_h \times \bbQ_{h+1}} \sum_{i^\dagger \in \bbI_h^\dagger} \abs{e_{q_{h+1}}^\top M_h(x_h) \bm{G}_h e_{i^\dagger}} \cdot \abs{e_{i^\dagger} \bm{G}_h^\dagger z} \pi(x_h | \tau_{h-1}) \\
        &\stepa{=} \max_\pi \sum_{(x_h, q_{h+1}) \in \bbX_h \times \bbQ_{h+1}} \sum_{i^\dagger \in \bbI_h^\dagger} \abs{m_h(x_{t(h)}, q_{h+1}) \bm{G}_h e_{i^\dagger}} \cdot \abs{e_{i^\dagger} \bm{G}_h^\dagger z} \pi(x_h | \tau_{h-1}) \\
        &\stepb{=} \max_\pi \sum_{i^\dagger} \paren{\sum_{(x_h, q_{h+1})} \probbar{x_h, q_{h+1} \given i^\dagger} \pi(x_h | \tau_{h-1}) } \abs{e_{i^\dagger}^\top \bm{G}_h^\dagger z}
    \end{align*}
    where step (a) uses the definition of $M_h$ and step (b) uses the definition of $m_h(\omega_h)^\top \coloneqq \tilde{m}_h(\omega_h)^\top \bm{G}_h^\dagger$ and $\bra{\tilde{m}_h(\omega_h)}_{i^\dagger} \coloneqq \Probbar{\omega_h \,|\, i^\dagger}$. Now note that,
    \begin{align*}
        \sum_{(x_h, q_{h+1})} \probbar{x_h, q_{h+1} \given i^\dagger} \pi(x_h | \tau_{h-1}) &= \sum_{x_h} \sum_{\act(q_{h+1})} \sum_{\obs(q_{h+1})} \probbar{x_h, \obs(q_{h+1}) \given i^\dagger, \act(q_{h+1})} \pi(x_h) \\
        &= \sum_{\act(q_{h+1})} 1 \\
        &= \abs{\bbQ_{h+1}^A},
    \end{align*}
    where the second line is since for any fixed action sequence, the sum over the probabilities of all observation sequences is 1.

    Thus, putting this together, we obtain the following,
    \begin{align*}
        \max_\pi \sum_{x_h} \onenorm{M_h(x_h) z} \pi(x_h | \tau_{h-1}) &\leq \abs{\bbQ_{h+1}^A} \cdot \onenorm{\bm{G}_h^\dagger z}\\
        &\leq \frac{\sqrt{\abs{\bbI_h^\dagger}} \abs{\bbQ_{h+1}^A}}{\alpha} \onenorm{z},
    \end{align*}
    where the last line again follows by the assumption on the eigenvalues of $\bm{G}_h$.
\end{proof}

\section{Proof of Theorem~\ref{theorem:post_ucb}: UCB Algorithm for Generalized PSRs (Team Setting)}\label{sec:post_ucb_proof}

In this section, we prove~\Cref{theorem:post_ucb} which states that~\Cref{alg:post_ucb} returns a near-optimal policy in a polynomial number of iterations. The proof is adapted from~\parencite{huangProvablyEfficientUCBtype2023} and generalized to our setting with generalized PSRs (\Cref{def:gen_psr}). The proof is organized into several subsections. In~\Cref{ssec:proof_gen_psr_properties}, we show that the total variation distance between trajectories under the true model and the estimated model can be bounded in terms of the estimation error of the observable operators $\sset{M_h}_h$. In~\Cref{ssec:proof_mle_results} we state some general results on maximum likelihood estimation which show that the MLE model has small error on the collected dataset. In~\Cref{ssec:proof_ucb_tv} we prove that the bonus term is an upper confidence bound for the total variation distance. In~\Cref{ssec:proof_sublinear_est} we show that the estimation error is sublinear in the number of iterations (i.e., $O(\sqrt{K}))$. Finally, in~\Cref{ssec:team_thm_proof} we put this all together to prove the theorem.

\subsection{Properties of Generalized PSRs}\label{ssec:proof_gen_psr_properties}

Recall that a PSR model $\theta = \paren{\bm{M}, \psi_0, \phi_H}$ consists of operators $\bm{M} = \set{M_h}_{h=1}^{H-1}$, $M_h: \bbX_{h} \to \reals^{d_{h} \times d_{h-1}}$, $\phi_H: \bbX_H \to \reals^{d_{H-1}}$ (assumed to be the identity mapping), and $\psi_0$ (assumed to be known for the purposes of presentation). Recall that, for any trajectory $\tau_{h-1} = (x_{1}, \ldots, x_{h-1})$, under model $\theta$, we have
\begin{equation}\label{eq:Mpsi_psi_relation}
    \begin{split}
        M_h(x_h) \bar{\psi}_{h-1}(\tau_{h-1}) &= \frac{\psi_h(\tau_h)}{\bar{\bbP}_\theta\paren{\tau_{h-1}}} \\
        &= \frac{\psi_h(\tau_h)}{\bar{\bbP}_{\theta}\paren{x_h \given \tau_{h-1}} \bar{\bbP}_\theta\paren{\tau_{h-1}}} \bar{\bbP}_\theta\paren{x_h \given \tau_{h-1}} \\
        &= \bar{\psi}_h(\tau_h) \bar{\bbP}_\theta\paren{x_h \given \tau_{h-1}}
    \end{split}
\end{equation}

Here, the notation $\bar{\bbP}_\theta\paren{x_h \given \tau_{h-1}}$ means the probability of $x_h$ conditioned on the history $\tau_{h-1}$, with all actions executed. In particular, if $x_h$ is an action, then $\bar{\bbP}_\theta\paren{x_h \given \tau_{h-1}} = 1$ and $M_h(x_h) \bar{\psi}_{h-1}(\tau_{h-1}) = \bar{\psi}_h(\tau_h)$.

The following proposition shows that the total variation distance between the distribution of trajectories of two PSR models can be bounded in terms of the difference in their observable operators.

\begin{proposition}\label{prop:tv_dist_leq_est_err}
For any policy $\pi$ and $\theta, \hat{\theta} \in \Theta$, we have,
\begin{align*}
    \tv{\probb{\hat{\theta}}{\pi},\, \probb{\theta}{\pi}} &\leq \sum_{h=1}^H \sum_{\tau_H \in \bbH_H} \pi(\tau_h) \abs{\hat{m}_h(\omega_h)^\top \paren{\hat{M}_h(x_h) - M_h(x_h)} \psi_{h-1}(\tau_{h-1})},\\
    \tv{\probb{\hat{\theta}}{\pi},\, \probb{\theta}{\pi}} &\leq \sum_{h=1}^H \sum_{\tau_H \in \bbH_H} \pi(\tau_h) \abs{m_h(\omega_h)^\top \paren{\hat{M}_h(x_h) - M_h(x_h)} \hat{\psi}_{h-1}(\tau_{h-1})},\\
\end{align*}
\end{proposition}
\begin{proof}
    The probability of any trajectory $\tau_H = (x_1, \ldots, x_H)$ can be written in terms of products of the observable operators $M_h(x_h)$ of a PSR model (\Cref{eq:psr_def_M_psi}). Hence, we have,
    \begin{equation*}
        \begin{split}
            \tv{\probb{\hat{\theta}}{\pi},\, \probb{\theta}{\pi}} &= \frac{1}{2} \sum_{\tau_H} \abs{\probb{\hat{\theta}}{\pi}(\tau_H) - \probb{\theta}{\pi}(\tau_H)} \\
            &=  \frac{1}{2} \sum_{\tau_H} \pi(\tau_H) \cdot \abs{\paren{\prod_{h=1}^{H} \hat{M}_h(x_h)} \psi_0 -\paren{\prod_{h=1}^{H} M_h(x_h)} \psi_0}\\
            &\leq  \frac{1}{2} \sum_{\tau_H} \pi(\tau_H) \sum_{h=1}^H \abs{\hat{m}_{h}(x_{h+1:H})^\top \paren{\hat{M}_h(x_h) - M_h(x_h)} \psi_{h-1}(\tau_{h-1})}, \\
        \end{split}
    \end{equation*}
    where the second line follows by the triangle inequality after noting that for any trajectory $\tau_H = x_{1:H} \in \bbH_H$, the following holds for any $h=1, \ldots, H$,
    \begin{equation*}
        \paren{\prod_{h=1}^{H} \hat{M}_h(x_h)} \psi_0 -\paren{\prod_{h=1}^{H} M_h(x_h)} \psi_0 = \hat{m}_{h}(x_{h+1:H})^\top \hat{M}_{h}(x_{h}) \hat{\psi}_{h-1}(x_{1:h-1}) - m_{h}(x_{h+1:H})^\top M_{h}(x_{h}) \psi_{h-1}(x_{1:h-1}).
    \end{equation*}

    By the same argument, we obtain the second inequality,
    \begin{equation*}
        \begin{split}
            \tv{\probb{\hat{\theta}}{\pi},\, \probb{\theta}{\pi}} &= \frac{1}{2} \sum_{\tau_H} \abs{\probb{\hat{\theta}}{\pi}(\tau_H) - \probb{\theta}{\pi}(\tau_H)} \\
            &\leq  \frac{1}{2} \sum_{\tau_H} \pi(\tau_H) \sum_{h=1}^H \abs{m_{h}(x_{h+1:H})^\top \paren{\hat{M}_h(x_h) - M_h(x_h)} \hat{\psi}_{h-1}(\tau_{h-1})}.
        \end{split}
    \end{equation*}
\end{proof}
In this result, recall that we assume $\psi_0$ is known to the agent, to simplify the presentation. If $\psi_0$ was not known, there would be another term due to the estimation as $\hat{\psi}_0$ \parencite[see][Lemma C.3]{liuOptimisticMLEGeneric2022}. Note that the sample complexity of estimating $\psi_0$ is small compared to learning the other parameters.

\subsection{General Results on MLE}\label{ssec:proof_mle_results}

In this section, we state some general results on maximum likelihood estimation which ultimately guarantee that the estimated model produced by the procedure in~\Cref{alg:post_ucb} has a small estimation error. The results are stated without proof. The proofs are given in~\parencite{huangProvablyEfficientUCBtype2023} and use standard techniques on MLE analysis~\parencite{geerRatesConvergenceMaximum2006}. This ultimately leads us to a lemma which states that the estimation error of the MLE model is small on the collected data.


The first proposition states that the log-likelihood of the true model $\theta^*$ is large compared to any other model.

\begin{proposition}[Proposition 4 of~\cite{huangProvablyEfficientUCBtype2023}]
    Fix $\varepsilon < \frac{1}{KH}$. With probability at least $1 - \delta$, for any $\bar{\theta} \in \bar{\Theta}_\varepsilon$ and any $k \in [K]$, the following holds:
    \begin{align*}
        &\forall \bar{\theta} \in \bar{\Theta}_\varepsilon, \sum_h \sum_{(\tau_h, \pi) \in \calD_h} \log \probb{\bar{\theta}}{\pi}(\tau_h) - 3 \log \frac{K \abs{\bar{\Theta}_\varepsilon}}{\delta} \leq \sum_h \sum_{(\tau_h, \pi) \in \calD_h^k} \log \probb{\theta^*}{\pi}(\tau_h)\\
        &\forall \bar{\theta} \in \bar{\Theta}_\varepsilon, \sum_{(\tau_H, \pi) \in \calD^k} \log \probb{\bar{\theta}}{\pi}(\tau_H) - 3 \log \frac{K \abs{\bar{\Theta}_\varepsilon}}{\delta} \leq \sum_{(\tau_h, \pi) \in \calD_h^k} \log \probb{\theta^*}{\pi}(\tau_h)
    \end{align*}
\end{proposition}

The second proposition provides an upper bound on the total variation distance between the distributions of futures given histories on the empirical history of trajectories. This result ensures that the model estimated by~\Cref{alg:post_ucb} is accurate on the sampled trajectories.

\begin{proposition}[Proposition 5 in~\cite{huangProvablyEfficientUCBtype2023}]\label{prop:mle_E_omega}
    Fix $\ p_{\min}$ and $\ \varepsilon \leq \frac{p_{\min}}{K H}$. Let \newline $\Theta_{\min}^k = \set{\theta: \forall h, (\tau_h, \pi) \in \calD_h^k,\ \probb{\theta}{\pi}(\tau_h) \geq p_{\min}}$. Then, with probability at least $1 - \delta$, for any $k \in [K], \theta \in \Theta_{\min}^k$, we have,
    \begin{equation*}
        \sum_{h} \sum_{(\tau_h, \pi) \in \calD_h^k} \tvsq{\probb{\theta}{\pi}(\omega_h | \tau_h), \probb{\theta^*}{\pi}(\omega_h | \tau_h)} \leq 6 \sum_h \sum_{(\tau_h, \pi) \in \calD_h^k} \log \frac{\probb{\theta^*}{\pi}(\tau_H)}{\probb{\theta}{\pi}(\tau_H)} + 31 \log \frac{K \abs{\bar{\Theta}_\varepsilon}}{\delta}.
    \end{equation*}
\end{proposition}

The next proposition is standard in the analysis of maximum likelihood estimation. $\mathtt{D}_{\mathtt{H}}$ denotes the Hellinger distance.

\begin{proposition}[Proposition 6 of~\cite{huangProvablyEfficientUCBtype2023}]\label{prop:mle_E_pi}
    Let $\varepsilon < \frac{1}{K^2 H^2}$. Then, with probability at least $1 - \delta$, the following holds for all $\theta \in \Theta$ and $k \in [K]$,
    \begin{equation*}
        \sum_{\pi \in \calD^k} \hellingersq{\probb{\theta}{\pi}(\tau_H), \probb{\theta^*}{\pi}(\tau_H)} \leq \frac{1}{2} \sum_{(\tau_H, \pi) \in \calD^k} \log \frac{\probb{\theta^*}{\pi}(\tau_H)}{\probb{\theta}{\pi}(\tau_H)} + 2 \log \frac{K \abs{\bar{\Theta}_\varepsilon}}{\delta}.
    \end{equation*}
\end{proposition}

The final proposition of this section states that when $p_{\min}$ is chosen as in~\Cref{theorem:post_ucb}, the true model $\theta^*$ lies in the constraint $\Theta_{\min}^k$ with high probability.
\begin{proposition}\label{prop:mle_E_min}
    Fix $p_{\min} \leq \frac{\delta}{K H \prod_{h=1}^H \abs{\bbX_{h}}}$. Then, with probability at least $1 - \delta$, we have $\theta^* \in \Theta_{\min}^k$ $\forall k$.
\end{proposition}

\begin{proof}
    For each $k \in [K]$, we have $\theta^* \in \Theta_{\min}^k$ if $\probb{\theta^*}{\pi^k}(\tau_h^k) \geq p_{\min}$ for all $h \in [H],\, (\tau_h^k, \pi^k) \in \calD_h^k$. Consider the probability of $\theta^*$ violating this constraint for some trajectory in the dataset. For each $k, h, (\tau_h^k, \pi^k)$, we have
    \begin{equation*}
        \begin{split}
            \prob{\probb{\theta^*}{\pi^k}(\tau_h^k) < p_{\min}} &= \Expect_{\pi}\bra{{\prob{\probb{\theta^*}{\pi^k}(\tau_h^k) < p_{\min} \given \pi^k=\pi}}} \\
            &= \Expect_{\pi}\bra{\sum_{\tau_h \in \bbH_h} \probb{\theta^*}{\pi}(\tau_h^k = \tau_h) \Ind{\probb{\theta^*}{\pi}(\tau_h) < p_{\min}}} \\
            &< \sum_{\tau_h \in \bbH_h} p_{\min}\\
            &= \abs{\bbH_h} p_{\min}\\
            &\leq \frac{\delta}{K H}.
        \end{split}
    \end{equation*}
    In the above, the first line is by the law of total probability, where the expectation is over the policy $\pi^k$ used while collecting the $(h,k)$-th trajectory, and the inner probability is over trajectories $\tau_h^k$. The second line calculates the probability of the event $\sset{\probb{\theta^*}{\pi^k}(\tau_h^k) < p_{\min}}$. Taking a union bound over $k \in [K]$, $h \in [H]$, and $(\tau_h, \pi) \in \calD_h$ implies that $\prob{\theta^* \in \Theta_{\mathrm{min}}^k} \geq 1 - \delta$.
\end{proof}

In what follows, let $\calE_{\omega}, \calE_{\pi}, \calE_{\mathrm{min}}$ be the events in~\Cref{prop:mle_E_omega,prop:mle_E_pi,prop:mle_E_min}, respectively. Let $\calE = \calE_{\omega} \cap \calE_{\pi} \cap \calE_{\mathrm{min}}$ be the intersection of all events.~\Cref{prop:mle_E_omega,prop:mle_E_pi,prop:mle_E_min} guarantee the event $\calE$ occurs with high probability, $\prob{\calE} \geq 1 - 3 \delta$, by a union bound.

The following result states that the estimated model is accurate on the past exploration policies and dataset of collected trajectories. This holds for both the conditional probabilities of futures given past trajectories in the dataset as well as over full trajectories. The result follows from the MLE analysis in~\Cref{prop:mle_E_omega,prop:mle_E_pi,prop:mle_E_min}.
\begin{lemma}\label{lemma:estimation_guarantee}
    Let $\beta = 31 \log \frac{K \abs{\bar{\Theta}_\varepsilon}}{\delta}$, and suppose $\varepsilon \leq \frac{\delta}{K^2 H^2 \prod_{h}\abs{\bbX_{h}}}$, where $\bar{\Theta}_{\varepsilon}$ is the optimistic $\varepsilon$-net in~\Cref{prop:bracketing_num_opt_net}. Then, under event $\calE$, the following holds,
    \begin{align*}
        &\sum_h \sum_{(\tau_h, \pi) \in \calD_h^k} \tvsq{\probb{\hat{\theta}^k}{\pi}(\omega_h | \tau_h), \probb{\theta^*}{\pi}(\omega_h | \tau_h)} \leq 7 \beta, \text{and}\\
        &\sum_{\pi \in \calD^k} \hellingersq{\probb{\hat{\theta}^k}{\pi}(\tau_H), \probb{\theta^*}{\pi}(\tau_H)} \leq 7 \beta, \\
    \end{align*}
\end{lemma}
\begin{proof}
    The proof follows by~\Cref{prop:mle_E_omega,prop:mle_E_pi,prop:mle_E_min}. The argument is direct and is identical to Lemma 1 of~\cite{huangProvablyEfficientUCBtype2023}.
\end{proof}

\subsection{UCB for Total Variation Distance}\label{ssec:proof_ucb_tv}
\textbf{Notation.} Let $m^*, \sset{M_h^*}_h$ be the observable operators of the true PSR $\theta^*$, and let $\sset{\hat{M}_h^k}_h$ be the algorithm's estimates of the observable operators corresponding to $\hat{\theta}^k$.

Recall that~\Cref{prop:tv_dist_leq_est_err} shows that the total variation distance between the distribution over trajectories of two PSRs is bounded by the estimation error of the observable operators $M_h$. The following result constructs a bound on the estimation error of the observable operators $M_h(x_h)$. The proof is adapted from~\parencite[Lemma 2]{huangProvablyEfficientUCBtype2023} to our setting with generalized PSRs.

\begin{lemma}\label{lemma:ucb_tv}
    Under event $\calE$, for any policy $\pi$ and $k\in [K]$, we have,
    \begin{equation*}
        \sum_{\tau_H}\abs{m^{\star}(\omega_h)^\top \paren{ \hat{M}^k_{h}(x_{h}) - M^{\star}_{h}(x_{h})}\hat{\psi}^k_{h-1}(\tau_{h-1})}\pi(\tau_H) \leq \Expect_{\tau_{h-1}\sim \bbP_{\hat{\theta}^{k}}}^\pi \bra{\alpha_{h-1}^k \norm{\hatbar{\psi}_{h-1}^k(\tau_{h-1})}_{(\hat{U}_{h-1}^k)^{-1}}}
    \end{equation*}
    where,
    \begin{align*}
        \hat{U}_{h-1}^k &= \lambda I + \sum_{\tau_{h-1} \in \calD_{h-1}^k}\bra{\hatbar{\psi}_h^k(\tau_{h-1})\hatbar{\psi}_h^k(\tau_{h-1})^\top}\\
        \paren{\alpha_{h-1}^k}^2 &= \frac{4\lambda Q_A^2 d}{\gamma^4} + \frac{4 \max_{s\in \calA}\abs{\bbX_s}^2 Q_A^2}{\gamma^2} \sum_{\tau_{h-1} \in \calD_{h-1}^k} \tvsq{\probb{\thetahatk}{\uhexp{h-1}}\paren{\omega_{h-1}^o \given \tau_{h-1}, \omega_{h-1}^a}, \probb{\theta^*}{\uhexp{h-1}}\paren{\omega_{h-1}^o \given \tau_{h-1}, \omega_h^a}}
    \end{align*}
\end{lemma}
\begin{proof}
    To ease notation, we index the future trajectories $\omega_{h-1} = (x_h, \ldots, x_H) \in \bbF_{h-1}$ by $i$ and history trajectories $\tau_{h-1} = (x_1, \ldots, x_{h-1}) \in \bbH_{h-1}$ by $j$. We denote $m^{\star}(\omega_h)^\top\paren{\hat{M}^k_{h}(x_{h}) - M^{\star}_{h}(x_{h})}$ as $w_i^\top$, $\hatbar{\psi}_h^k(\tau_{h-1})$ as $x_j$, and $\pi(\omega_{h-1}|\tau_{h-1})$ as $\pi_{i|j}$.

    The following bound follows from the Cauchy-Schwarz inequality,
    \begin{align*}
        &\sum_{\tau_H}\abs{m^{\star}(\omega_h)^\top \paren{ \hat{M}^k_{h}(x_{h}) - M^{\star}_{h}(x_{h})}\hat{\psi}^k_{h-1}(\tau_{h-1})}\pi(\tau_H) \\
        &\stepa{=} \sum_{\omega_{h-1}} \sum_{\tau_{h-1}} \abs{m^{\star}(\omega_h)^\top \paren{ \hat{M}^k_{h}(x_{h}) - M^{\star}_{h}(x_{h})} \hatbar{\psi}_h^k(\tau_{h-1})} \pi(\omega_{h-1}| \tau_{h-1}) \probb{\thetahatk}{\pi}(\tau_{h-1}) \\
        &= \sum_{i} \sum_j \abs{w_i^\top x_j} \pi_{i|j} \probb{\thetahatk}{\pi}(j) \\
        &= \sum_{i} \sum_j  \paren{\pi_{i|j} \cdot \sign(w_i^\top x_j) w_i}^\top x_j \cdot \probb{\thetahatk}{\pi}(j) \\
        &= \sum_{j} \paren{\sum_i  \pi_{i|j} \cdot \sign(w_i^\top x_j) w_i}^\top x_j \cdot \probb{\thetahatk}{\pi}(j) \\
        &= \Expect_{j \sim \probb{\thetahatk}{\pi}}\bra{\paren{\sum_i  \pi_{i|j} \cdot \sign(w_i^\top x_j) w_i}^\top x_j} \\
        &\stepb{\leq} \Expect_{j \sim \probb{\thetahatk}{\pi}}\bra{\norm{x_j}_{\paren{\hat{U}_{h-1}^k}^{-1}} \norm{\sum_i \pi_{i|j} \cdot \sign(w_i^\top x_j) \cdot w_i }_{\hat{U}_{h-1}^k}}.
    \end{align*}
    Step (a) follows from the fact that
    $\hat{\psi}^k_{h-1}(\tau_{h-1}) = \hatbar{\psi}_{h}^{k}(\tau_{h-1}) \cdot (\hat{\phi}_{h-1}^{k})^\top \hat{\psi}_{h-1}^{k}(\tau_{h-1}) = \hatbar{\psi}_{h}^{k}(\tau_{h-1}) \cdot \probbarunder{\thetahatk}{\tau_{h-1}}$
    and $\probbarunder{\thetahatk}{\tau_{h-1}} \cdot \pi(\tau_H) = \pi(\omega_{h-1} | \tau_{h-1}) \cdot \probb{\thetahatk}{\pi}(\tau_{h-1})$. Step (b) is the Cauchy-Schwarz inequality.

    Fix $\tau_{h-1} = j_0$. Let $I_1 := \norm{\sum_{i}\pi_{i|j_0}\cdot \sign (w_i^\top x_{j_0})\cdot w_i}_{\hat{U}_{h-1}^k}^2$, which we bound next. By the definition of $\hat{U}_{h-1}^k$, we partition this term into two parts,
    \begin{equation*}
        I_1 = \underbrace{\lambda \twonorm{\sum_{i}\pi_{i|j_0}\cdot \sign(w_i^\top x_{j_0})\cdot w_i}^2}_{I_2} +
        \underbrace{\sum_{j\in D_{h-1}^\tau}\bra{\paren{\sum_{i}\pi_{i|j_0}\cdot \sign(w_i^\top x_{j_0})\cdot w_i}^\top x_j}^2}_{I_3}.
    \end{equation*}

    We bound $I_2$ and $I_3$ separately. By the triangle inequality, $\sqrt{I_2}$ is bound by a sum of two terms,
    \begin{align*}
        \sqrt{I_2} &= \sqrt{\lambda} \max_{\substack{z\in \reals^{d_{h-1}}: \twonorm{z}= 1}} \abs{\sum_{i}\pi_{i|j_0}\cdot \sign (w_i^\top x_{j_0})\cdot w_i^\top z}\\
        &\stepa{\leq} \sqrt{\lambda} \max_{\substack{\twonorm{z}= 1}}\sum_{\omega_{h-1}}\abs{m^{\star}(\omega_h^\top )\paren{\hat{M}^k_{h}(x_{h}) - M^{\star}_{h}(x_{h})}z}\pi(\omega_{h-1}|j_0)\\
        &\stepb{\leq}\sqrt{\lambda}\max_{\substack{\twonorm{z}= 1}}
        \sum_{\omega_{h-1}}\abs{m^{\star}(\omega_h)^\top  \hat{M}_{h}^k(x_{h})z}\pi(\omega_{h-1}|j_0)\\
        &\ \ \ + \sqrt{\lambda}\max_{\substack{\twonorm{z}= 1}}
        \sum_{\omega_{h-1}}\abs{m^{\star}(\omega_h)^\top  M_{h}^\star(x_{h})z}\pi(\omega_{h-1}|j_0),
    \end{align*}
    where step (a) is by the definition of $w_i^\top, \pi_{i|j_0}$ and the triangle inequality, and step (b) is by the triangle inequality.

    Consider the first term. It can be bound via the definition of $\gamma$-well-conditioning as follows,
    \begin{align*}
        \max_{\substack{\twonorm{z}= 1}}
            &\sum_{\omega_{h-1}}\abs{m^{\star}(\omega_h)^\top  \hat{M}_{h}^k(x_{h})z}\pi(\omega_{h-1}|j_0) \\
            &= \max_{\substack{\twonorm{z}= 1}} \sum_{x_h} \paren{\sum_{\omega_{h}}\abs{m^{\star}(\omega_h)^\top  \hat{M}_{h}^k(x_{h})z}\pi(\omega_{h}|j_0, x_h)} \pi(x_h | j_0) \\
            &\stepa{\leq} \frac{1}{\gamma} \max_{\substack{\twonorm{z}= 1}} \sum_{x_h} \onenorm{\hat{M}_{h}^k(x_{h})z} \pi(x_h | j_0) \\
            &\stepb{\leq} \frac{1}{\gamma} \max_{\substack{\twonorm{z}= 1}} \frac{\abs{\bbQ_{h+1}^A}\onenorm{z}}{\gamma} \\
            &\stepc{\leq} \frac{\sqrt{d} Q_A}{\gamma^2}
    \end{align*}
    where step (a) is by the first condition in~\Cref{ass:psr_gamma_wellcond}, step (b) is by the second condition of~\Cref{ass:psr_gamma_wellcond}, and step (c) is by the fact that $\max_{z \in \reals^{d_{h-1}}: \twonorm{z}=1} \onenorm{z} = \sqrt{d_{h-1}} \leq \sqrt{d}$ and $\abs{\bbQ_{h+1}^A} \leq Q_A$. In the above, note that we used the $\gamma$-well-conditioning of PSR $\hat{\theta}^k$ in step (a) and the $\gamma$-well-conditioning of PSR $\theta^*$ in step (b). The second term in $\sqrt{I_2}$ admits an identical bound, simply by using the well-conditioning of the PSR $\theta^*$ in both steps. Hence, we have that
    \begin{equation}
        I_2 \leq 4 \frac{\lambda d Q_A^2}{\gamma^4}.
    \end{equation}

    Now we upper bound $I_3$,
    \begin{align*}
        I_3 &\leq \sum_{\tau_{h-1}\in \mathcal{D}_{h-1}^k} \paren{\sum_{\omega_{h-1}} \abs{m^{\star}(\omega_h)^\top \paren{\hat{M}_{h}^k(x_{h}) - M^{\star}_{h}(x_{h})}\hatbar{\psi}^k(\tau_{h-1})}\pi(\omega_{h-1}|j_0)}^2\\
        &\leq \sum_{\tau_{h-1}\in \mathcal{D}_{h-1}^k} \Biggl(\underbrace{\sum_{\omega_{h-1}} \abs{ m^\star(\omega_h)^\top \paren{\hat{M}_{h}^k(x_{h})\hatbar{\psi}^k(\tau_{h-1}) - M_{h}^\star(x_{h})\bar{\psi}^\star(\tau_{h-1})}} \pi(\omega_{h-1}|j_0)}_{I_4}\\
        &\ \ \ + \underbrace{\sum_{\omega_{h-1}}\abs{ m^\star(\omega_h)^\top M_{h}^{\star}(x_{h})\paren{\hatbar{\psi}^k(\tau_{h-1}) - \bar{\psi}^\star(\tau_{h-1})}} \pi(\omega_{h-1}|j_0)}_{I_5} \Biggr)^2\\
        &=: \sum_{\tau_{h-1}\in \mathcal{D}_{h-1}^k}(I_4 + I_5)^2
    \end{align*}
    where the second equality follows from the triangle inequality by adding and subtracting $m^*(\omega_h)^\top M_h^*(x_h) \bar{\psi}^*(\tau_{h-1})$ inside the absolute value. We now bound each of $I_4$ and $I_5$.

    \begin{align*}
        I_4 &\coloneqq \sum_{\omega_{h-1}} \abs{ m^\star(\omega_h)^\top \paren{\hat{M}_{h}^k(x_{h})\hatbar{\psi}^k(\tau_{h-1}) - M_{h}^\star(x_{h})\bar{\psi}^\star(\tau_{h-1})}} \pi(\omega_{h-1}|j_0) \\
        &\stepa{=} \sum_{\omega_{h-1}} \abs{ m^\star(\omega_h)^\top \paren{\probbarunder{\thetahatk}{x_h \given \tau_{h-1}} \hatbar{\psi}_h(\tau_h) - \probbarunder{\theta^*}{x_h \given \tau_{h-1}} \bar{\psi}_h^*(\tau_h)}} \pi(\omega_{h-1}|j_0) \\
        &\stepb{=} \sum_{x_h} \paren{\sum_{\omega_{h}} \abs{ m^\star(\omega_h)^\top \paren{\probbarunder{\thetahatk}{x_h \given \tau_{h-1}} \hatbar{\psi}_h(\tau_h) - \probbarunder{\theta^*}{x_h \given \tau_{h-1}} \bar{\psi}_h^*(\tau_h)}} \pi(\omega_{h}|j_0, x_h)} \pi(x_h | j_0) \\
        &\stepc{\leq} \frac{1}{\gamma} \sum_{x_h} \onenorm{\probbarunder{\thetahatk}{x_h \given \tau_{h-1}} \hatbar{\psi}_h(\tau_h) - \probbarunder{\theta^*}{x_h \given \tau_{h-1}} \bar{\psi}_h^*(\tau_h)} \pi(x_h | j_0) \\
        &\stepd{=} \frac{1}{\gamma} \sum_{x_h} \sum_{q_h \in \bbQ_h} \abs{\probbarunder{\thetahatk}{x_h, q_h \given \tau_{h-1}} - \probbarunder{\theta^*}{x_h, q_h \given \tau_{h-1}}} \pi(x_h | j_0)
    \end{align*}
    where step (a) is by the fact that $M_h(x_h) \bar{\psi}_{h-1}(\tau_{h-1}) = \probbar{x_h \given \tau_{h-1}} \bar{\psi}(\tau_h)$, as shown in~\Cref{eq:Mpsi_psi_relation}, step (b) uses $\omega_{h-1} = (x_h, \omega_h)$ and $\pi(\omega_{h-1}|j_0) = \pi(x_h |j_0) \pi(\omega_h | j_0, x_h)$, step (c) is by~\Cref{ass:psr_gamma_wellcond}, and step (d) follows by the definition $\bar{\psi}_h$, $\bra{\bar{\psi}_h(\tau_h)}_l = \probbarunder{\theta}{q_h^l \given \tau_h}$.

    Now, we turn to bound the $I_5$ term. We have
    \begin{align*}
        I_5 &=  \sum_{\omega_{h}} \sum_{x_h} \abs{ m_h^\star(\omega_h)^\top M_{h}^{\star}(x_{h})\paren{\hatbar{\psi}^k(\tau_{h-1}) - \bar{\psi}^\star(\tau_{h-1})}} \pi(\omega_{h}|j_0, x_h) \pi(x_h | j_0) \\
        &\stepa{=}  \sum_{\omega_{h-1}} \abs{ m_{h-1}^\star(\omega_{h-1})^\top \paren{\hatbar{\psi}^k(\tau_{h-1}) - \bar{\psi}^\star(\tau_{h-1})}} \pi(\omega_{h-1}|j_0) \\
        &\stepa{\leq} \frac{1}{\gamma} \onenorm{\hatbar{\psi}^k(\tau_{h-1}) - \bar{\psi}^\star(\tau_{h-1})}\\
        &= \frac{1}{\gamma} \sum_{q_{h-1} \in \bbQ_{h-1}} \abs{\probbarunder{\thetahatk}{q_{h-1} \given \tau_{h-1}} - \probbarunder{\theta^*}{q_{h-1} \given \tau_{h-1}}},
    \end{align*}
    where step (a) is since $m_h^*(\omega_h)^\top M_h^*(x_h) = m_{h-1}^*(\omega_{h-1})^\top$, step (b) is by the first condition of~\Cref{ass:psr_gamma_wellcond}, and the final equality is again by the definition of $\bar{\psi}$.

    Combining the above, we have that,
    \begin{align*}
        I_3 &\leq \sum_{\tau_{h-1} \in \calD_{h-1}^k} \paren{I_4 + I_5}^2 \\
        &\leq \begin{aligned}[t]
            \sum_{\tau_{h-1} \in \calD_{h-1}^k} \Biggl(&\frac{1}{\gamma} \sum_{x_h \in \bbX_h} \sum_{q_h \in \bbQ_h} \abs{\probbarunder{\thetahatk}{x_h, q_h \given \tau_{h-1}} - \probbarunder{\theta^*}{x_h, q_h \given \tau_{h-1}}} \pi(x_h | \tau_{h-1}) \\
            +&\frac{1}{\gamma} \sum_{q_{h-1} \in \bbQ_{h-1}} \abs{\probbarunder{\thetahatk}{q_{h-1} \given \tau_{h-1}} - \probbarunder{\theta^*}{q_{h-1} \given \tau_{h-1}}}\Biggr)^2
        \end{aligned}\\
        &\leq\begin{aligned}[t]
            \frac{1}{\gamma^2} \cdot \sum_{\tau_{h-1} \in \calD_{h-1}^k} \Biggl( &\sum_{(x_h, q_h) \in \bbX_h \times \bbQ_h} \abs{\probbarunder{\thetahatk}{x_h, q_h \given \tau_{h-1}} - \probbarunder{\theta^*}{x_h, q_h \given \tau_{h-1}}} \pi(x_h | \tau_{h-1}) \\
            + &\sum_{q_{h-1} \in \bbQ_{h-1}} \abs{\probbarunder{\thetahatk}{q_{h-1} \given \tau_{h-1}} - \probbarunder{\theta^*}{q_{h-1} \given \tau_{h-1}}}\Biggr)^2
        \end{aligned}
    \end{align*}
    Now, we decompose the summations above over $\bbX_h \times \bbQ_h$ and $\bbQ_{h-1}$ into separate summations over observation futures and action futures. That is, $(x_h, q_h)$ is decomposed into $(\omega_{h-1}^a, \omega_{h-1}^o)$, where $\omega_{h-1}^a = \act(x_h, q_h)$ and $\omega_{h-1}^o = \obs(x_h, q_h)$, and the summations are over $\omega_{h-1}^a \in \act(\bbX_h \times \bbQ_h)$ and $\omega_{h-1}^o \in \obs(\bbX_h \times \bbQ_h)$. Similarly, $q_{h-1}$ can be decomposed into $(q_{h-1}^o, q_{h-1}^a) \in \obs(\bbQ_{h-1}) \times \act(\bbQ_{h-1})$. Hence, the bound on $I_3$ can be written as,
    \begin{align*}
        I_3 &\leq\begin{aligned}[t]
            \frac{1}{\gamma^2} \cdot \sum_{\tau_{h-1} \in \calD_{h-1}^k} \Biggl(&\sum_{\omega_{h-1}^a} \sum_{\omega_{h-1}^o} \abs{\probbarunder{\thetahatk}{\omega_{h-1}^o \given \tau_{h-1}, \omega_{h-1}^a} - \probbarunder{\theta^*}{\omega_{h-1}^o \given \tau_{h-1}, \omega_h^a}} \pi(x_h | \tau_{h-1}) \\
            + &\sum_{q_{h-1}^a} \sum_{q_{h-1}^o} \abs{\probbarunder{\thetahatk}{q_{h-1}^o \given \tau_{h-1}, q_{h-1}^a} - \probbarunder{\theta^*}{q_{h-1}^o \given \tau_{h-1}, q_{h-1}^a}}\Biggr)^2
        \end{aligned}\\
        &\leq\frac{1}{\gamma^2} \cdot \sum_{\tau_{h-1} \in \calD_{h-1}^k} \paren{\sum_{\omega_{h-1}^a \in \bbQ_{h-1}^{\mathtt{\exp}}} \sum_{\omega_{h-1}^o}  \abs{\probbarunder{\thetahatk}{\omega_{h-1}^o \given \tau_{h-1}, \omega_{h-1}^a} - \probbarunder{\theta^*}{\omega_{h-1}^o \given \tau_{h-1}, \omega_{h-1}^a}}}^2 \\
        &=\frac{1}{\gamma^2} \abs{\bbQ_{h-1}^\mathrm{\exp}}^2 \cdot \sum_{\tau_{h-1} \in \calD_{h-1}^k} \tvsq{\probb{\thetahatk}{\uhexp{h-1}}\paren{\omega_{h-1}^o \given \tau_{h-1}, \omega_{h-1}^a}, \probb{\theta^*}{\uhexp{h-1}}\paren{\omega_{h-1}^o \given \tau_{h-1}, \omega_{h-1}^a}}.
    \end{align*}
    Where the second inequality is by the definition of $\Qhexp{h-1} = \act\paren{\bbX_h \times \bbQ_h \cup \bbQ_{h-1}}$. Here, the second summation is over $\omega_{h-1}^o \in \obs\paren{\bbX_h \times \bbQ_h \cup \bbQ_{h-1}}$. The final equality uses the fact the under the policy $\uhexp{h-1}$ the probability of each action sequence $\omega_{h-1}^o$ is $1 / \abs{\Qhexp{h-1}}$. Note that $\abs{\Qhexp{h-1}} \leq \abs{\act\paren{\bbX_h \times \bbQ_h}} + \abs{\act\paren{\bbQ_{h-1}}}$, and hence we have $\abs{\Qhexp{h-1}} \leq 2 \max_{s \in \calA} \abs{\bbX_s} Q_A$ for all $h$. Hence, we have,
    \begin{equation}
        I_3 \leq 4\max_{s \in \calA}\abs{\bbX_s}^2 Q_A^2 \frac{1}{\gamma^2} \sum_{\tau_{h-1} \in \calD_{h-1}^k} \tvsq{\probb{\thetahatk}{\uhexp{h-1}}\paren{\omega_{h-1}^o \given \tau_{h-1}, \omega_{h-1}^a}, \probb{\theta^*}{\uhexp{h-1}}\paren{\omega_{h-1}^o \given \tau_{h-1}, \omega_h^a}}.
    \end{equation}

    Putting this together with the bounds on $I_2$ and $I_3$, we get that,
    \begin{align*}
        I_1 &\leq \frac{4\lambda Q_A^2 d}{\gamma^4} + 4 \max_{s\in \calA}\abs{\bbX_s}^2 Q_A^2 \frac{1}{\gamma^2} \sum_{\tau_{h-1} \in \calD_{h-1}^k} \tvsq{\probb{\thetahatk}{\uhexp{h-1}}\paren{\omega_{h-1}^o \given \tau_{h-1}, \omega_{h-1}^a}, \probb{\theta^*}{\uhexp{h-1}}\paren{\omega_{h-1}^o \given \tau_{h-1}, \omega_h^a}}\\
        &=: \paren{\alpha_{h-1}^k}^2,
    \end{align*}
    completing the proof.
\end{proof}


Using the above bound on the difference between the observable operators of the true model and the estimated model, we now bound the total variation distance between the distributions of trajectories through~\Cref{prop:tv_dist_leq_est_err}.

\begin{lemma}\label{lemma:tv_bound}
    Under even $\calE$, the total variation distance between the estimated model at iteration $k$, $\thetahatk$, and the true model $\theta^*$, is bounded by,
    \begin{equation}
        \tv{\probb{\thetahatk}{\pi}(\tau_H), \probb{\theta^*}{\pi}(\tau_H)} \leq \alpha \cdot \Expect_{\tau_H \sim \probb{\thetahatk}{\pi}}\bra{\sqrt{\sum_{h=0}^{H-1} \norm{\hatbar{\psi}^k(\tau_h)}_{(\hat{U}_h^k)^{-1}}^2}},
    \end{equation}
    for any policy $\pi$, where
    \begin{equation*}
        \alpha^2 = \frac{4\lambda H Q_A^2 d}{\gamma^4} + 28 \max_{s\in \calA}\abs{\bbX_s}^2 Q_A^2 \frac{1}{\gamma^2} \beta
    \end{equation*}
\end{lemma}
\begin{proof}
    Consider $\alpha_{h-1}^k$ in the previous lemma. We have that,
    \begin{align*}
        &\sum_{h=1}^{H} \paren{\alpha_{h-1}^k}^2 \\
        &= \frac{4\lambda H Q_A^2 d}{\gamma^4} + 4 \max_{s\in \calA}\abs{\bbX_s}^2 Q_A^2 \frac{1}{\gamma^2} \sum_{h=1}^{H}  \sum_{\tau_{h-1} \in \calD_{h-1}^k} \tvsq{\probb{\thetahatk}{\uhexp{h-1}}\paren{\omega_{h-1}^o \given \tau_{h-1}, \omega_{h-1}^a}, \probb{\theta^*}{\uhexp{h-1}} \paren{\omega_{h-1}^o \given \tau_{h-1}, \omega_h^a}} \\
        &\leq \frac{4\lambda H Q_A^2 d}{\gamma^4} + 4 \max_{s\in \calA}\abs{\bbX_s}^2 Q_A^2 \frac{1}{\gamma^2} 7 \beta =: \alpha^2,
    \end{align*}
    where the inequality is by the bound on the total variation distance established in~\Cref{lemma:estimation_guarantee}.

    Now, by~\Cref{prop:tv_dist_leq_est_err}, the total variation distance is bounded by the estimation error:
    \begin{align*}
        &\tv{\probb{\thetahatk}{\pi}(\tau_H), \probb{\theta^*}{\pi}(\tau_H)} \\
        &\stepa{\leq} \sum_{h=1}^{H} \sum_{\tau_H} \abs{m^{\star}(\omega_h)^\top \paren{ \hat{M}^k_{h}(x_{h}) - M^{\star}_{h}(x_{h})}\hat{\psi}^k_{h-1}(\tau_{h-1})}\pi(\tau_H) \\
        &\stepb{\leq} \sum_{h=1}^{H} \Expect_{\tau_{h-1}\sim \bbP_{\hat{\theta}^{k}}}^\pi \bra{\alpha_{h-1}^k \norm{\hatbar{\psi}_{h-1}^k(\tau_{h-1})}_{(\hat{U}_{h-1}^k)^{-1}}}\\
        &\stepc{\leq} \alpha \cdot \Expect_{\tau_H \sim \probb{\thetahatk}{\pi}}\bra{\sqrt{\sum_{h=0}^{H-1} \norm{\hatbar{\psi}^k(\tau_h)}_{(\hat{U}_h^k)^{-1}}^2}},
    \end{align*}
    where step (a) is by~\Cref{prop:tv_dist_leq_est_err}, step (b) is by~\Cref{lemma:ucb_tv}, and step (c) is by the Cauchy-Schwarz inequality and the calculation above bounding $\sum_h \paren{\alpha_{h-1}^k}^2$.
\end{proof}

A direct corollary is the following bound on the error in the estimated value function, which establishes that the bonus term $\hat{b}^k$ gives an upper confidence bound.

\begin{corollary}[Upper confidence bound]\label{cor:upper_conf_bound}
    Under the event $\calE$, for any $k \in [K]$, any reward function $R: \prod_{h \in [H]} \bbX_h \to [0, 1]$, and any policy $\pi$, we have,
    \begin{equation*}
        \abs{V_{\thetahatk}^R(\pi) - V_{\theta^*}^R(\pi)} \leq V_{\thetahatk}^{\hat{b}^k},
    \end{equation*}
    where $\hat{b}^k(\tau_H) = \min\set{\alpha \sqrt{\sum_h \norm{\hatbar{\psi}^{k}(\tau_h)}_{(\hat{U}_h^k)^{-1}}^2}, 1}$.
\end{corollary}
\begin{proof}
    By a direct calculation,
    \begin{align*}
        \abs{V_{\thetahatk}^R(\pi) - V_{\theta^*}^R(\pi)} &= \abs{\sum_{\tau_H} R(\tau_H) \probb{\thetahatk}{\pi}(\tau_H) - \sum_{\tau_H} R(\tau_H) \probb{\theta^*}{\pi}(\tau_H)} \\
        &\stepa{\leq} \sum_{\tau_H} \abs{ \probb{\thetahatk}{\pi}(\tau_H) - \probb{\theta^*}{\pi}(\tau_H)} \\
        &= \tv{\probb{\thetahatk}{\pi}(\tau_H), \probb{\theta^*}{\pi}(\tau_H)} \\
        &\stepb{\leq} \alpha \cdot \Expect_{\tau_H \sim \probb{\thetahatk}{\pi}}\bra{\sqrt{\sum_{h=0}^{H-1} \norm{\hatbar{\psi}^k(\tau_h)}_{(\hat{U}_h^k)^{-1}}^2}}\\
        &\stepc{\leq} \alpha \sum_{\tau_H} \hat{b}^k(\tau_H) \probb{\thetahatk}{\pi}(\tau_H) \\
        &=: V_{\thetahatk}^{\hat{b}^k}
    \end{align*}
    where step (a) is by the triangle inequality and the fact that $R(\tau_H) \in [0,1]$, step (b) is by~\Cref{lemma:tv_bound}, and step (c) is by the definition of $\hat{b}^k$.
\end{proof}

\subsection{\texorpdfstring{$\sum_{k=1}^K V_{\thetahatk}^{\hat{b}^k}$}{Cumulative UCB} is sublinear}\label{ssec:proof_sublinear_est}
The next step is to prove that $\sum_{k=1}^K V_{\thetahatk}^{\hat{b}^k} = O(\sqrt{K})$. To do that, we first prove that the estimated prediction features and the ground-truth prediction features can be related through the total-variation distance between the estimated model and the true model.

\begin{lemma}\label{lemma:features_bound}
    Under event $\calE$, for any $k\in [K]$, we have:
    \begin{align*}
        &\Expect_{\tau_H \sim \probb{\theta^*}{\pi}}\bra{\sqrt{\sum_{h=0}^{H-1} \norm{\hatbar{\psi^k}(\tau_h)}_{(\hat{U}_h^k)^{-1}}^2}} \\
        &\leq \frac{2 H Q_A}{\sqrt{\lambda}} \tv{\probb{\theta^*}{\pi}(\tau_h), \probb{\thetahatk}{\pi}(\tau_h)} + \paren{1 + \frac{2 \max_{s \in \calA} \abs{\bbX_s} Q_A \sqrt{7 r \beta} }{\sqrt{\lambda}}} \sum_{h=0}^{H-1} \Expect_{\tau_h \sim \probb{\theta^*}{\pi}}\norm{\bar{\psi^*}(\tau_h)}_{(\Uhk)^{-1}}
    \end{align*}
\end{lemma}

\begin{proof}
    First, we recall the definition of $\hat{U}_h^k$, and we define its ground-truth counterpart replacing estimated features with true features,
    \begin{align*}
        \hat{U}_h^k &= \lambda I + \sum_{\tau \in \calD_h^k} \hatbar{\psi^k}(\tau_h) \hatbar{\psi^k}(\tau_h)^\top, \\
        U_h^k &= \lambda I + \sum_{\tau \in \calD_h^k} \bar{\psi^*}(\tau_h) \bar{\psi^*}(\tau_h)^\top.
    \end{align*}

    For any trajectory $\tau_H \in \bbH_H$, we have,
    \begin{align*}
        \sqrt{\sum_{h=0}^{H-1} \norm{\hatbar{\psi^k}(\tau_h)}_{(\Uhathk)^{-1}}^2} &\stepa{\leq} \sum_{h=0}^{H-1} \norm{\hatbar{\psi^k}(\tau_h)}_{(\hat{U}_h^k)^{-1}}\\
        &\leq \frac{1}{\sqrt{\lambda}} \sum_{h=0}^{H-1} \twonorm{\hatbar{\psi^k}(\tau_h) - \bar{\psi^*}(\tau_h)} + \sum_{h=0}^{H-1} \paren{1 + \frac{\sqrt{r}\sqrt{\sum_{\tau_h \in \calD_h^k} \twonorm{\hatbar{\psi^k}(\tau_h) - \bar{\psi^*}(\tau_h)}^2}}{\sqrt{\lambda}}} \norm{\bar{\psi^*}(\tau_h)}_{(U_h^k)^{-1}},
    \end{align*}
    where step (a) is simply using $\twonorm{x} \leq \onenorm{x}$ and step (b) is by the identity~\parencite[Lemma 13]{huangProvablyEfficientUCBtype2023}. Note that $r$ is the rank of the PSR and $r \geq \mathrm{rank}(\{\hatbar{\psi^k}(\tau_h) \,:\, \tau_h \in \bbH_h\}), \mathrm{rank}(\{\bar{\psi^*}(\tau_h) \,:\, \tau_h \in \bbH_h\})$.

    Moreover, we have,
    \begin{align*}
        \twonorm{\hatbar{\psi^k}(\tau_h) - \bar{\psi^*}(\tau_h)} &\leq \onenorm{\hatbar{\psi^k}(\tau_h) - \bar{\psi^*}(\tau_h)} \\
        &\stepa{=} \sum_{q_h \in \bbQ_h} \abs{\probbarunder{\thetahatk}{q_h^o \given \tau_h, q_h^a} - \probbarunder{\theta^*}{q_h^o \given \tau_h, q_h^a}}\\
        &\stepb{\leq} 2 \max_{s \in \calA} \abs{\bbX_s} Q_A \tv{\probb{\thetahatk}{\uhexp{h-1}}(\cdot | \tau_h), \probb{\theta^*}{\uhexp{h-1}}(\cdot | \tau_h)},
    \end{align*}
    where we used the definition of $\bar{\psi}$ in (a) and the definition of the $\uhexp{h-1}$ in (b).

    Now, by~\Cref{lemma:estimation_guarantee}, we have,
    \begin{align*}
        \sqrt{\sum_{h=0}^{H-1} \norm{\hatbar{\psi^k}(\tau_h)}_{(\hat{U}_h^k)^{-1}}^2} &\leq \frac{1}{\sqrt{\lambda}} \sum_{h=0}^{H-1} \twonorm{\hatbar{\psi^k}(\tau_h) - \bar{\psi^*}(\tau_h)} + \sum_{h=0}^{H-1} \paren{1 + \frac{\sqrt{r}\sqrt{\sum_{\tau_h \in \calD_h^k} \twonorm{\hatbar{\psi^k}(\tau_h) - \bar{\psi^*}(\tau_h)}^2}}{\sqrt{\lambda}}} \norm{\bar{\psi^*}(\tau_h)}_{(\Uhk)^{-1}} \\
        &\leq \frac{1}{\sqrt{\lambda}} \sum_{h=0}^{H-1} \twonorm{\hatbar{\psi^k}(\tau_h) - \bar{\psi^*}(\tau_h)} + \paren{1 + \frac{2 \max_{s \in \calA} \abs{\bbX_s} Q_A \sqrt{7 r \beta} }{\sqrt{\lambda}}} \sum_{h=0}^{H-1} \norm{\bar{\psi^*}(\tau_h)}_{(\Uhk)^{-1}},
    \end{align*}
    where the first line is combining the calculations above and the second line is by the estimation guarantee of~\Cref{lemma:estimation_guarantee}.

    The first term can be bounded in expectation under $\probb{\theta^*}{\pi}$ for any $\pi$ as follows,
    \begin{align*}
        \sum_{h=0}^{H-1} \Expect_{\tau_h \sim \probb{\theta^*}{\pi}} \bra{\twonorm{\hatbar{\psi^k}(\tau_h) - \bar{\psi^*}(\tau_h)}} &\leq \sum_{h=0}^{H-1} \Expect_{\tau_h \sim \probb{\theta^*}{\pi}} \bra{\onenorm{\hatbar{\psi^k}(\tau_h) - \bar{\psi^*}(\tau_h)}} \\
        &\leq \sum_{h=0}^{H-1} \sum_{\tau_h} \onenorm{\hatbar{\psi^k}(\tau_h) \paren{\probb{\theta^*}{\pi}(\tau_h) - \probb{\thetahatk}{\pi}(\tau_h)} + \hatbar{\psi^k}(\tau_h) \probb{\thetahatk}{\pi}(\tau_h) - \bar{\psi^*}(\tau_h) \probb{\theta^*}{\pi}(\tau_h)} \\
        &\stepa{\leq} \sum_{h=0}^{H-1} \sum_{\tau_h} \onenorm{\hatbar{\psi^k}(\tau_h)} \abs{\probb{\theta^*}{\pi}(\tau_h) - \probb{\thetahatk}{\pi}(\tau_h)} + \onenorm{\hatbar{\psi^k}(\tau_h) \probb{\thetahatk}{\pi}(\tau_h) - \bar{\psi^*}(\tau_h) \probb{\theta^*}{\pi}(\tau_h)}\\
        &\stepb{\leq} \sum_{h=0}^{H-1} \sum_{\tau_h} \paren{\onenorm{\hatbar{\psi^k}(\tau_h)} \abs{\probb{\theta^*}{\pi}(\tau_h) - \probb{\thetahatk}{\pi}(\tau_h)} + \onenorm{\hat{\psi^k}(\tau_h) - \psi^*(\tau_h) } \pi(\tau_h)}\\
        &\stepc{\leq} 2 Q_A \sum_{h=0}^{H-1} \tv{\probb{\theta^*}{\pi}(\tau_h), \probb{\thetahatk}{\pi}(\tau_h)} \\
        &\stepd{\leq} 2 H Q_A \tv{\probb{\theta^*}{\pi}(\tau_h), \probb{\thetahatk}{\pi}(\tau_h)},
    \end{align*}
    where step (a) is the triangle inequality, step (b) is the definition of $\bar{\psi}(\tau_h)$, step (c) is since $\onenorm{\hatbar{\psi^k}(\tau_h)} \leq \abs{\bbQ_h^A} \leq Q_A$ for any $\tau_h$ and the definition of $\psi(\tau_h)$, and step (d) is simply $\tv{\probb{\theta^*}{\pi}(\tau_h), \probb{\thetahatk}{\pi}(\tau_h)} \geq \tv{\probb{\theta^*}{\pi}(\tau_h), \probb{\thetahatk}{\pi}(\tau_h)}$.

    Putting this together concludes the proof,
    \begin{align*}
        &\Expect_{\tau_H \sim \probb{\theta^*}{\pi}}\bra{\sqrt{\sum_{h=0}^{H-1} \norm{\hatbar{\psi^k}(\tau_h)}_{(\hat{U}_h^k)^{-1}}^2}} \\
        &\leq \frac{2 H Q_A}{\sqrt{\lambda}} \tv{\probb{\theta^*}{\pi}(\tau_h), \probb{\thetahatk}{\pi}(\tau_h)} + \paren{1 + \frac{2 \max_{s \in \calA} \abs{\bbX_s} Q_A \sqrt{7 r \beta} }{\sqrt{\lambda}}} \sum_{h=0}^{H-1} \Expect_{\tau_h \sim \probb{\theta^*}{\pi}}\norm{\bar{\psi^*}(\tau_h)}_{(\Uhk)^{-1}}.
    \end{align*}
\end{proof}

The following lemma bounds the cumulative estimation error of the probability of trajectories. It can be proved via an $\ell_2$ Eluder argument~\parencite{chenPartiallyObservableRL2022,zhongGECUnifiedFramework2023,qiu2023posterior}. A significant portion of the proof is very similar to that of~\Cref{prop:tv_dist_leq_est_err}, involving an exchange of $\hat{(\cdot)}$ and $(\cdot)^*$. We include the proof for completeness.

\begin{lemma}\label{lemma:cumulative_tv_error_bound}
    Under event $\calE$, for any $h\in \{0, \ldots, H-1\}$, we have
    \begin{equation*}
        \sum_{k}\tv{\probb{\theta^\star}{\pi^k}(\tau_H), \probb{\thetahatk}{\pi^k}(\tau_H)} \lesssim \frac{\max_{s \in \calA} \abs{\bbX_s} Q_A \sqrt{\beta}}{\gamma}\sqrt{rHK\log\paren{1+\frac{dQ_A K}{\gamma^4}}}.
    \end{equation*}
    Here, $a \lesssim b$ indicates that there is an absolute positive constant $c$ s.t. $a\leq c\cdot b$.
\end{lemma}
\begin{proof}
    Recall that, by the first inequality in~\Cref{prop:tv_dist_leq_est_err}, we have:
    \begin{equation*}
        \tv{\probb{\theta^\star}{\pi^k}(\tau_H), \probb{\thetahatk}{\pi^k}(\tau_H)}\leq \sum_{h=1}^H \sum_{\tau_{H}}\abs{\hat{m}^k(\omega_h)^\top \paren{\hat{M}^k_{h}(x_{h}) - M^{\star}_{h}(x_{h})}\psi^{\star}(\tau_{h-1})}\pi^k(\tau_H)
    \end{equation*}
    This is very similar to the inequality in~\Cref{lemma:ucb_tv}, with the difference being that the quantities associated with the estimated model and the true model are exchanged. Since both correspond to a PSR, the analysis follows a similar series of steps. We will use analogous notation to~\Cref{lemma:ucb_tv}. We index the future trajectory $\omega_{h-1} = (x_h, \ldots, x_H)$ by $i$ and history trajectory $\tau_{h-1} = (x_1, \ldots, x_{h-1})$ by $j$. We denote $\hat{m}^k(\omega_h)^\top\paren{\hat{M}^k_{h}(x_{h}) - M^{\star}_{h}(x_{h})}$ as $w_i$, $\bar{\psi}^{\star}(\tau_{h-1})$ as $x_j$, and $\pi(\omega_{h-1}|\tau_{h-1})$ as $\pi_{i|j}$.

    Define the matrix,
    \begin{equation*}
        \Lambda_h^k = \lambda_0 I + \sum_{t < k}\Expect_{j\sim\probb{\theta^{\star}}{\pi^{t}}}\bra{x_j x_j^\top}
    \end{equation*}
    where $\lambda_0$ is a constant to be determined later.

    For any policy $\pi$, using a similar calculation as in~\Cref{lemma:ucb_tv}, we have,
    \begin{align*}
    &\sum_{\tau_{H}}\abs{\hat{m}^k(\omega_h)^\top \paren{\hat{M}^k_{h}(x_{h}) - M^{\star}_{h}(x_{h})}\psi^{\star}(\tau_{h-1})}\pi^k(\tau_H)\\
    &=\Expect_{j\sim\probb{\theta^{\star}}{\pi^{k}}}
    \bra{\sum_{i}\pi_{i|j}\abs{w_i^\top x_j}}\\
    &=\Expect_{j\sim\probb{\theta^{\star}}{\pi^{k}}}\bra{\paren{\sum_{i}\pi_{i|j}\sign(w_i^\top x_j)w_i}^\top x_j}\\
    &\leq \Expect_{j\sim\probb{\theta^{\star}}{\pi^{k}}}\bra{\norm{x_j}_{\Lambda_h^{\dagger}}\norm{\sum_{i}\pi_{i|j}\sign(w_i^\top x_j)w_i}_{\Lambda_h}}
    \end{align*}
    where the last line is the Cauchy-Schwarz inequality.

    Fix $j = j_0$ and consider the term: $\norm{\sum_{i}\pi_{i|j_0}\sign(w_i^\top x_{j_0})w_i}_{\Lambda_h}$ in the above. This term can be partitioned in the same manner as in~\Cref{lemma:ucb_tv} by simply using the definition of $\Lambda_h$ and expanding,
    \begin{align*}
    &\norm{\sum_{i}\pi_{i|j_0}\sign(w_i^\top x_{j_0})w_i}_{\Lambda_h}^2\\
    &= \underbrace{\lambda_0 \norm{\sum_{i}\pi_{i|j_0}\cdot \sign (w_i^\top x_{j_0})\cdot w_i}^2_2}_{I_1} +
        \underbrace{\sum_{t<k}\Expect_{j\sim\probb{\theta^{\star}}{\pi^{k}}}\bra{\paren{\sum_{i}\pi_{i|j_0}\cdot \sign (w_i^\top x_{j_0})\cdot w_i^\top x_j}^2}}_{I_2}.
    \end{align*}

    We bound each term separately. The process is nearly identical to the proof of~\Cref{lemma:ucb_tv}, but we show it for completeness.

    $\sqrt{I_1}$ is bounded by the sum of two terms,
    \begin{align*}
        \sqrt{I_1} &= \sqrt{\lambda_0} \max_{\substack{z\in \reals^{d_{h-1}}: \twonorm{z}= 1}} \abs{\sum_{i}\pi_{i|j_0}\cdot \sign (w_i^\top x_{j_0})\cdot w_i^\top z}\\
        &\stepa{\leq} \sqrt{\lambda_0} \max_{\substack{\twonorm{z}= 1}}\sum_{\omega_{h-1}}\Bigg\lvert \hat{m}^k (\omega_h)^\top\Big(\hat{M}^k_{h}(x_{h}) - M^{\star}_{h}(x_{h})\Big)z\Bigg\rvert\pi(\omega_{h-1}|j_0)\\
        &\stepb{\leq} \sqrt{\lambda_0}\max_{\substack{\twonorm{z}= 1}}
        \sum_{\omega_{h-1}}\abs{\hat{m}^k (\omega_h)^\top  \hat{M}_{h}^k(x_{h})z}\pi(\omega_{h-1}|j_0)\\
        &\ \ \ + \sqrt{\lambda_0}\max_{\substack{\twonorm{z}= 1}}
        \sum_{\omega_{h-1}}\abs{\hat{m}^k (\omega_h)^\top  M_{h}^\star(x_{h})z}\pi(\omega_{h-1}|j_0),
    \end{align*}
    where step (a) is the definition of $w_i, \pi_{i|j_0}$, and the triangle inequality, and step (b) is the triangle inequality.

    Both terms can be bounded by the $\gamma$-well-conditioning assumption on $\thetahatk$ and $\theta^*$. Consider the first term,
    \begin{align*}
        \max_{\substack{\twonorm{z}= 1}}
            &\sum_{\omega_{h-1}}\abs{\hat{m}^k (\omega_h)^\top  \hat{M}_{h}^k(x_{h})z}\pi(\omega_{h-1}|j_0) \\
            &= \max_{\substack{\twonorm{z}= 1}} \sum_{x_h} \paren{\sum_{\omega_{h}}\abs{\hat{m}^k (\omega_h)^\top  \hat{M}_{h}^k(x_{h})z}\pi(\omega_{h}|j_0, x_h)} \pi(x_h | j_0) \\
            &\stepa{\leq} \max_{\substack{\twonorm{z}= 1}} \sum_{x_h} \frac{1}{\gamma} \onenorm{\hat{M}_{h}^k(x_{h})z} \pi(x_h | j_0) \\
            &\stepb{\leq} \frac{1}{\gamma} \max_{\substack{\twonorm{z}= 1}} \frac{\abs{\bbQ_{h+1}^A}\onenorm{z}}{\gamma} \\
            &\stepc{\leq} \frac{\sqrt{d} Q_A}{\gamma^2}
    \end{align*}
    where step (a) is by the first condition in~\Cref{ass:psr_gamma_wellcond}, step (b) is by the second condition of~\Cref{ass:psr_gamma_wellcond}, and step (c) is by the fact that $\max_{z \in \reals^{d_{h-1}}: \twonorm{z}=1} \onenorm{z} = \sqrt{d_{h-1}} \leq \sqrt{d}$ and $\abs{\bbQ_{h+1}^A} \leq Q_A$. In the above, note that we used the $\gamma$-well-conditioning of PSR $\hat{\theta}^k$ in both step (a) and step (b). The second term in $\sqrt{I_1}$ admits an identical bound, simply by using the well-conditioning of the PSR $\thetahatk$ in the first step and $\theta^*$ in the second step. Hence, we have that
    \begin{equation}
        I_1 \leq 4 \frac{\lambda_0 d Q_A^2}{\gamma^4}.
    \end{equation}

    Now, we consider the term $I_2$
    \begin{align*}
        I_2 &\leq \sum_{t<k}\Expect_{\tau_{h-1} \sim\probb{\theta^{\star}}{\pi^{k}}}\bra{\paren{\sum_{\omega_{h-1}} \abs{\hat{m}^k(\omega_h)^\top \paren{\hat{M}_{h}^k(x_{h}) - M^{\star}_{h}(x_{h})}\bar{\psi}^*(\tau_{h-1})}\pi(\omega_{h-1}|j_0)}^2}\\
        &\leq \sum_{t<k}\Expect_{j\sim\probb{\theta^{\star}}{\pi^{k}}} \Biggl[ \Biggl( \underbrace{\sum_{\omega_{h-1}}\abs{ \hat{m}^k(\omega_h)^\top \hat{M}_{h}(x_{h})\paren{\bar{\psi}^\star(\tau_{h-1}) - \hatbar{\psi}^k(\tau_{h-1})}} \pi(\omega_{h-1}|j_0)}_{I_3}\\
        &\ \ + \underbrace{\sum_{\omega_{h-1}} \abs{ \hat{m}^k(\omega_h)^\top \paren{\hat{M}_{h}^k(x_{h})\hatbar{\psi}^k(\tau_{h-1}) - M_{h}^\star(x_{h})\bar{\psi}^\star(\tau_{h-1})}} \pi(\omega_{h-1}|j_0)}_{I_4} \Biggr)^2 \Biggr]\\
        &=: \sum_{t<k}\Expect_{j\sim\probb{\theta^{\star}}{\pi^{k}}}(I_3 + I_4)^2
    \end{align*}
    where the line follows by the fact that $x \leq \aabs{x}$ and the line follows from the triangle inequality by adding and subtracting $\hat{m}_h(\omega_h) \hat{M}_h(x_h) \hatbar{\psi}^k(\tau_{h-1})$ inside the absolute value. We now bound each of $I_3$ and $I_4$.

    First, we bound $I_3$ as follows,
    \begin{align*}
        I_3 &=  \sum_{\omega_{h-1}} \abs{\hat{m}_h^k(\omega_h)^\top \hat{M}_{h}(x_{h})\paren{\bar{\psi}^\star(\tau_{h-1}) - \hatbar{\psi}^k(\tau_{h-1})}} \pi(\omega_{h-1}|j_0) \\
        &\stepa{=}  \sum_{\omega_{h-1}} \abs{\hat{m}_{h-1}^k(\omega_{h-1})^\top \paren{\bar{\psi}^\star(\tau_{h-1}) - \hatbar{\psi}^k(\tau_{h-1})}} \pi(\omega_{h-1}|j_0) \\
        &\stepb{\leq} \frac{1}{\gamma} \onenorm{\hatbar{\psi}^k(\tau_{h-1}) - \bar{\psi}^\star(\tau_{h-1})} \\
        &= \frac{1}{\gamma} \sum_{q_{h-1} \in \bbQ_{h-1}} \abs{\probbarunder{\thetahatk}{q_{h-1} \given \tau_{h-1}} - \probbarunder{\theta^*}{q_{h-1} \given \tau_{h-1}}},
    \end{align*}
    where step (a) is since $\hat{m}(\omega_h)^\top \hat{M}_h(x_h) = \hat{m}(\omega_{h-1})^\top$, step (b) is by~\Cref{ass:psr_gamma_wellcond}, and the final equality is by the definition of $\bar{\psi}$.

    \begin{align*}
        I_4 &= \sum_{\omega_{h-1}} \abs{ \hat{m}^k(\omega_h)^\top \paren{\hat{M}_{h}^k(x_{h})\hatbar{\psi}^k(\tau_{h-1}) - M_{h}^\star(x_{h})\bar{\psi}^\star(\tau_{h-1})}} \pi(\omega_{h-1}|j_0) \\
        &\stepa{=} \sum_{\omega_{h-1}} \abs{ \hat{m}^k(\omega_h)^\top \paren{\probbarunder{\thetahatk}{x_h \given \tau_{h-1}} \hatbar{\psi}_h(\tau_h) - \probbarunder{\theta^*}{x_h \given \tau_{h-1}} \bar{\psi}_h^*(\tau_h)}} \pi(\omega_{h-1}|j_0) \\
        &= \sum_{x_h} \paren{\sum_{\omega_{h}} \abs{ \hat{m}^k(\omega_h)^\top \paren{\probbarunder{\thetahatk}{x_h \given \tau_{h-1}} \hatbar{\psi}_h(\tau_h) - \probbarunder{\theta^*}{x_h \given \tau_{h-1}} \bar{\psi}_h^*(\tau_h)}} \pi(\omega_{h}|j_0, x_h)} \pi(x_h | j_0) \\
        &\stepb{\leq} \frac{1}{\gamma} \sum_{x_h} \onenorm{\probbarunder{\thetahatk}{x_h \given \tau_{h-1}} \hatbar{\psi}_h(\tau_h) - \probbarunder{\theta^*}{x_h \given \tau_{h-1}} \bar{\psi}_h^*(\tau_h)} \pi(x_h | j_0) \\
        &\stepc{=} \frac{1}{\gamma} \sum_{x_h} \sum_{q_h \in \bbQ_h} \abs{\probbarunder{\thetahatk}{x_h, q_h \given \tau_{h-1}} - \probbarunder{\theta^*}{x_h, q_h \given \tau_{h-1}}} \pi(x_h | j_0)
    \end{align*}
    where step (a) is by the fact that $M_h(x_h) \bar{\psi}_{h-1}(\tau_{h-1}) = \probbar{x_h \given \tau_{h-1}} \bar{\psi}(\tau_h)$, as shown in~\Cref{eq:Mpsi_psi_relation}, step (b) is by~\Cref{ass:psr_gamma_wellcond}, and step (c) is since $\bra{\bar{\psi}_h(\tau_h)}_l = \probbarunder{\theta}{q_h^l \given \tau_h}$.

    Combining the above, we have that,
    \begin{align*}
        I_2 &\leq \sum_{t<k}\Expect_{j\sim\probb{\theta^{\star}}{\pi^{k}}}\paren{I_3 + I_4}^2 \\
        &\leq \begin{aligned}[t]
            \sum_{t<k}\Expect_{j\sim\probb{\theta^{\star}}{\pi^{k}}} \Biggl[\Biggl(
                & \frac{1}{\gamma} \sum_{q_{h-1} \in \bbQ_{h-1}} \abs{\probbarunder{\thetahatk}{q_{h-1} \given \tau_{h-1}} - \probbarunder{\theta^*}{q_{h-1} \given \tau_{h-1}}} \\
            +& \frac{1}{\gamma} \sum_{x_h} \sum_{q_h \in \bbQ_h} \abs{\probbarunder{\thetahatk}{x_h, q_h \given \tau_{h-1}} - \probbarunder{\theta^*}{x_h, q_h \given \tau_{h-1}}} \pi(x_h | j_0)\Biggr)^2\Biggr]
        \end{aligned}\\
        &=\begin{aligned}[t]
            \frac{1}{\gamma^2} \cdot \sum_{t<k}\Expect_{j\sim\probb{\theta^{\star}}{\pi^{k}}} \Biggl[\Biggl(&\sum_{q_{h-1} \in \bbQ_{h-1}} \abs{\probbarunder{\thetahatk}{q_{h-1} \given \tau_{h-1}} - \probbarunder{\theta^*}{q_{h-1} \given \tau_{h-1}}} \\
            +&\frac{1}{\gamma} \sum_{x_h} \sum_{q_h \in \bbQ_h} \abs{\probbarunder{\thetahatk}{x_h, q_h \given \tau_{h-1}} - \probbarunder{\theta^*}{x_h, q_h \given \tau_{h-1}}} \pi(x_h | j_0) &\Biggr)^2 \Biggr]
        \end{aligned}\\
        &\stepa{\leq}\frac{1}{\gamma^2} \cdot \sum_{t<k}\Expect_{j\sim\probb{\theta^{\star}}{\pi^{k}}}\paren{\sum_{\omega_{h-1}^a \in \bbQ_{h-1}^{\mathtt{\exp}}} \sum_{\omega_{h-1}^o}  \abs{\probbarunder{\thetahatk}{\omega_{h-1}^o \given \tau_{h-1}, \omega_{h-1}^a} - \probbarunder{\theta^*}{\omega_{h-1}^o \given \tau_{h-1}, \omega_h^a}}}^2 \\
        &=\frac{\abs{\bbQ_{h-1}^\mathrm{\exp}}^2}{\gamma^2} \cdot \sum_{t<k}\Expect_{j\sim\probb{\theta^{\star}}{\pi^{k}}}  \bra{\tvsq{\probb{\thetahatk}{\uhexp{h-1}}\paren{\omega_{h-1}^o \given \tau_{h-1}, \omega_{h-1}^a}, \probb{\theta^*}{\uhexp{h-1}}\paren{\omega_{h-1}^o \given \tau_{h-1}, \omega_h^a}}}\\
        &\stepb{\leq} \frac{4 \max_{s \in \calA} \abs{\bbX_s}^2 Q_A^2}{\gamma^2} \cdot \sum_{t<k}\mathtt{D}_{\mathtt{H}}^{2}\paren{\probb{\thetahatk}{\nu_h(\pi^t, \uhexp{h-1})}\paren{\tau_H}, \probb{\theta^*}{\nu_h(\pi^t, \uhexp{h-1})}\paren{\tau_H}},
    \end{align*}
    where step (a) follows from the definition of $\bbQ_{h-1}^{\mathrm{\exp}}$ (same as~\Cref{lemma:ucb_tv}), and step (b) is because the Hellinger distance bounds the total variation distance and since $\abs{\Qhexp{h-1}} \leq 2 \max_{s \in \calA} \abs{\bbX_s} Q_A$. Hence, we have,
    \begin{equation*}
        I_2 \leq 4 \max_{s\in \calA}\abs{\bbX_s}^2 Q_A^2 \frac{1}{\gamma^2} \sum_{t<k}\hellingersq{\probb{\thetahatk}{\nu_h(\pi^t, u_{\Qhexp{h-1}})}\paren{\tau_H}, \probb{\theta^*}{\nu_h(\pi^t, u_{\Qhexp{h-1}})}\paren{\tau_H}}. 
    \end{equation*}

    Now, combining the bound on $I_1$ and $I_2$ allows us to finally bound $\norm{\sum_{i}\pi_{i|j}\sign(w_i^\top x_j)w_i}_{\Lambda_h}^2$ as follows,
    \begin{align*}
        &\norm{\sum_{i}\pi_{i|j}\sign(w_i^\top x_j)w_i}_{\Lambda_h}^2 \\
        &\leq \frac{4\lambda_0 Q_A^2 d}{\gamma^4} + \frac{4 \max_{s \in \calA} \abs{\bbX_s}^2 Q_A^2}{\gamma^2} \cdot \sum_{t<k}\mathtt{D}_{\mathtt{H}}^{2}\paren{\probb{\thetahatk}{\nu_h(\pi^t, \uhexp{h-1})}\paren{\tau_H}, \probb{\theta^*}{\nu_h(\pi^t, \uhexp{h-1})}\paren{\tau_H}}\\
        &=: \paren{\tilde{\alpha}_{h-1}^k}^2,
    \end{align*}
    We choose $\lambda_0 = \frac{\gamma^4}{4Q_A^2d}$, and bound $\tilde{\alpha}^2 \coloneq \sum_h\paren{\tilde{\alpha}_{h-1}^k}^2$ as follows,
    \begin{align*}
        \sum_h\paren{\tilde{\alpha}_{h-1}^k}^2 &= H + \frac{4 \max_{s\in\calA}\abs{\bbX_s}^2 Q_A^2 \beta}{\gamma^2} \sum_{\pi \in \calD^k} \mathtt{D}_{\mathtt{H}}^{2}\paren{\probb{\thetahatk}{\nu_h(\pi^t, \uhexp{h-1})}\paren{\tau_H}, \probb{\theta^*}{\nu_h(\pi^t, \uhexp{h-1})}\paren{\tau_H}} \\
        &\leq H + \frac{28 \max_{s\in\calA}\abs{\bbX_s}^2 Q_A^2 \beta}{\gamma^2} \\
        &\lesssim \frac{\max_{s\in\calA}\abs{\bbX_s}^2 Q_A^2 \beta}{\gamma^2},
    \end{align*}
    where the second line is by the estimation guarantee of~\Cref{lemma:estimation_guarantee}.

    Thus, we have,
    \begin{align*}
        \tv{\probb{\theta^\star}{\pi^k}(\tau_H), \probb{\thetahatk}{\pi^k}(\tau_H)} &\leq \sum_{h=1}^{H} \sum_{\tau_{H}}\abs{\hat{m}^k(\omega_h)^\top \paren{\hat{M}^k_{h}(x_{h}) - M^{\star}_{h}(x_{h})}\psi^{\star}(\tau_{h-1})}\pi^k(\tau_H)\\
        &\leq \sum_{h=1}^{H} \Expect_{\tau_{h-1} \sim\probb{\theta^{\star}}{\pi^{k}}}\bra{\norm{\bar{\psi}^*(\tau_{h-1})}_{\Lambda_h^{\dagger}}\norm{\sum_{i}\pi_{i|j}\sign(w_i^\top x_j)w_i}_{\Lambda_h}} \\
        &\leq \Expect_{\tau_{h-1} \sim\probb{\theta^{\star}}{\pi^{k}}}\bra{\sum_{h=1}^{H} \norm{\bar{\psi}^*(\tau_{h-1})}_{\Lambda_h^{\dagger}}\norm{\sum_{i}\pi_{i|j}\sign(w_i^\top x_j)w_i}_{\Lambda_h}} \\
        &\stepa{\leq} \Expect_{\tau_{h-1} \sim\probb{\theta^{\star}}{\pi^{k}}}\bra{\sqrt{\sum_{h=1}^{H} \norm{\bar{\psi}^*(\tau_{h-1})}_{\Lambda_h^{\dagger}}^2} \sqrt{ \sum_{h=1}^H \norm{\sum_{i}\pi_{i|j}\sign(w_i^\top x_j)w_i}_{\Lambda_h}^2}} \\
        &\stepb{\leq} \tilde{\alpha} \cdot \Expect_{\tau_{h-1} \sim\probb{\theta^{\star}}{\pi^{k}}}\bra{\sqrt{\sum_{h=1}^{H} \norm{\bar{\psi}^*(\tau_{h-1})}_{\Lambda_h^{\dagger}}^2}} \\
        &\leq \tilde{\alpha} \cdot \sqrt{\sum_{h=1}^{H} \Expect_{\tau_{h-1} \sim\probb{\theta^{\star}}{\pi^{k}}}\bra{ \norm{\bar{\psi}^*(\tau_{h-1})}_{\Lambda_h^{\dagger}}^2}},
    \end{align*}
    where step (a) is by the Cauchy-Schwarz inequality and step (b) is by the bound established above. Since the total variation distance is bounded above by 2, we have
    \begin{equation*}
        \tv{\probb{\theta^\star}{\pi^k}(\tau_H), \probb{\thetahatk}{\pi^k}(\tau_H)} \leq \min\set{\tilde{\alpha} \cdot \sqrt{\sum_{h=1}^{H} \Expect_{\tau_{h-1} \sim\probb{\theta^{\star}}{\pi^{k}}}\bra{ \norm{\bar{\psi}^*(\tau_{h-1})}_{\Lambda_h^{\dagger}}^2}}, 2}.
    \end{equation*}

    Finally, the proof is completed by summing over $k$ using the elliptical potential lemma as follows,
    \begin{align*}
        \sum_{k=1}^{K} \tv{\probb{\theta^\star}{\pi^k}(\tau_H), \probb{\thetahatk}{\pi^k}(\tau_H)} &\leq \sum_{k=1}^K \min\set{\tilde{\alpha} \cdot \sqrt{\sum_{h=1}^{H} \Expect_{\tau_{h-1} \sim\probb{\theta^{\star}}{\pi^{k}}}\bra{ \norm{\bar{\psi}^*(\tau_{h-1})}_{\Lambda_h^{\dagger}}^2}}, 2}\\
        &\stepa{\leq} \sqrt{K} \sqrt{\sum_{k=1}^K \sum_{h=1}^{H} \min\set{\tilde{\alpha}^2 \cdot \Expect_{\tau_{h-1} \sim\probb{\theta^{\star}}{\pi^{k}}} \bra{ \norm{\bar{\psi}^*(\tau_{h-1})}_{\Lambda_h^{\dagger}}^2}, 4}}\\
        &\leq \sqrt{K} \tilde{\alpha} \sqrt{\sum_{k=1}^K \sum_{h=1}^{H} \min\set{\Expect_{\tau_{h-1} \sim\probb{\theta^{\star}}{\pi^{k}}} \bra{ \norm{\bar{\psi}^*(\tau_{h-1})}_{\Lambda_h^{\dagger}}^2}, 4/\tilde{\alpha}^2}}\\
        &\stepb{\leq} \sqrt{K H} \tilde{\alpha} \sqrt{(1 + 4/\tilde{\alpha}^2) r \log(1 + K / \lambda_0)} \\
        &\stepc{\lesssim} \frac{\max_{s\in\calA} \abs{\bbX_s} Q_A}{\gamma} \sqrt{r K H \beta \log(1 + K / \lambda_0)} \\
        &\stepd{\lesssim}  \frac{\max_{s\in\calA} \abs{\bbX_s} Q_A}{\gamma} \sqrt{ r K H \beta \log(1 + d Q_A K / \gamma)}.
    \end{align*}
    Here, step (a) is uses the relationship between the $\ell_1$ and $\ell_2$ norms $\onenorm{\cdot} \leq \sqrt{d} \twonorm{\cdot}$. Step (b) is by the elliptical potential lemma (\cite[Lemma 14]{huangProvablyEfficientUCBtype2023}; see also~\cite{dani2008stochastic,abbasi2011improved,carpentier2020elliptical}). Step (c) uses the bound on $\tilde{\alpha}$ established above and the fact that $\sqrt{1 + 4/\tilde{\alpha}^2}$ is bounded by an absolute constant. Step (d) uses the definition of $\lambda_0$ and the fact that $\sqrt{28 (1 + 4/\tilde{\alpha}^2)}$ is bounded by an absolute constant.
\end{proof}

Using the two lemmas above, we are now ready to show that $\sum_{k=1}^{K} {V_{\thetahatk}^{\hat{b}^k}(\pi^k)} = O(\sqrt{K})$. The argument is identical to~\parencite[Lemma 6]{huangProvablyEfficientUCBtype2023} and does not require modification for generalized PSRs. We recount the argument for completeness.

\begin{lemma}\label{lemma:sum_V_sublinear}
    Under the event $\calE$, with probability at least $1- \delta$, we have:
    \begin{align*}
        \sum_{k=1}^K V_{\thetahatk}^{\hat{b}^k}(\pi^k) \lesssim \paren{\sqrt{r} + \frac{Q_A \sqrt{H}}{\gamma}} \frac{\max_{s \in \calA}^2 Q_A^2 H \sqrt{drH\beta K \beta_0}}{\gamma^2}
    \end{align*}
    where $\beta_0 = \max\{\log(1 + K/\lambda), \log (1 + dQ_A K/\gamma)\}$, and $\lambda = \frac{\gamma\max_{s\in\calA}\abs{\bbX_s}Q_A\beta \max\{\sqrt{r}, Q_A\sqrt{H}/\gamma\}}{\sqrt{dH}}$
\end{lemma}
\begin{proof}
    First, we note that,
    \begin{equation*}
        \val{\thetahatk}{\bhatk}(\pi^k) = \sum_{\tau} \probb{\thetahatk}{\pi^k}(\tau) \bhatk(\tau) = \sum_{\tau} \probb{\theta^*}{\pi^k}(\tau) \bhatk(\tau) + \sum_{\tau} (\probb{\thetahatk}{\pi^k}(\tau) - \probb{\theta^*}{\pi^k}(\tau)) \bhatk(\tau)
        \leq \val{\theta^*}{\bhatk}(\pi^k) + \tv{\probb{\thetahatk}{\pi^k}, \probb{\theta^*}{\pi^k}},
    \end{equation*}
    where we recall that $\bhatk(\cdot) \in [0,1]$. Hence, we may focus on bounding the value of $\bhatk$ under the true model $\theta^*$ and use the bound on the cumulative total variation estimation error established in~\Cref{lemma:cumulative_tv_error_bound}.

    Recall the definition of the bonus term,
    \[\bhatk(\tau_H) \coloneq \min\set{\alpha \sqrt{\sum_h \norm{\hatbar{\psi}^k(\tau_h)}_{(\Uhathk)^{-1}}}, 1},\]
    which is defined in terms of the estimated prediction features $\hatbar{\psi}$. Recall also that in~\Cref{lemma:features_bound} we established a bound on the expectation of the prediction features under the true model, which corresponds to $\val{\theta^*}{\bhatk}$. Hence, we proceed to bound $\sum_{k} \val{\theta^*}{\bhatk}(\pi^k)$ as follows,
    \begin{align*}
        &\sum_{k} \val{\theta^*}{\bhatk}(\pi^k) \\
        &= \sum_{k} \Expect_{\tau_H \sim \probb{\theta^*}{\pi^k}}\bra{\bhatk(\tau_H)} \\
        &= \sum_{k} \Expect_{\tau_H \sim \probb{\theta^*}{\pi^k}}\bra{\min\set{\alpha \sqrt{\sum_h \norm{\hatbar{\psi}^k(\tau_h)}^2}, 1}} \\
        &\stepa{\leq} \sum_{k=1}^K \min\set{\alpha \paren{1 + \frac{2 \max_{s \in \calA}\aabs{\bbX_s} Q_A \sqrt{7 r \beta}}{\sqrt{\lambda}}} \sum_{h=0}^{H-1} \Expect_{{\tau_H \sim \probb{\theta^*}{\pi^k}}}\bra{\norm{\bar{\psi}^*(\tau_h)}_{(\Uhk)^{-1}}} + \sum_{k=1}^K \frac{\alpha H Q_A}{\sqrt{\lambda}} \tv{\probb{\theta^*}{\pi^k}, \probb{\thetahatk}{\pi^k}}, 1}\\
        &\stepb{\leq} \underbrace{\sum_{k=1}^K \min\set{\alpha \paren{1 + \frac{2 \max_{s \in \calA}\aabs{\bbX_s} Q_A \sqrt{7 r \beta}}{\sqrt{\lambda}}} \sum_{h=0}^{H-1} \Expect_{{\tau_H \sim \probb{\theta^*}{\pi^k}}}\bra{\norm{\bar{\psi}^*(\tau_h)}_{(U_h^k)^{-1}}}, 1}}_{I_1} + \sum_{k=1}^K \frac{\alpha H Q_A}{\sqrt{\lambda}} \tv{\probb{\theta^*}{\pi^k}, \probb{\thetahatk}{\pi^k}},
    \end{align*}
    where step (a) is by~\Cref{lemma:features_bound} and step (b) is since $\min(a+b, c) \leq \min(a, c) + b$ when $a, b, c$ are non-negative.

    Next, we bound the term $I_1$. Recall the definition of $\Uhk \coloneq \lambda I + \sum_{\tau_h \in \calD_h^k} \psibarstar(\tau_h) \psibarstar(\tau_h)^\top$. Also, note that the process
    \begin{equation*}
        \paren{\Expect_{\tau_h \sim \probb{\theta^*}{\pi^k}} \bra{\norm{\psibarstar(\tau_h)}_{(\Uhk)^{-1}}} - \norm{\psibarstar(\tau_{h}^{k+1, h+1})}_{(\Uhathk)^{-1}}}_{k=1}^{K}
    \end{equation*}
    is a martingale. Hence, by the Azuma-Hoeffding inequality, we have that with probability at least $1 - \delta$,
    \begin{align*}
        I_1 &\leq \sqrt{2 K \log(2/\delta)} + \sum_{k=1}^K \min\set{\alpha \paren{1 + \frac{2 \max_{s \in \calA}\aabs{\bbX_s} Q_A \sqrt{7 r \beta}}{\sqrt{\lambda}}} \sum_{h=0}^{H-1} \norm{\bar{\psi}^*(\tau_h^{k+1, h+1})}_{(\Uhk)^{-1}}, 1} \\
        &\lesssim \sqrt{2 K \log(2/\delta)} + \alpha \paren{1 + \frac{2 \max_{s \in \calA}\aabs{\bbX_s} Q_A \sqrt{7 r \beta}}{\sqrt{\lambda}}} H \sqrt{r K \log (1 + K / \lambda)} \\
        &\lesssim \alpha \paren{1 + \frac{\max_{s \in \calA}\aabs{\bbX_s} Q_A \sqrt{7 r \beta}}{\sqrt{\lambda}}} H \sqrt{r K \log (1 + K / \lambda)} 
    \end{align*}
    where the second line is by the Elliptical potential lemma (\cite[Lemma 14]{huangProvablyEfficientUCBtype2023}; see also~\cite{dani2008stochastic,abbasi2011improved,carpentier2020elliptical}).

    We now return to bounding $\sum_{k=1}^K V_{\thetahatk}^{\hat{b}^k}(\pi^k)$. For convenience we define $\beta_0 \coloneq \max\sset{\log(1 + K/\lambda), \log(1 + d Q_A K / \gamma)}$ and we choose $\lambda$ as follows,
    \begin{equation*}
        \lambda = \frac{\gamma \max_{s \in \calA}\abs{\bbX_s}^2 Q_A \beta \max\sset{\sqrt{r}, Q_A \sqrt{H} / \gamma}}{\sqrt{d H}}.
    \end{equation*}

    We have,
    \begin{align*}
        \sum_{k=1}^K V_{\thetahatk}^{\hat{b}^k}(\pi^k) &\leq \sum_{k=1}^K \val{\theta^*}{\bhatk}(\pi^k) + \sum_{k=1}^K \tv{\probb{\thetahatk}{\pi^k}, \probb{\theta^*}{\pi^k}} \\
        &\leq I_1 + \paren{1 + \frac{\alpha H Q_A}{\sqrt{\lambda}}} \sum_{k=1}^K \tv{\probb{\theta^*}{\pi^k}, \probb{\thetahatk}{\pi^k}} \\
        &\stepa{\lesssim} \alpha \paren{1 + \frac{\max_{s \in \calA}\aabs{\bbX_s} Q_A \sqrt{7 r \beta}}{\sqrt{\lambda}}} H \sqrt{r K \beta_0} + \frac{\alpha H}{\sqrt{\lambda}} \frac{Q_A^2 \max_{s \in \calA}\aabs{\bbX_s} \sqrt{\beta}}{\gamma} \sqrt{r H K \beta_0} \\
        &= \alpha \paren{1 + \frac{\max_{s \in \calA}\aabs{\bbX_s} Q_A \sqrt{7 r \beta}}{\sqrt{\lambda}} + \frac{\max_{s \in \calA}\aabs{\bbX_s} Q_A^2 \sqrt{H \beta}}{\gamma \sqrt{\lambda}}} H \sqrt{r K \beta_0} \\
        &\stepb{\lesssim} \paren{\frac{Q_A \sqrt{H d \lambda}}{\gamma^2} + \frac{\max_{s \in \calA} \aabs{\bbX_s} Q_A \sqrt{\beta}}{\gamma}} \paren{1 + \frac{\max_{s \in \calA}\aabs{\bbX_s} Q_A \sqrt{\beta} \max\sset{\sqrt{r}, Q_A \sqrt{H}/\gamma}}{\sqrt{\lambda}}} H \sqrt{r K \beta_0} \\
        &= \paren{1 + \frac{\sqrt{d H \lambda}}{\max_{s \in \calA}\aabs{\bbX_s} \sqrt{\beta} \gamma}} \paren{1 + \frac{\max_{s \in \calA}\aabs{\bbX_s} Q_A \sqrt{\beta} \max\sset{\sqrt{r}, Q_A \sqrt{H}/\gamma}}{\sqrt{\lambda}}} \frac{\max_{s \in \calA}{\aabs{\bbX_s} Q_A H \sqrt{r K \beta \beta_0}}}{\gamma}\\
        &\stepc{=} \paren{1 + \sqrt{\frac{Q_A \max\sset{\sqrt{r}, Q_A \sqrt{H}/\gamma} \sqrt{dH}}{\gamma}}}^2 \frac{\max_{s \in \calA}{\aabs{\bbX_s} Q_A H \sqrt{r K \beta \beta_0}}}{\gamma}\\
        &\lesssim \paren{1 + \frac{Q_A \sqrt{dH} \max\set{\sqrt{r}, Q_A \sqrt{H} / \gamma}}{\gamma}} \frac{\max_{s \in \calA}{\aabs{\bbX_s} Q_A H \sqrt{r K \beta \beta_0}}}{\gamma} \\
        &\leq \paren{\sqrt{r} + \frac{Q_A \sqrt{H}}{\gamma}} \frac{\max_{s \in \calA}{\aabs{\bbX_s} Q_A^2 H \sqrt{r d H K \beta \beta_0}}}{\gamma^2}, 
    \end{align*}
    where step (a) is by~\Cref{lemma:cumulative_tv_error_bound} and the bound on $I_1$ established above, step (b) uses the definition of $\alpha$ and the fact that $\alpha \lesssim \frac{Q_A \sqrt{H d \lambda}}{\gamma^2} + \frac{\max_{s \in \calA} \aabs{\bbX_s} Q_A \sqrt{\beta}}{\gamma}$, and step (c) is by plugging in the choice of $\lambda$.
\end{proof}

\subsection{Proof of Theorem~\ref{theorem:post_ucb}}\label{ssec:team_thm_proof}

\begin{theorem*}[Restatement of~\Cref{theorem:post_ucb}]
    Suppose~\Cref{ass:psr_gamma_wellcond} holds. Let $p_{\min} = O\paren{\frac{\delta}{K H \prod_{h=1}^H \abs{\bbX_{h}}}}$, $\lambda = \frac{\gamma \max_{s \in \calA} \abs{\bbX_{s}}^2 Q_A \beta \max\set{\sqrt{r}, Q_A \sqrt{H} / \gamma}}{\sqrt{d H}}$, $\alpha = O\paren{\frac{Q_A \sqrt{H d}}{\gamma^2} \sqrt{\lambda} + \frac{\max_{s \in \calA} \abs{\bbX_{s}} Q_A \sqrt{\beta} }{\gamma}}$, and let $\beta = O(\log \abs{\bar{\Theta}_\varepsilon})$, where $\varepsilon = O(\frac{p_{\min}}{K H})$. Then, with probability at least $1 - \delta$,~\Cref{alg:post_ucb} returns a model $\theta^\epsilon$ and a policy $\pi$ that satisfy
    \begin{equation*}
        V_{\theta^\epsilon}^{R}(\pi^*) - V_{\theta^\epsilon}^{R}(\pi) \leq \varepsilon, \ \text{and} \ \forall \pi, \ \tv{\probb{\theta^\epsilon}{\pi}(\tau_H), \probb{\theta^*}{\pi}(\tau_H)} \leq \varepsilon.
    \end{equation*}
    In addition, the algorithm terminates with a sample complexity of,
    \begin{equation*}
        \tilde{O}\paren{\paren{r + \frac{Q_A^2 H}{\gamma^2}} \frac{r d H^3 \max_{s \in \calA} \abs{\bbX_{s}}^2 Q_A^4 \beta}{\gamma^4 \epsilon^2}}.
    \end{equation*}
\end{theorem*}

\begin{proof}
    By~\Cref{prop:mle_E_min,prop:mle_E_omega,prop:mle_E_pi}, the event $\calE$ occurs with high probability, $\prob{\calE} \geq 1 - 3 \delta$. Suppose $\calE$ holds. Then, by the upper confidence bound established in~\Cref{cor:upper_conf_bound}, if~\Cref{alg:post_ucb} terminates, then the following must hold,
    \begin{equation*}
        \forall \pi, \ \tv{\probb{\theta^\epsilon}{\pi}(\tau_H), \probb{\theta^*}{\pi}(\tau_H)} = 2 \max_{R} \abs{V_{\theta^\epsilon}^{R}(\pi) - V_{\theta^*}^{R}(\pi)} \leq V_{\theta^\epsilon}^{\hat{b}^\epsilon}(\pi) \leq \epsilon,
    \end{equation*}
    where the maximization is over reward functions $R: \bbH_H \to [0,1]$. The last inequality is simply the termination condition of~\Cref{alg:post_ucb}.

    Now, the difference between the optimal value and the value of $\pi$ (the policy returned by the algorithm) can be bounded as follows,
    \begin{align*}
        V_{\theta^*}^{R}(\pi^*) - V_{\theta^*}^{R}(\pi) &= V_{\theta^*}^{R}(\pi^*) - V_{\theta^\epsilon}^{R}(\pi^*) + V_{\theta^\epsilon}^{R}(\pi^*) - V_{\theta^\epsilon}^{R}(\pi) + V_{\theta^\epsilon}^{R}(\pi) - V_{\theta^*}^{R}(\pi)\\
        &\leq 2 \max_{\pi} V_{\theta^\epsilon}^{\hat{b}^\epsilon}(\pi) \leq \epsilon,
    \end{align*}
    where the inequality follows from the fact that $\pi = \argmax_{\pi} V_{\theta^\epsilon}^{R}(\pi)$ and by~\Cref{cor:upper_conf_bound}.

    Recall that by~\Cref{lemma:sum_V_sublinear}, we have,
    \begin{align*}
        \sum_{k=1}^KV_{\thetahatk, \bar{b}^k}^{\pi^k} \lesssim \paren{\sqrt{r} + \frac{Q_A \sqrt{H}}{\gamma}} \frac{\max_{s\in\calA}\abs{\bbX_s} Q_A^2 H \sqrt{r d H K \beta \beta_0}}{\gamma^2}.
    \end{align*}

    By the pigeon-hole principle and the termination condition of~\Cref{alg:post_ucb}, the algorithm must terminate within
    \begin{equation*}
        K = \tilde{O}\paren{\paren{r + \frac{Q_A^2 H}{\gamma^2}} \frac{r d H^2 Q_A^4 \max_{s\in\calA}\abs{\bbX_s}^2 \beta}{\gamma^4 \epsilon^2}}
    \end{equation*}
    episodes. Since each episode contains $H$ iterations, this implies a sample complexity of
    \begin{equation*}
        K = \tilde{O}\paren{\paren{r + \frac{Q_A^2 H}{\gamma^2}} \frac{r d H^3 Q_A^4 \max_{s\in\calA}\abs{\bbX_s}^2 \beta}{\gamma^4 \epsilon^2}}.
    \end{equation*}
    Therefore, we conclude the proof of Theorem \ref{theorem:post_ucb}. 
\end{proof}

\section{Proof of Theorem~\ref{theorem:posg_ucb}: UCB Algorithm for Generalized PSRs (Game Setting)}\label{sec:game_thm_proof}

\begin{theorem*}[Restatement of~\Cref{theorem:posg_ucb}]
    Suppose~\Cref{ass:psr_gamma_wellcond} holds. Let $p_{\min} = O\paren{\frac{\delta}{K H \prod_{h=1}^H \abs{\bbX_{h}}}}$, $\lambda = \frac{\gamma \max_{s \in \calA} \abs{\bbX_{s}}^2 Q_A \beta \max\set{\sqrt{r}, Q_A \sqrt{H} / \gamma}}{\sqrt{d H}}$, $\alpha = O\paren{\frac{Q_A \sqrt{H d}}{\gamma^2} \sqrt{\lambda} + \frac{\max_{s \in \calA} \abs{\bbX_{s}} Q_A \sqrt{\beta} }{\gamma}}$, and let $\beta = O(\log \abs{\bar{\Theta}_\varepsilon})$, where $\varepsilon = O(\frac{p_{\min}}{K H})$. Then, with probability at least $1 - \delta$,~\Cref{alg:posg_ucb} returns a model $\theta^\epsilon$ and a policy $\pi$ which is an $\varepsilon$-approximate equilibrium (either NE or CCE). That is,
    \begin{equation*}
        V_{\theta^*}^i(\pi) \geq V_{\theta^*}^{i, \dagger}(\pi^{-i}) - \varepsilon, \, \forall i \in [N].
    \end{equation*}
    In addition, the algorithm terminates with a sample complexity of,
    \begin{equation*}
        \tilde{O}\paren{\paren{r + \frac{Q_A^2 H}{\gamma^2}} \frac{r d H^3 \max_{s \in \calA} \abs{\bbX_{s}}^2 Q_A^4 \beta}{\gamma^4 \epsilon^2}}.
    \end{equation*}
\end{theorem*}

\begin{proof}
    Recall that the model-estimation portion of~\Cref{alg:posg_ucb} is identical to~\Cref{alg:post_ucb}. Hence, by~\Cref{theorem:post_ucb}, the returned estimated model $\theta^\varepsilon$ satisfies,
    \begin{equation*}
        \tv{\probb{\theta^\varepsilon}{\bm{\pi}}(\tau_H), \probb{\theta^*}{\bm{\pi}}(\tau_H)} \leq \varepsilon / 2,
    \end{equation*}
    for any collection of policies $\bm{\pi} = (\pi^i : i \in [N])$. This implies that $V_{\theta^*}^i(\pi) \geq V_{\theta^\varepsilon}^i(\pi) - \varepsilon / 2$ for all $i \in [N]$.

    Let $\Gamma^i = \Gamma_{\mathrm{ind}}^i$ in the case of running the algorithm to find a Nash equilibrium and $\Gamma^i = \Gamma_{\mathrm{cor}}^i$ in the case of a coarse correlated equilibrium. Recall that the collection of policies $\pi = (\pi^1, \ldots, \pi^N)$ returned by the algorithm are an equilibrium under $\theta^{\varepsilon}$. That is, for all $i \in [N]$,
    \begin{equation*}
        V_{\theta^\varepsilon}^i(\pi) = \max_{\tilde{\pi}^i \in \Gamma_{\mathrm{ind}}^i} V_{\theta^\varepsilon}^i(\tilde{\pi}^i, \pi^{-i}) =: V_{\theta^\varepsilon}^{i, \dagger}(\pi^{-i}).
    \end{equation*}

    Moreover, note that,
    \begin{align*}
        \abs{V_{\theta^\varepsilon}^{i, \dagger}(\pi^{-i}) - V_{\theta^*}^{i, \dagger}(\pi^{-i})} &= \abs{\max_{\tilde{\pi}^i} V_{\theta^\varepsilon}^i(\tilde{\pi}^i, \pi^{-i}) - \max_{\tilde{\pi}^i} V_{\theta^*}^i(\tilde{\pi}^i, \pi^{-i})} \\
        &\leq \max_{\tilde{\pi}^i} \abs{V_{\theta^\varepsilon}^i(\tilde{\pi}^i, \pi^{-i}) - V_{\theta^*}^i(\tilde{\pi}^i, \pi^{-i})} \\
        &\leq \varepsilon / 2,
    \end{align*}
    where the final inequality is since $\tv{\probb{\theta^\varepsilon}{\bm{\pi}}(\tau_H), \probb{\theta^*}{\bm{\pi}}(\tau_H)} \leq \varepsilon / 2$ for any $\bm{\pi}$. Thus, $V_{\theta^\varepsilon}^{i, \dagger}(\pi^{-i}) \geq V_{\theta^*}^{i, \dagger}(\pi^{-i}) - \varepsilon / 2$.

    Putting this together, we have,
    \begin{align*}
        V_{\theta^*}^i(\pi) &\geq V_{\theta^\varepsilon}^i(\pi) - \varepsilon / 2 \\
        &= V_{\theta^\varepsilon}^{i, \dagger}(\pi^{-i}) - \varepsilon / 2 \\
        &\geq V_{\theta^*}^{i, \dagger}(\pi^{-i}) - \varepsilon.
    \end{align*}
    Hence, $\pi$ is an $\varepsilon$-approximate equilibrium (either NE or CCE).
\end{proof}

\newpage

\end{document}